\documentclass{article}
\usepackage[preprint]{neurips_2025}

\usepackage[utf8]{inputenc} 
\usepackage[T1]{fontenc}    
\usepackage{hyperref}       
\usepackage{url}            
\usepackage{booktabs}       
\usepackage{wrapfig}
\usepackage{amsfonts}       
\usepackage{nicefrac}       
\usepackage{microtype}      
\usepackage[dvipsnames]{xcolor}        
\usepackage[T1]{fontenc}
\usepackage{graphicx}   
\usepackage{amsmath}
\usepackage{amssymb} 
\usepackage{amsfonts}
\usepackage{amscd}
\usepackage{bm}
\usepackage{latexsym}
\usepackage{verbatim}
\usepackage[small,bf]{caption}
\usepackage{subcaption}
\usepackage{bbm}
\usepackage{microtype}
\usepackage{kantlipsum}
\usepackage{mathtools, cuted}
\usepackage{units}
\usepackage{comment}
\usepackage{color}  
\usepackage{url}
\usepackage{hyperref}
\usepackage{listings}
\usepackage{tikz} 
\usepackage{tikzlings}
\usepackage{tikzducks}
\usepackage{pgfplots}
\pgfplotsset{compat=1.18}
\usepackage{fancyhdr}
\usetikzlibrary{patterns}
\usepackage{algorithmicx}
\usepackage{algorithm}
\usepackage{algpseudocode}
\usepackage{framed}
\usepackage{fullpage}
\usepackage{tcolorbox}
\usepackage{tgpagella}
\usepackage{fge}
\usepackage{mathdots} 
\usepackage{enumitem}
\usepackage{amsmath}

\usepackage{multicol}

\allowdisplaybreaks  

\hypersetup{
    colorlinks=true,%
    citecolor=blue,%
    filecolor=blue,%
    linkcolor=blue,%
    urlcolor=blue
}

\usepackage[toc, page]{appendix}

\newtheorem{theorem}{Theorem}[section]
\newtheorem{lemma}[theorem]{Lemma}
\newtheorem{corollary}[theorem]{Corollary}

\newtheorem{claim}[theorem]{Claim}

\makeatletter
\newsavebox{\mybox}\newsavebox{\mysim}
\newcommand{\distras}[1]{%
  \savebox{\mybox}{\hbox{\kern3pt$\scriptstyle#1$\kern3pt}}%
  \savebox{\mysim}{\hbox{$\sim$}}%
  \mathbin{\overset{#1}{\kern\z@\resizebox{\wd\mybox}{\ht\mysim}{$\sim$}}}%
}

\newcommand{\sff}{\mr{I\!I}}

\renewcommand{\mathbf}{\boldsymbol} 

\newcommand{\mb}{\mathbf}
\newcommand{\mc}{\mathcal}

\newcommand{\bb}{\mathbb}

\newcommand{\indicator}[1]{\mathbbm 1_{#1}}

\newcommand{\wh}{\widehat}

\newcommand{\paren}[1]{\left( #1 \right)}
\newcommand{\innerprod}[2]{\left\langle #1,  #2 \right\rangle}

\newcommand\subsetsim{\mathrel{%
  \ooalign{\raise0.2ex\hbox{$\subset$}\cr\hidewidth\raise-0.8ex\hbox{\scalebox{0.9}{$\sim$}}\hidewidth\cr}}}
\newcommand\supsetsim{\mathrel{%
  \ooalign{\raise0.2ex\hbox{$\supset$}\cr\hidewidth\raise-0.8ex\hbox{\scalebox{0.9}{$\sim$}}\hidewidth\cr}}}

\newcommand{\mr}{\mathrm}

\numberwithin{equation}{section}

\def \endprf{\hfill {\vrule height6pt width6pt depth0pt}\medskip}

\newenvironment{proof}{\noindent {\bf Proof} }{\endprf\par}

\definecolor{fuchsia}{rgb}{1.0, 0.0, 1.0}

\definecolor{carminered}{rgb}{1.0, 0.0, 0.22}

\definecolor{bittersweet}{rgb}{1.0, 0.44, 0.37}

\definecolor{ao(english)}{rgb}{0.0, 0.5, 0.0}

\definecolor{byzantine}{rgb}{0.74, 0.2, 0.64}

\definecolor{amber}{rgb}{1.0, 0.75, 0.0}

\definecolor{amethyst}{rgb}{0.6, 0.4, 0.8}

\definecolor{rrw}{rgb}{0.9, 0.2, 0.3}

\definecolor{blue-violet}{rgb}{0.54, 0.17, 0.89}

\definecolor{celestialblue}{rgb}{0.29, 0.59, 0.82}

\definecolor{atomictangerine}{rgb}{1.0, 0.6, 0.4}

\definecolor{darkgreen}{rgb}{0.0, 0.2, 0.13}
\definecolor{brass}{rgb}{0.71, 0.65, 0.26}

\definecolor{comment}{RGB}{166, 38, 164}

\definecolor{royalpurple}{rgb}{0.37, 0.22, 0.57}
\definecolor{fuchsia(web)}{rgb}{0.65, 0.2, 0.65}
\definecolor{lava}{rgb}{0.81, 0.06, 0.13}
\definecolor{forestgreen(web)}{rgb}{0.13, 0.55, 0.13}

\DeclareMathOperator*{\diam}{\text{diam}}

\usepackage{soul}

\title{Local Averaging Accurately Distills Manifold Structure From Noisy Data}

\author{
  Yihan Shen\\
  Department of Computer Science\\
  Data Science Institute \\ 
  Columbia University\\
  New York, NY, 10027 \\
  \texttt{ys3524@columbia.edu} \\
  \And
  Shiyu Wang \\
  Department of Electrical Engineering \\
  Data Science Institute \\ 
  Columbia University\\
  New York, NY, 10027 \\
  \texttt{sw3601@columbia.edu} \\
  \AND
  Arnaud Lamy \\
  Department of Electrical Engineering \\
  Data Science Institute \\ 
  Columbia University\\
  New York, NY, 10027 \\
  \texttt{aml2379@columbia.edu} \\
  \And
  Mariam Avagyan \\
  Department of Electrical Engineering \\
  Data Science Institute \\ 
  Columbia University\\
  New York, NY, 10027 \\
  \texttt{ma3810@columbia.edu} \\
  \And
  John Wright \\
  Department of Electrical Engineering \\
  Department of Applied Physics and Applied Mathematics \\ 
  Data Science Institute \\ 
  Columbia University\\
  New York, NY, 10027 \\
  \texttt{jw2966@columbia.edu} \\
}

\begin{document}

\maketitle

\begin{abstract}
High-dimensional data are ubiquitous, with examples ranging from natural images to scientific datasets, and often reside near low-dimensional manifolds. Leveraging this geometric structure is vital for  downstream tasks, including signal denoising, reconstruction, and generation. However, in practice, the manifold is typically unknown and only noisy samples are available. A fundamental approach to uncovering the manifold structure is local averaging, which is a cornerstone of state-of-the-art provable methods for manifold fitting and denoising. However, to the best of our knowledge, there are no works that rigorously analyze the accuracy  of local averaging in a manifold setting in high-noise regimes. In this work, we provide theoretical analyses of a two-round mini-batch local averaging method applied to noisy samples drawn from a $d$-dimensional manifold $\mathcal M \subset \mathbb{R}^D$, under a relatively high-noise regime where the noise size is comparable to the reach $\tau$. We show that with high probability, the averaged point $\hat{\mathbf q}$ achieves the bound $d(\hat{\mathbf q}, \mathcal M) \leq \sigma \sqrt{d\left(1+\frac{\kappa\mathrm{diam}(\mathcal M)}{\log(D)}\right)}$, where $\sigma, \mathrm{diam(\mathcal M)},\kappa$ denote the standard deviation of the Gaussian noise, manifold's diameter and a bound on its extrinsic curvature, respectively. This is the first analysis of local averaging accuracy over the manifold in the relatively high noise regime where $\sigma \sqrt{D} \approx \tau$. The proposed method can serve as a preprocessing step for a wide range of provable methods designed for lower-noise regimes. Additionally, our framework can provide a theoretical foundation for a broad spectrum of denoising and dimensionality reduction methods that rely on local averaging techniques. 
\end{abstract}

\section{Introduction}

 Many commonly occurring high-dimensional data -- including natural and medical images, videos, and scientific measurements such as gravitational waves and neuronal recordings -- reside near low-dimensional manifolds, a principle known as the {\em manifold hypothesis}. This intrinsic low-dimensional structure has inspired advances in manifold learning, fitting, embedding, and denoising, all aimed at revealing and utilizing this underlying geometry. In real-world settings, where the manifold is unknown and only noisy samples are available, a natural approach to unveiling the manifold structure is local averaging -- as illustrated in Figure \ref{fig:averaging_example}, often an average of noisy points is closer to the manifold than any of the points themselves. This improvement is especially significant for high-dimensional data, making local averaging a crucial building block for both practical methods and theoretical estimators.

 \begin{figure}
 \centerline{
 \begin{tikzpicture}
     \node at (-5,.35) {\includegraphics[width=2.05in]{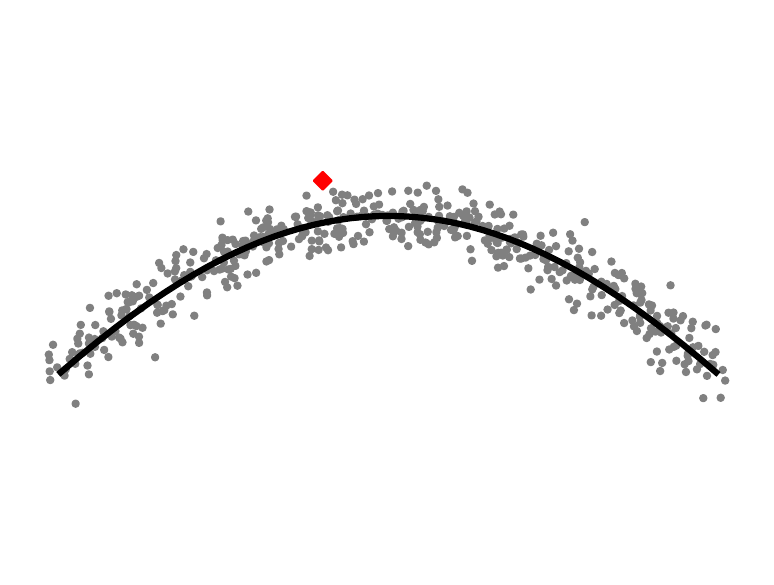}};
     \node at (-5,-.6) {\bf Problem Setup};
     \node at (0,.35) {\includegraphics[width=2.05in]{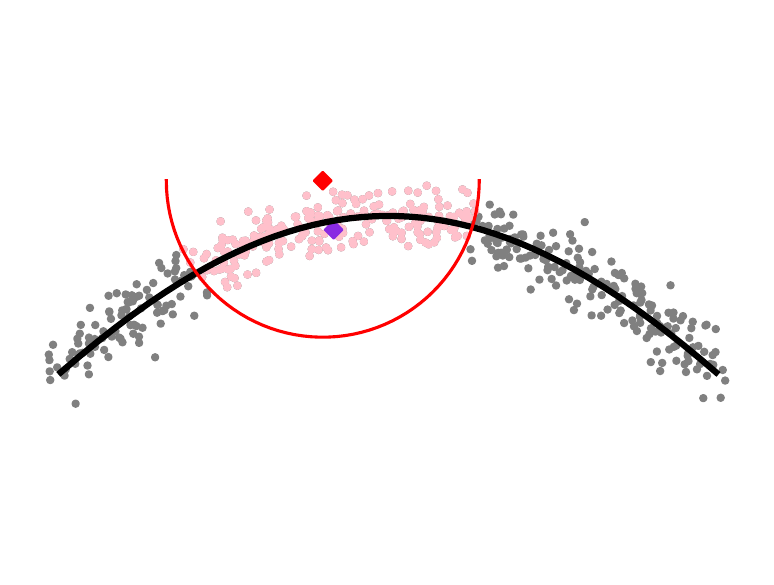}};
     \node at (0,-.6) {\bf First Round Local Avg.}; 
     \node at (5,.35) {\includegraphics[width=2.05in]{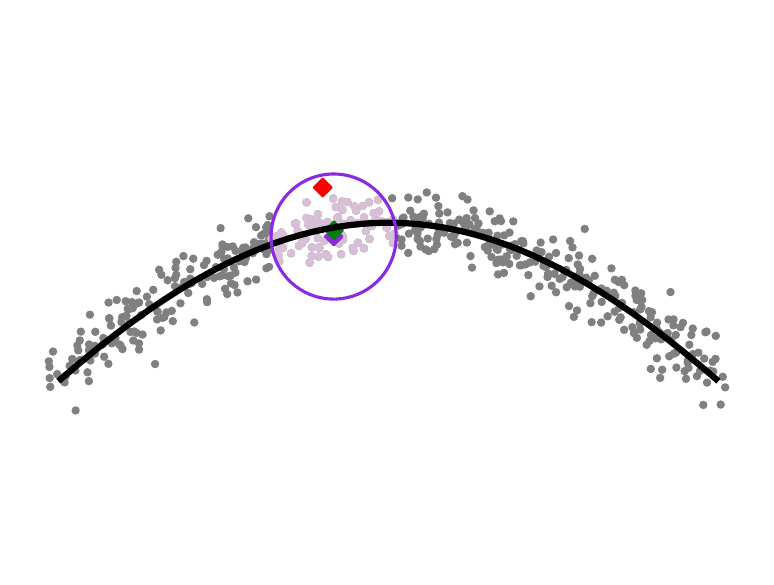}};
     \node at (5,-.6) {\bf Second Round Local Avg.};
     \node at (5,-3) {\includegraphics[width=2in]{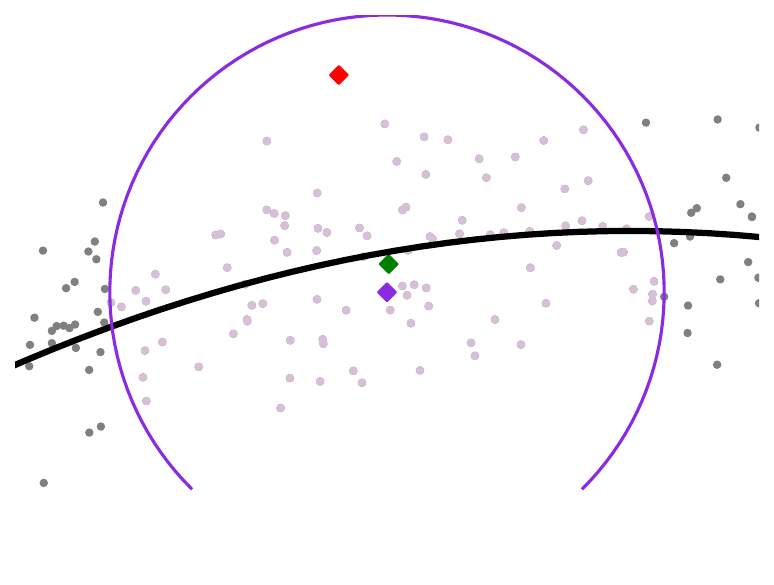}};
    \node at (-2.5,-2.75) {\begin{minipage}{3.5in}\begin{tcolorbox} {\bf Our results (sketch):} for $\sigma \sqrt{D} \lesssim \mathrm{reach}(\mc M)$ \color{Purple}
    \begin{itemize} 
    \item  $d(\mb q^1,\mc M) \lesssim \sigma \sqrt{d (\kappa \, \mr{diam}(\mc M) + \log(D)) } $ \color{PineGreen} 
    \item $d(\mb q^2,\mc M) \lesssim \sigma \sqrt{d (1 + \kappa \, \mr{diam}(\mc M) / \log(D)) }$ \end{itemize}\end{tcolorbox} \end{minipage}}; 
    \node at (5,-1.5) {\color{red} $\mb q^0$};
    \node at (5.3,-3.1) {\color{Purple} $\mb q^1$};
    \node at (5.25,-2.5) {\color{PineGreen} $\mb q^2$};
    \node at (5,-4) {\bf Detail};
     \end{tikzpicture} 
 } \vspace{-.25in}
\caption{{\bf Local Averaging for Manifold Estimation.} Left: problem setup - we observe noisy points $\mb x_1 = \mb x_{1,\natural} + \mb z_1, \mb x_2 = \mb x_{2,\natural} + \mb z_2, \dots$ near a manifold $\mc M$. In contrast to existing theory, our results pertain to the {\em large noise regime} ($\sigma \sqrt{D} \approx \mr{reach}(\mc M)$). Middle: {\color{Purple} \bf local average  $\mb q^1$} around a {\color{red} \bf noisy point $\mb q^0 = \mb x_i$}. Right:  {\color{PineGreen} \bf second round local average $\mb q^2$} around the first round average {\color{Purple} $\mb q^1$}. Our main result shows that with appropriately chosen neighborhood size, {\color{Purple} $\mb q^1$} and {\color{PineGreen} $\mb q^2$} are very close to $\mc M$ - and controls the distance in terms of the noise standard deviation $\sigma$, intrinsic dimension $d$ and manifold curvature $\kappa$. Crucially, {\color{PineGreen}$d(\mb q^2,\mc M)$} bounded by $\sigma \sqrt{d}$, {\em regardless of ambient dimension $D$.} } \label{fig:averaging_example}
 \end{figure}
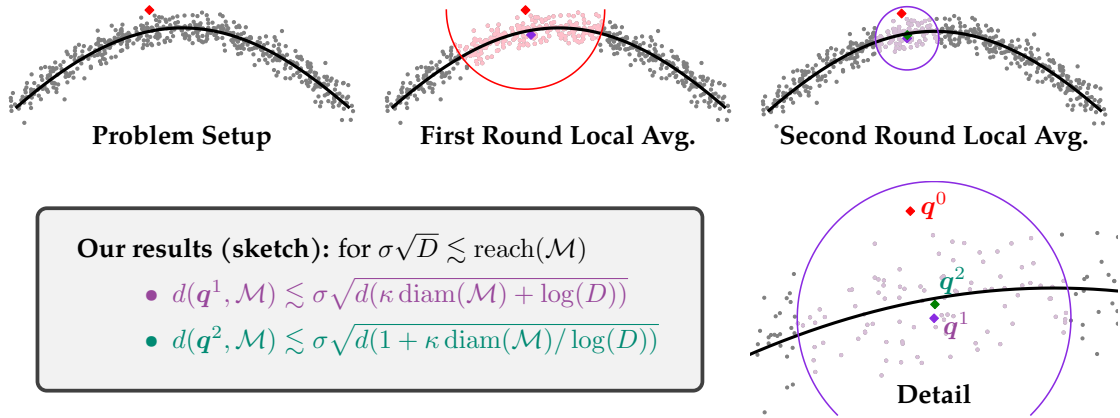
 
A variety of approaches based on local averaging aim to leverage commonalities between neighboring points to extract clean structure from messy data. For instance, the {\em mean shift} algorithm \cite{fukunaga1975estimation} iteratively moves data points towards local averages to approximate the gradient of the data density. {\em Nonlocal means} denoising \cite{buades2005non, buades2011non} reduces the noise in images by averaging the similar structure in image patches. Local averages also play a key role in capturing manifold structure: classical methods such as locally-linear embedding (LLE) \cite{roweis2000lle} seek a low-dimensional coordinate system that preserves local weighted averages. State-of-the-art provable methods for manifold estimation and denoising are also built on local averaging \cite{yao2023manifold, yao2025manifold}. 

Despite its widespread and successful use, the theoretical understanding of local averaging over manifolds remains incomplete. Existing analyses depend on restrictive assumptions -- such as low noise level--that limit their applicability. This gap between the practice and theory highlights the need for a comprehensive theoretical framework for studying local averaging in higher noise regimes.

In this work, we close this gap by providing a theoretical analysis of the simple yet powerful local averaging method in the large noise regime. We consider data sampled from a $d$-dimensional Riemannian manifold $\mc M \subset \mathbb{R}^D$, corrupted by independent Gaussian noise $\mb z \sim \mc{N}(\mb 0, \sigma^2 \mb I)$. We analyze a simple two-stage local averaging method, where each stage operates by averaging over a mini-batch of neighboring samples, developing a finite-sample error bound in the regime in which $\sigma\sqrt{D} \lesssim \tau$, where $\tau $ is the reach of $\mc M$. As we review below, existing analyses pertain to much smaller noise levels (e.g., translating the results into gaussian noise, implicitly assuming $\sigma D \sqrt{\log(D)} < \tau $ in \cite{aizenbud2025estimation}, or $\sigma e^{c D/d} \ll 1$ in \cite{yao2023manifold, yao2025manifold}).

Under natural assumptions, for Stage I (“Coarse Localization”), we show that averaging over a single mini-batch of nearby samples produces a point whose distance to the manifold is bounded by $O\left(\sigma\sqrt{d\left(\log(D) + \kappa\:\mr{diam}(\mc M)\right)}\right)$. For stage II ("Fine Refinement"), with slightly stronger assumptions, we show that the distance is reduced to $O\left(\sigma\sqrt{d\left(1 + \frac{\kappa\:\mr{diam}(\mc M)}{\log(D)}\right)}\right)$. 

Our main technical innovation is a careful analysis of localization in this problem -- controlling the probability that a noisy point $\mb x_\natural + \mb z$ belongs to the ball $B(\mb q, R )$. This probability exhibits a sharp phase transition around $\| \mb q - \mb x_\natural \|_2^2 \approx R^2 - \sigma^2 D$; we develop sharp subgaussian bounds for this and related probabilities, leading to estimates that accurately capture localization even when $\sigma$ is relatively large. 

To the best knowledge of the authors, our work provides the first theoretical guarantee for local averaging  in the relatively large noise regime ($\sigma\sqrt{D}\approx \tau$). Since local averaging is ubiquitous in the area of machine learning and signal processing, our theory provides both a design principle –– choosing the appropriate neighborhood radii and minibatch sizes according to noise and manifold geometry -- and a performance guarantee that researchers can directly apply while using local averaging. Moreover, our novel analysis framework offers a valuable tool for further theoretical investigations into this local averaging problem. Apart from the theoretical novelty, our results also have immediate practical implications: they justify the use of averaging in graph Laplacian construction\citep{belkin2008towards}, similarity distance estimation \citep{trillos2019local} , and net‐building \citep{wang2025fast}. Our work could be integrated as a pre-processing step for the more refined techniques for manifold estimation \cite{yao2023manifold} which have assumed a much smaller noise level.

\section{Relation to the Literature} \label{sec:rtl} 

Local averaging forms the core of classical algorithms such as Mean Shift \cite{fukunaga1975estimation} and Nonlocal Means \cite{buades2005non}. Mean Shift\footnote{Reminiscent of \cite{schnell1964methode}.} repeatedly averages data points to seek modes in the data distribution and can reveal manifold structure in noisy data \cite{genovese2014nonparametric}. Existing theory on mean shift analyzes its rate of convergence to a mode of the data distribution for various kernels and under various hypotheses (see, e.g., \cite{comaniciu1999mean, huang2018convergence,carreira2007gaussian,yamasaki2024convergence}). While inspiring, these results do not have direct implications on the accuracy of local averaging as a manifold estimator.  

A wide variety of provable manifold estimation and denoising methods rely on local geometry (\cite{aamari2019nonasymptotic,wang2010manifold}). For example, \cite{mohammed2017manifold} leverages the Approximate Squared‐Distance Function (ASDF) framework, which one can estimate based on local neighborhoods of clean samples, to  recover $\mathcal{M}$.  Similarly, \cite{yao2025manifold} estimates local biases by averaging and fits a smooth manifold by estimating tangents at the projected points onto the latent manifold. \cite{yao2023manifold} propose a two‐step mapping that moves any noisy point $\mathbf{x}$ with $d(\mathbf{x},\mathcal{M})\le c\,\sigma$ closer to the manifold, yielding a final accuracy of $O\bigl(\sigma^2\log(1/\sigma)\bigr)$. Both \cite{yao2023manifold} and \cite{yao2025manifold} enjoy rigorous theoretical guarantees and achieve state-of-the-art accuracy, but they all implicitly assume a “small‐noise’’ regime where \ $\sigma e^{cD/d}\ll 1$.

\cite{aizenbud2025estimation} shows  accurate point‐and‐tangent estimation under certain sampling and smoothness conditions without requiring the noise level to be exponentially small. However, this work still restricts samples to lie in a bounded tubular neighborhood\footnote{Many other works on smooth manifold reconstruction \cite{chazal2008smooth, genovese2012minimax}, homology estimation \cite{niyogi2008finding}, and manifold distance estimation \cite{ trillos2019local} operate under similar stringent noise models: for example, bounding its magnitude or assuming it lies exactly in the normal space at each clean point. Our work also compares to previous results in pairwise distance estimation \cite{trillos2019local}, and provide more accurate results in the gaussian noise regime.} of $\mathcal{M}$ of radius $\tau / \sqrt{D\log D}$. Translating to the Gaussian noise model studied here, this scaling would ask $\sigma \sqrt{D} \lesssim \tau / \sqrt{D \log D}$. Below, we will study the Gaussian noise model under the less restrictive assumption $\sigma \sqrt{D} \lesssim \tau$. Prior algorithmic theory by \cite{fefferman2019fitting} also pertains to the large noise regime, giving a method which accurately estimates $\mc M$ when $\sigma \sqrt{D} \lesssim \tau$. \cite{fefferman2019fitting} also analyzes a PCA preprocessing step which further weakens this assumption to   $\sigma\sqrt{D_{\mathrm{PCA}}}\le c\,\tau$ for a projection dimension $D_{\mathrm{PCA}}$ that only depends on $\mathrm{Vol}(\mc M)$ and $d$. The main practical downside of \cite{fefferman2019fitting} is the rather intricate algorithm, which involves local PCA, discretization and refined local estimation. The complexity of this algorithm means that its importance is mostly theoretical. In contrast, in this work, we study the behavior of a very simple, practical and widely used approximation technique (local averaging) in the moderate-to-large noise regime.

\section{Problem Formulation}
Let $\mathcal{M} \subset \mathbb{R}^D$ be a smooth, compact $d$-dimensional Riemannian manifold embedded in the higher dimension ambient space $\mathbb{R}^D$. We observe a stream of noisy data points whose clean signals are uniformly sampled from $\mc M$, i.e.,
\begin{equation} \label{eq: data model}
\{ \mb{x}_i \}_{i\in \mathbb{N}} \subset \mathbb{R}^D, \quad \text{where } \mathbf{x}_i = \mathbf{x}_{i,\natural} + \mb z_i, \quad \mathbf{x}_{i,\natural} \in \mathcal{M},\; \mb z_i \overset{iid}{\sim} \mathcal{N}(\mb 0, \sigma^2 \mb I).
\end{equation}
Our goal is to produce cleaner data points $\{\mb q_i\}_{i\in \mathbb{N}} \subset \mathbb{R}^D$ referred to as {\em landmarks}, with a provable upper bound on $d\Bigl(\mb q_i, \mc M \Bigr).$ The landmarking problem -- identifying representative points that lie closer to the manifold from noisy samples -- is  challenging due to the interplay between the noise size and the manifold's geometric properties. Since the noisy sample's distance to the manifold scales as $\Theta(\sigma \sqrt{D})$, the regime where $\sigma \sqrt{D} \leq \tau$, where the projection of noisy samples onto the manifold is unique, is of greater theoretical interest \cite{aizenbud2025estimation, yao2025manifold,yao2023manifold}.

The simplest approach to identify points closer to the manifold and distill its structure from noisy samples is local averaging. It leverages the fact that nearby points share similar tangent structures, and their mean are often closer to the manifold. In practice, similar to the idea of {\em Mean Shift}, iterative local averaging is simple and useful: starting from an initial noisy sample $\mb q^0$, at each iteration $\ell$, it collects $\mb N_\ell$ neighboring samples whose extrinsic distance to $\mb q^\ell$ is less than $R_\ell:$
\[
\mathcal X_\ell = \{\mathbf x_{\ell,i}\}_{i=1}^{N_\ell}
    \;\subset\; B\bigl(\mathbf q^\ell,\,R_\ell\bigr),
\]
where each $\mb x_{\ell,i}$ follows the data model in Equation \eqref{eq: data model}, and updates the point by averaging over these neighbors:
\begin{equation}
    \mb q^{\ell+1} \leftarrow \frac{1}{N_\ell} \sum_{i=1}^{N_{\ell}} \mb x_{\ell,i}.
\end{equation}

Under the additive Gaussian noise model, we would like to show that the local average $\frac{1}{N_\ell} \sum_{i=1}^{N_{\ell}} \mb x_{\ell,i} = \frac{1}{N_\ell} \sum_{i=1}^{N_{\ell}} \mb x_{\ell,i, \natural} + \frac{1}{N_\ell} \sum_{i=1}^{N_{\ell}} \mb z_{\ell,i}$
effectively suppresses the noise and produces a point close to the manifold. Intuitively, averages of independent gaussian noise vectors have a zero expectation and tend to reduce the noise magnitude, but in this setting we are averaging noise vectors $\mb z_{\ell,i}$ conditioned on the noisy observations $\mb x_{\ell,i}$ being close to $\mb q^\ell.$ We bound its norm by introducing a novel technical tool: subgaussian bounds for grouping probability function $h(s)$ (\ref{section: introduce h}). We also analyze the performance of averaging the latent clean signals of samples near a given landmark using a similar framework around $h(s)$. Geometrically, the local neighborhood of the manifold resembles an affine space. As a result, the average of clean points within a small neighborhood remains close to the manifold. It should be emphasized that local averaging is widely used and serves as the core building block in many state-of-the-art provable manifold estimation methods \cite{yao2025manifold, yao2023manifold}.

In this work, we provide rigorous analysis of the distance between points produced by local averaging through a two-round process (Algorithm \ref{alg:sl}) and the manifold. Our technical contribution lies in the development of subgaussian bounds on $h(s),$ which describes the probability that noisy points fall into an extrinsic ball centered at a given landmark, along with related probability estimates. This analysis provides a foundation for downstream tasks such as tangent space and curvature estimation. Notably, our analyses are adaptive to moderate-to-large noise regime where $\sigma \sqrt{D} \lesssim \tau$. As such, our results can serve as a valuable pre-processing step for methods aimed at handling lower noise scenarios.

\section{Main Result: Accuracy of Two-Stage Local Averaging}
In this algorithm, we perform two consecutive minibatch-averaging steps and introduce a Gaussian perturbation immediately after the first averaging. This noise injection resembles that of stochastic gradient Langevin dynamics (SGLD). Here, it puts a lower bound on the distance between first-round average and the manifold, which is crucial for further decreasing the error bound in the second stage. This procedure is stated more precisely as Algorithm \ref{alg:sl}. Our main theoretical result shows that when the radii $R_1$, $R_2$ are appropriately chosen, the output $\wh{\mb q}$ of Algorithm \ref{alg:sl} is very close to the manifold $\mc M$: 

\begin{algorithm}[t]
\caption{$\mathtt{TwoRoundLandmarking}$} \label{alg:sl} \vspace{-.175in}
\begin{multicols}{2}
    \begin{algorithmic}
    \State {\bf Input:} noisy samples $\mb x_1, \mb x_2, \mb x_3, \dots $, \\ minibatch sizes $N_{\mr{mb},1}$, $N_{\mr{mb},2}$, radii $R_1, R_2$.
    \State {\color{red} $\mb q^{0} \leftarrow \mb x_1$. \Comment{\bf Init.\ with noisy sample}}
    \State {\color{Purple}$\mc X_1 \leftarrow \emptyset$.     \Comment{ \bf First minibatch}}
    \State $i \leftarrow 2$.
    \While{$|\mc X_1| < N_{\mr{mb},1}$} 
    \If{$\| \mb x_i - \mb q^{0} \| \le R_1$}
    \State $\mc X_{1} \leftarrow \mc X_{1} \cup \{ \mb x_i \} $
    \EndIf
    \State $i \leftarrow i + 1$
    \EndWhile
    \State $\mb \vartheta \sim_{\mr{iid}} \mc N(0,\sigma^2 D^{-1/4})$
    \State {\color{Purple} $\mb q^{1} \leftarrow \tfrac{1}{N_{\mr{mb},1}} \sum_{\mb x \in \mc X_{1}} \mb x + \mb \vartheta$ \Comment{\bf Local Avg.}}
    \State {\color{PineGreen} $\mc X_2 \leftarrow \emptyset$. \Comment{\bf Second minibatch}}
    \While{$|\mc X_2| < N_{\mr{mb},2}$} 
    \If{$\| \mb x_i - \mb q^{1} \| \le R_2$}
    \State $\mc X_{2} \leftarrow \mc X_{2} \cup \{ \mb x_i \} $
    \EndIf
    \State $i \leftarrow i + 1$
    \EndWhile  
    \State {\color{PineGreen} $\mb q^{2} \leftarrow \tfrac{1}{N_{\mr{mb},2}} \sum_{\mb x \in \mc X_{2}} \mb x$ \Comment{\bf Local Avg.}}
    \State {\bf Output}: $\wh{\mb q} = \mb q^{2}$. 
    \end{algorithmic}
\end{multicols}
\vspace{-.075in}
\end{algorithm}

\begin{theorem}[Landmarking Accuracy – Two-Stage Refinement]\label{thm:landmark-accuracy}
Let $\mathcal M\subset\mathbb R^D$ be a compact, connected, geodesically complete $d$-dimensional submanifold with extrinsic curvature bounded by~$\kappa$, and intrinsic diameter 
\[
\operatorname{diam}(\mathcal M)\;=\;\sup_{\mb x,\mb y\in\mathcal M}d_{\mathcal M}(\mb x,\mb y)\,,
\]
satisfying $\kappa\,\diam(\mathcal M)\ge1$.  Let $\bar\kappa=\max\{1,\kappa\}$.  Suppose the ambient dimension $D$ is large enough such that
\[
D^{1/12}\;\ge\;C_1\,\max\Bigl\{\bar\kappa^2\,d,\;\kappa\,\diam(\mathcal M)\,d\Bigr\},
\]
and the noise standard deviation~$\sigma$ satisfies 
\[
\sigma\,D\;\ge\;1,\quad
\sigma\,D^{20}\;\ge\;\frac1\kappa,
\quad\text{and}\quad
\sigma\sqrt{D}\;\le\;{c_1\tau},
\]
where $\tau=\operatorname{reach}(\mathcal M)$ is the radius of the largest tubular neighborhood on which the projection onto~$\mathcal M$ is unique.

\medskip
\noindent\textbf{Stage I (Coarse Localization).}  
Run one minibatch pass of Algorithm~\ref{alg:sl} with
\[
N_{\mathrm{mb, 1}}
\;\ge\;C_2
\frac{\log(D)\,\diam(\mathcal M)^2}{\sigma^2\,\bigl(\log(D)+\kappa\,\diam(\mathcal M)\bigr)},
\]
and choose the acceptance radius
\[
R_1^2
\;=\;
\sigma^2\,(2D - d - 3)
\;+\;
C_3\,\bar\kappa^2\,\sigma^2\,d
\;+\;
3C_3\,\bar\kappa\,\sigma^2\sqrt{Dd}.
\]
Then with probability at least $1 - 5e^{-c_2d}$, the first landmark satisfies
\[
d\Bigl(\mb q^1,\mathcal M\Bigr)
\;\le\;
C_4\,\bar\kappa\,\sigma\sqrt{d}
\;+\;
C_4\,\sigma\sqrt{\,d\bigl(\kappa\,\diam(\mathcal M)+\log(D)\bigr)\,}\,.
\]

\medskip
\noindent\textbf{Stage II (Fine Refinement).}  
Assuming in addition $\frac{1}{\kappa}\ge C_5\,\sigma\sqrt{D\log(D)}$, perform a second minibatch pass with
\[{
N_{\mathrm{mb, 2}}
\;\ge\;C_6
\frac{\log^2(D)\,\diam(\mathcal M)^2}{\sigma^2\,\bigl(\log(D)+\kappa\,\diam(\mathcal M)\bigr)}}
\quad\text{and}\quad
R_2^2
=\;
\sigma^2\bigl(D-3 +D^{3/4} + 2C_7\,D^{5/12}\bigr).
\]
Then with { probability at least $1 - 9e^{-c_2d}$} the refined landmark $\mb q^2$ satisfies
\[
d\bigl(\mb q^2,\mathcal M\bigr)
\;\le\;
C\,\sigma\,\sqrt{\,d\,\left(\frac{\kappa\,\diam(\mathcal M)}{\log(D)}\,+ 1\right)} \,.
\]
\end{theorem}

In the first round, our bound has a (mild) dependence on the ambient dimension $D$ -- namely, the error is proportional to $\sigma \sqrt{d \log (D)}$. With mildly stronger assumption on $\kappa$, the second round of averaging removes this dependence. The resulting accuracy, $O(\sigma\sqrt{d})$, is a natural scale associated with  manifold denoising.\footnote{Namely, if the manifold $\mc M$ is {\em known}, and we attempt to denoise an observation of the form $\mb x = \mb x_\natural + \mb z$, the minimum achievable RMSE $(\bb E[ \| \wh{\mb x} - \mb x_\natural \|^2 ])^{1/2}$ is on the order of $\sigma \sqrt{d}$. In this sense, $\sigma \sqrt{d}$ is a fundamental limit for denoising. Note that {\em it is} possible to achieve a smaller error in estimating the manifold $\mc M$. As described in Section \ref{sec:rtl} several previous works achieve Hausdorff error $O(\sigma^2)$, using more complicated methods or with more restrictive assumptions on the noise level $\sigma$. 
} The two-stage averaging procedure in Algorithm \ref{alg:sl} extends naturally to multiple rounds of averaging; due to its minibatch update it admits a natural online implementation. As described in Section \ref{sec:rtl}, this method is significantly simpler than most existing provable manifold estimation methods; due to its reliance on very simple statistical operations it scales well with ambient dimension $D$. 

Theorem~\ref{thm:landmark-accuracy} guarantees that the landmark $\mb q^1$ produced by local averaging around $\mb q^{0}$ is very close to $\mathcal M$. As a by-product, our analysis also controls the distance between the first round local average and $\mb q^0_\natural = \mc P_{\mc M} \mb q^0$, the projection of $\mb q^0$ onto $\mc M$. We will bound this distance in Corollary \ref{cor: signal estimation via local averaging} below. Based on this bound, we can interpret the first round local average as a (reasonably accurate) estimate of $\mb q^0_\natural$. In Algorithm \ref{alg:sl}, $\mb q^0 = \mb x_1 = \mb x_{1,\natural} + \mb z_1$ is itself a noisy version of a clean point $\mb x_{1,\natural} \in \mc M$. It is not difficult to show that with high probability $\| \mb q^0_\natural - \mb x_{1,\natural} \| \lesssim \sigma \sqrt{d}$, and hence Corollary \ref{cor: signal estimation via local averaging} also controls the accuracy in estimating $\mb x_{1,\natural}$. We have 

\begin{corollary}[Signal Estimation via Local Averaging]\label{cor: signal estimation via local averaging}
Under the same assumptions as Theorem~\ref{thm:landmark-accuracy} and with the same accepting radius $R_1$ and minibatch size $N_{\mr{mb, 1}}$, let $\mb q^{0}$ be the coarse landmark in Stage I and 
$
\mb {q}^0_{\natural}$
be its projection onto $\mc M$. Then there exist constants \(C_1,C_2>0\) such that
with probability at least $1-3e^{-C_1d}$,
the local average's distance from the original clean signal is bounded by
\begin{equation} \label{eqn:signal-accuracy} 
    \left\|\frac{1}{N_{\mathrm{mb, 1}}}
\sum_{\mb x \in \mc X_1}^{}
\mb x -\,\mb q^{0}_{\natural}\right\|
\le 
C_2\sigma\sqrt{\bar{\kappa}}(Dd)^{1/4}+C_2\sigma\Bigl(Dd\:\bigl(\kappa\mr{diam}(\mc M) + \log(D)\bigr)\Bigr)^{1/4}.
\end{equation}
\end{corollary}

This corollary allows us to use the local average of noisy $\mb x$ around ${\mb q^{0}}$ as an accurate approximation for $\mb q^{0}_{\natural}$ (and hence for $\mb x_{1,\natural}$, which lies within distance $O(\sigma \sqrt{d})$ of $\mb q^0_\natural$), and hence the local average can be used in many downstream applications. For instance, to estimate pairwise distances between two clean signals from $\mc M$, $\mb q_{i,\natural}$ and $\mb q_{j,\natural}$, we simply compute $
\bigl\|\tfrac1{N_{\rm mb}}\sum_{\ell=1}^{N_{\rm mb}}\mb x_{i,\ell}\;-\;\tfrac1{N_{\rm mb}}\sum_{\ell=1}^{N_{\rm mb}}\mb x_{j,\ell}\bigr\|,
$ where each average is taken over the noisy minibatch grouped to $\mathbf q_i$ and $\mathbf q_j$, respectively. Our result implies that 
\begin{equation} \label{eqn:pairwise-distance} 
   \left| \bigl\|\tfrac1{N_{\rm mb}}\sum_{\ell=1}^{N_{\rm mb}}\mb x_{i,\ell}\;-\;\tfrac1{N_{\rm mb}}\sum_{\ell=1}^{N_{\rm mb}}\mb x_{j,\ell}\bigr\| - \| \mb q_{i,\natural} - \mb q_{j,\natural} \| \right|
\end{equation}
is bounded by $\sigma d^{1/4} D^{1/4} (\log D)^{1/4}$. In high dimensions, this is signficantly smaller than the distance $O(\sigma D^{1/2})$ between a pair of noisy points, which in turn has implications on the accuracy of algorithms that depend on pairwise distances, including  embedding methods such as isomap \cite{tenenbaum2000global} and spectral embedding  \cite{belkin2008towards}.

This bound can be compared with existing work on pairwise distance estimation. For example, \cite{trillos2019local} studies pairwise distance estimation with small bounded noise\footnote{With additional restrictions that include constraining the noise to the normal direction at the clean point $\mb x_\natural$.} ($\|\mb z\| \le \vartheta \le 1)$ and concluded an error bound of $C_{\mc M} \vartheta^{3/2}$, where the constant $C_{\mc M}$ depends on properties of the manifold $\mc M$ and clean data distribution. In our setting, the typical noise norm is $\sigma \sqrt{D}$; setting $\vartheta = \sigma \sqrt{D}$, one obtains a bound of order $\sigma^{3/2} D^{3/4}$. In the large noise regime $\sigma D^{1/2} = \Omega(1)$, \eqref{eqn:signal-accuracy}-\eqref{eqn:pairwise-distance} improves by a multiplicative factor of $D^{1/4}$. 

Another consequence of \eqref{eqn:pairwise-distance} is for the construction of nets for the manifold $\mc M$. We can construct an $O(\sigma \sqrt{D})$ net by collecting a large number of noisy samples; replacing these points with local averages yields an $O(\sigma (d D \log D)^{1/4})$ net, yielding a discretized approximation of $\mc M$ over which one can perform downstream tasks, including the denoising of new samples \cite{hein2006manifold,wang2025fast}. 

\section{Framework of Analysis and Proof Ideas}

\subsection{Decomposition of Landmark Error into Noise Error and Signal Error}

In this section, we decribe the framework of our analysis to bound the distance between the minibatch average and manifold $\mc M$. We bound the error of local averaging by separately analyzing the contributions from clean signals and noise. Each sample in the $\ell$-th minibatch $\mc X_{\ell}$ can be expressed as $\mb x = \mb x_\natural + \mb z$, and so 
\begin{equation}
\frac{1}{N_{\mr{mb},\ell}} \sum_{\mb x \in \mc X_{\ell}} \mb x \;=\;  \underset{\color{red} \text{\bf Signal average}}{ \frac{1}{N_{\mr{mb},\ell}} \sum_{\mb x = \mb x_\natural + \mb z \in \mc X_{\ell}} \mb x_\natural } \;+\; \underset{\color{blue}\text{\bf Noise average}}{  \frac{1}{N_{\mr{mb},\ell}} \sum_{\mb x = \mb x_\natural + \mb z \in \mc X_{\ell}} \mb z }. \label{eqn:decomp}
\end{equation} 
For the signal average, lemma \ref{lem:signal-avg} uses a first-order expansion to 
show that \[d\left(  \frac{1}{N_{\mr{mb},\ell}} \sum_{\mb x = \mb x_\natural + \mb z \in \mc X_{\ell}} \mb x_\natural, \; \mc M \right) \le \frac{\kappa}{N_{\mr{mb},\ell}} \sum_{\mb x = \mb x_\natural + \mb z \in \mc X_{\ell}} d^2_{\mc M}( \mb x_\natural, \mb q_{\natural}),\]
where
$
\mb q_{\natural} = \mc P_{\mc M}[ \mb q ]
$ is the nearest point on the manifold $\mc M$ to the landmark $\mb q.$ A carefully chosen acceptance radius $R$, together with tools described in the next subsection, controls $d^2_{\mc M}( \mb x_\natural, \mb q_{\natural} ).$ The noise average involves the average of independent gaussian vectors, but it has non-zero conditional mean since we are considering only the $\mb x_{i,\natural} + \mb z_i$ within extrinsic distance $R$ from $\mb q$. We control its expectation by developing a new technical tool, the subgaussian bound for the grouping probability function, which we introduce below in Section \ref{section: introduce h}. During the second round of averaging, our analysis shows the expectation of {\em squared intrinsic distance} satisfies
\begin{equation}
\bb E\Bigl[ \, d_{\mc M}^2(\mb x_\natural,\mb {q^{1}}_{\natural}) \mid \| \mb x - \mb {q^{1}}  \| \le R_2 \, \Bigr] \le C\sigma^2\sqrt{Dd \times \left(\kappa \mr{diam}(\mc M) + \log(D)\right)},
\end{equation} while 
the expectation of {\em noise average} satisfies 
\begin{equation}
\left\| \bb E\Bigl[ \, \mb z \mid \| \mb x - \mb {q^{1}}  \| \le R_2 \Bigr] \right\| \:\le\; C \frac{\sigma}{D^{1/8}}\sqrt{d\times\left(\kappa\, \mr{diam}(\mc M) + \log(D)\right)}, 
\end{equation}
where in both equations $\mb q^1$ is the local average of first round after perturbation. For finite sample bounds, we apply standard Hoeffding and Bernstein inequalities, and we eventually use the triangle inequality to combine the signal and noise error to bound $d\bigl(\mb q^2,\mathcal M\bigr).$  

\subsection{Phase Transition and Subgaussian Concentration for Grouping Probability $h(s)$ }\label{section: introduce h}

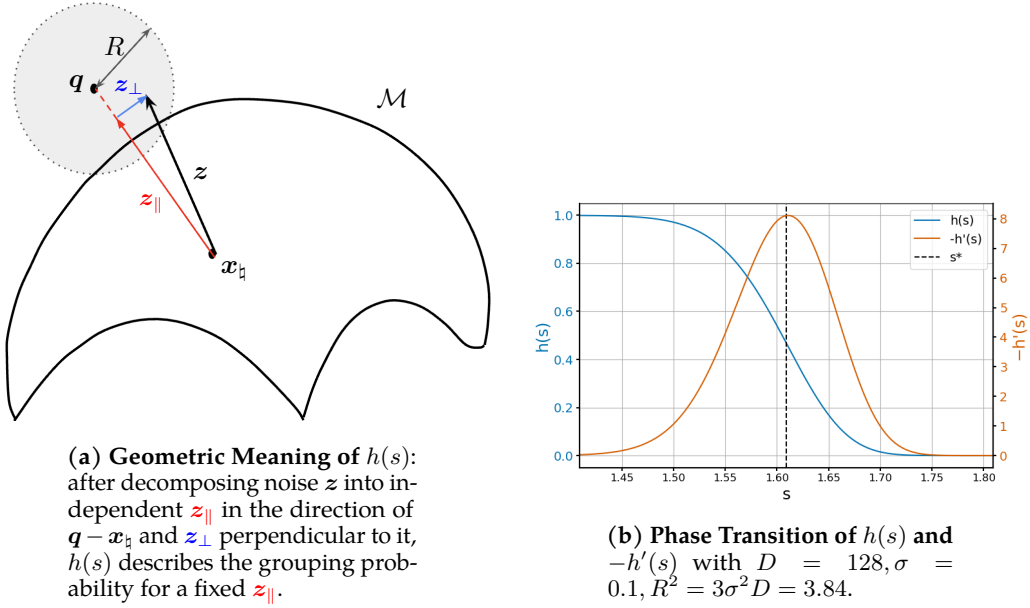
\begin{figure}[ht]
\vskip -0.1in
\begin{center}
\begin{subfigure}{0.48\columnwidth}
    \centering
    \begin{tikzpicture}
    \node at (0,0) {    \includegraphics[width=\linewidth]{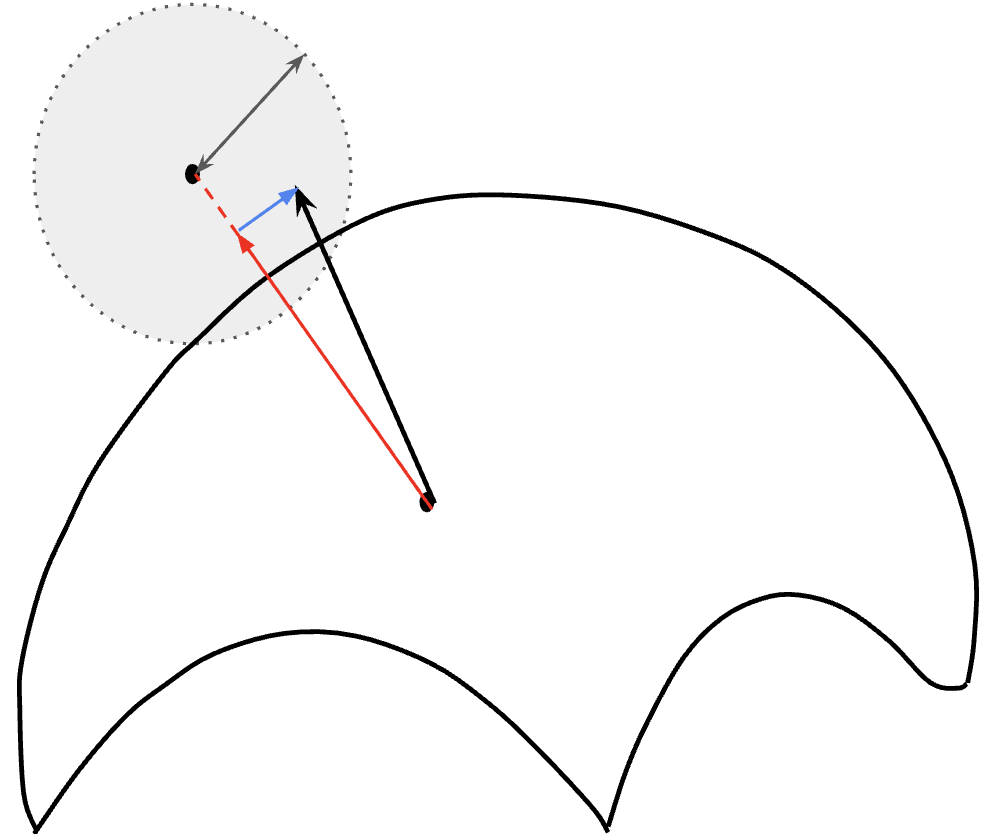}}; 
    \node at (-.2,-.8) {$\mb x_\natural$};
    \node at (1.9,1.5) {$\mc M$};
    \node at (-2.3,1.7) {$\mb q$};
    \node at (-1.8,2.2) {$R$};
    \node at (-.65,.5) {$\mb z$};
    \node at (-1.3,.1) {\footnotesize \color{red}$\mb z_\parallel$};
    \node at (-1.6,1.65) {\color{blue} \footnotesize $\mb z_\perp$};
    \end{tikzpicture}
    \caption{\textbf{Geometric Meaning of $h(s)$}: after decomposing noise $\mb z$ into independent {\color{red} $\mb z_\parallel$} in the direction of $\mb q - \mb x_\natural$ and {\color{blue}$\mb z_\perp$} perpendicular to it, $h(s)$ describes the grouping probability for a fixed {\color{red} $\mb z_\parallel$}.}
    \label{fig:h(s) on the manifold}
\end{subfigure}
\hfill
\begin{subfigure}{0.48\columnwidth}
    \centering
    \includegraphics[width=\linewidth]{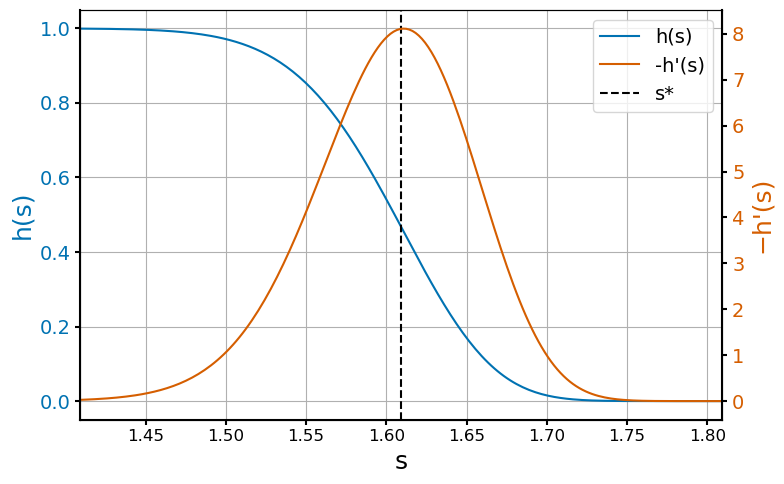}
    \caption{\textbf{Phase Transition of $h(s)$ and $-h'(s)$} with $D = 128, \sigma = 0.1, R^2 = 3\sigma^2D = 3.84$.}
    \label{fig:h_and_hprime_values}
\end{subfigure}
\hfill
\caption{Motivation of Grouping Probability $h(s)$}
\label{fig:behavior_of_h_and_hprime}
\end{center}
\end{figure}

A natural question in our analysis is, for a fixed $\mb x_{i,\natural} \in \mc M$, what is the probability that it can be grouped to $\mb q$? Formally, we need to find 
$$\bb P_{\mb z_i \sim_{\mr{iid}} \mc N(0,\sigma^2 I_D)}\left[\| \mb x_{i,\natural} + \mb z_i- \mb q \| \le R\right],$$ and we will use $\mc E_i$ to denote this random event. We introduce the following function to help us answer this question:
\begin{eqnarray}
    h(s) &=& \bb P_{g_2,\dots, g_D \sim_{\mr{iid}} \mc N(0,1)} \left[ g_2^2 + \dots + g_D^2 \le \frac{R^2 - s^2  }{\sigma^2} \right]. 
\end{eqnarray}
Given the rotational symmetry of $\mb z_i$, we can align the coordinate system such that $\mb e_1$ is in the direction of $\mb x_{i,\natural}- \mb q$. Then, as shown in figure \ref{fig:h(s) on the manifold}, $h(s)$ describes the probability of $\mb x_i = \mb x_{i,\natural} + \mb z_i$ being grouped into $\mb q$, given \[s = -\ z_i^{(1)} + \|\mb x_{i,\natural} - \mb q\|.\]
We note that $z_i^{(1)}$ ( {\color{red} $\mb z_{\parallel}$} in \ref{fig:h(s) on the manifold}) is independent from $z_i^{(2)},\cdots z_i^{(D)}$, ( {\color{blue} $\mb z_{\perp}$} in \ref{fig:h(s) on the manifold}). This allows us to decompose the noise term and represent $\bb P[\mc E_i]$ as a convolution between a gaussian scalar $z_i^{(1)}$ and $h(s).$ This representation naturally motivates us to derive subgaussian bounds for $h$. Up to reparameterization, $h(s)$ is the cumulative distribution function of a $(D-1)$-dimensional chi-square random variable, and as shown in fig \ref{fig:h_and_hprime_values}, it exhibits a clear phase transition around $s_\star$ where 
\begin{eqnarray}
    \frac{R^2- s_\star^2}{2 \sigma^2} \approx \frac{D-3}{2}.
\end{eqnarray}

This property of $h(s)$ motivates us to develop a subgaussian bound for it around $s_\star.$ Geometrically, $s_\star^2 \approx R^2 - \sigma^2(D-3) $ captures the typical distance between the $\mb q$ and $\mb x_{i,\natural}$, because in high ambient dimensions, $\mb z$ is close to perpendicular to $\mb x_i - \mb q$, which gives $\|\mb x_{i,\natural}-\mb q\|_2^2 \approx \|\mb x_i - \mb q\|_2^2 - \|\mb z_i\|_2^2$. Further,  convolving these tails with the Gaussian kernel preserves subgaussian decay, yielding matching upper and lower bounds for \(\phi*h\) near \(s_\star\). 

Translating this approach back to the language of the manifold, those $\mb x_{i,\natural} \in M$ which satisfies $\|\mb x_{i,\natural} - \mb q\|_2^2 \approx R^2-\sigma^2(D-3)$ contributes the most to our estimation of $\bb P[\mc E_i]$, and the contribution from the rest of $\mc M$ decays at the subgaussian rates.

\section{Discussion, Limitation, and Future Work}

In this work, we analyze the accuracy and sample complexity of local averaging for manifold estimation, providing the first such theoretical analysis in the moderate-to-high noise regime ($\sigma \sqrt{D} \lesssim \tau$) and achieving a final accuracy of $O\left(\sigma\sqrt{d\left(1 + \frac{\kappa\:\mr{diam}(\mc M)}{\log(D)}\right)}\right)$. The key innovation lies in developing a subgaussian bound on the probability that a noisy sample belongs to the ball $B(\mb q, R)$ and related probabilities. This tool enables rigorous analysis beyond the low-noise assumptions made in prior manifold estimation studies -- for example, \cite{yao2023manifold, yao2025manifold} implicitly requires $\sigma e^{cD/d} < 1$. 

As an additional contribution, we also bound the deviation of the local average of noisy samples around $\mb q$ from $\mb q_\natural,$ the projection of $\mb q$ onto the manifold. By setting the landmark itself to a noisy point, $\mb q^0 = \mb x_1 = \mb x_{1,\natural} + \mb z_1,$ since $\| \mb q^0_\natural - \mb x_{1,\natural} \| \lesssim \sigma \sqrt{d}$, and our result also controls the accuracy in estimating $\mb x_{1,\natural}.$ Compared to \cite{trillos2019local}, our results apply to a broader noise regime and exhibit improved scaling with respect to the ambient dimension. Together, these results contribute to a more comprehensive theoretical understanding of the simple and fundamental local averaging methods.

Since the most interesting applications of manifold estimation and denoising involve large noise, an important direction for future work is to consider noise levels $\sigma \sqrt{D} \gg \tau$. This regime is largely unstudied (with the exception of the  \cite{fefferman2023fitting}). Although $\sigma \sqrt{D} \le \tau$ is often regarded in the literature as a critical barrier, there is reason to believe that the fundamental limit lies beyond this threshold. Intuitively, this limit is set not by comparing the norm of the noise to the reach, but rather by comparing the maximum component of the noise along secant directions.\footnote{Fundamental limit is $\mathbb{E}[\Sigma] < \tau$, where $\Sigma = \sup_{\mb v\in S}\langle \mb v, \mb z\rangle  > \tau$, $S = \left \{\frac{\mb x-\mb x'}{\|\mb x-\mb x' \|}, \mb x\ne \mb x'\in \mc M \right \}$ is the secant bundle and $\mb z$ is the noise. Projecting noisy samples back onto the manifold may become unreliable beyond this limit.} As most of the noise energy lies orthogonal to the secant direction, it follows that $\mathbb{E}[\Sigma] \ll \sigma \sqrt{D}$ when $d \ll D$. We believe that the subgaussian bounds developed here will be helpful for approaching the critical  $\mathbb{E}[\Sigma] \approx \tau$ limit, but it would require a finer-grained accounting for the global geometry of the manifold beyond the near-far decomposition used in our proof. An interesting related question is whether {\em robust} local averaging methods will be required for accurate estimation in this regime.

Another promising direction for future work is to leverage our analysis to analyze more sophisticated manifold estimators. For example, our analysis can easily extend beyond single-point analysis to sequences of noisy samples, making it naturally adaptable to streaming data.  Our analysis can also support analysis of stochastic gradient descent (SGD). Specifically, when minimizing the population loss $\bb E_{\mb x}\Bigl[ \, \| \mb x - \mb q \|_2^2 \mid \| \mb x - \mb q \| \le R \, \Bigr]$ using SGD, the negative gradient -- computed using the subset of samples within the ball -- points in the direction of $\bb E_{\mb x}\Bigl[ \, \mb x \mid \| \mb x - \mb q \| \le R \, \Bigr]$. In our analysis, we provide a tight bound on empirical conditional mean  -- that is, the average of points within the ball $B(\mb q, R)$. We believe this will facilitate a more controlled analysis of gradient descent in the context of SGD.

One limitation of this work is that it does not provide analysis on geometric quantities such as tangent space, curvature, or the smoothness of the recovered manifold, which are often the focus of many other related works in manifold estimation. However, we believe our technical tools here can be generalized to analyze such quantities.
Extending the analysis beyond two rounds of local averaging is also a natural direction to explore. Doing so may require new technical innovations, as the subgaussian bound in our analyses breaks down when the landmark $\mb q$ is very close to the manifold, which occurs after two iterations in our setting. Additionally, in Corollary \ref{cor: signal estimation via local averaging}, our analysis could be sharpened by accounting for cancellations in the signal average, due to (approximate) symmetry around $\mb q_\natural$. 

\section*{Acknowledgement} 

The authors gratefully acknowledge support from the National Science Foundation, through the grant NSF 2112085. The authors thank Szabolcs and Zsuzsa Marka (Columbia University) for inspiring discussions around manifold denoising and astrophysics, and Tim Wang (Columbia / Walmart Labs) for discussions on manifold estimation and denoiing. Shiyu Wang also gratefully acknowledges support from the Chiang Chen Industrial fellowship.

\bibliographystyle{unsrt}
\bibliography{refs}

\newpage
\appendix
\onecolumn
\section{Analyzing Landmarking} 

We are motivated by the following minibatch gradient algorithm for landmarking: 

\begin{algorithm}[h]
\caption{$\mathtt{MinibatchSGDLandmarking}$}
    \begin{algorithmic}
    \State {\bf Input:} sequence of noisy samples $\mb x_1, \mb x_2, \mb x_3, \dots $, minibatch sizes $N_{\mr{mb,1}},N_{\mr{mb,2}},\cdots$, sequence of radii $R_1, R_2, \dots$. 
    \State $\mb q_1 \leftarrow \mb x_1$ \Comment{Initialize with the first sample}
    \State $N_1 \leftarrow 1$ \Comment{Number of points assigned to this landmark}
    \State $i \leftarrow 2$ \Comment{Current sample index} 
    \For{minibatch number $\ell = 1, 2, \dots$}
    \State $\mc X_\ell \leftarrow \emptyset$. 
    \While{$|\mc X_\ell| < N_{\mr{mb,\ell}}$} \Comment{Collect minibatch}
    \If{$\| \mb x_i - \mb q_1 \| \le R_\ell$}
    \State $\mc X_{\ell} \leftarrow \mc X_{\ell} \cup \{ \mb x_i \} $
    \EndIf
    \State $i \leftarrow i + 1$
    \EndWhile
    \State $\mb q_1 \leftarrow \frac{N_1}{N_1+N_{\mr{mb,\ell}}} \mb q_1 + \frac{1}{N_1+N_{\mr{mb,\ell}}} \sum_{\mb x \in \mc X_{\ell}} \mb x.$ \Comment{Update landmark with minibatch average}
    \State $N_1 \leftarrow N_1 + N_{\mr{mb,\ell}} $ \Comment{Update number of points assigned to this landmark}
    \EndFor
    \end{algorithmic}
\end{algorithm}

\vspace{.1in}

\noindent We analyze the following simplified version of the above method (the same algorithm as \ref{alg:sl}). This simplified version has the following features (i) it uses two minibatch steps, and (ii) it injects noise after the first minibatch step, in a similar fashion to stochastic gradient Langevin dynamics (SGLD),\footnote{Here, noise injection plays a purely technical role -- it puts a lower bound on the distance between first-round average and the manifold, which is needed to apply our results from subgaussian bounds} (iii) it updates the landmark to the average of current minibatch alone, with no mixing of previous estimates. 

\vspace{.1in} 

\begin{algorithm}[h]
\caption{$\mathtt{SimplifiedLandmarking}$} \label{alg:sl--appendix}
    \begin{algorithmic}
    \State {\bf Input:} sequence of noisy samples $\mb x_1, \mb x_2, \mb x_3, \dots $, minibatch size $N_{\mr{mb,1}},N_{\mr{mb,2}}$, radii $R_1, R_2$.
    \State $\mb q_1^{(0)} \leftarrow \mb x_1$. \Comment{Initialize with the first sample}
    \State $\mc X_1 \leftarrow \emptyset$.
    \State $i \leftarrow 2$.
    \Comment{Collect first minibatch}
    \While{$|\mc X_1| < N_{\mr{mb,1}}$} 
    \If{$\| \mb x_i - \mb q_1^{(0)} \| \le R_1$}
    \State $\mc X_{1} \leftarrow \mc X_{1} \cup \{ \mb x_i \} $
    \EndIf
    \State $i \leftarrow i + 1$
    \EndWhile
    \State $\mb \vartheta \sim_{\mr{iid}} \mc N(0,\sigma^2 D^{-1/4})$
    \State $\mb q_1^{(1)} \leftarrow \tfrac{1}{N_{\mr{mb,1}}} \sum_{\mb x \in \mc X_{1}} \mb x + \mb \vartheta$ \Comment{Update landmark}
    \State $\mc X_2 \leftarrow \emptyset$. \Comment{Collect second minibatch}
    \While{$|\mc X_2| < N_{\mr{mb,2}}$} 
    \If{$\| \mb x_i - \mb q_1^{(1)} \| \le R_2$}
    \State $\mc X_{2} \leftarrow \mc X_{2} \cup \{ \mb x_i \} $
    \EndIf
    \State $i \leftarrow i + 1$
    \EndWhile  
    \State $\mb q_1^{(2)} \leftarrow \tfrac{1}{N_{\mr{mb,2}}} \sum_{\mb x \in \mc X_{2}} \mb x$ \Comment{Final landmark}
    \State {\bf Output}: $\wh{\mb q} = \mb q_1^{(2)}$. 
    \end{algorithmic}
\end{algorithm}

\vspace{.025in} 

\noindent We analyze this landmarking procedure, arguing that with appropriate choices of $R_1, R_2$,  it produces accurate landmarks, satisfying the following bound:

\begin{theorem}[Landmarking Accuracy – Two-Stage Refinement]\label{thm:landmark-accuracy_appendix}
Let $\mathcal M\subset\mathbb R^D$ be a compact, connected, geodesically complete $d$-dimensional submanifold with extrinsic curvature bounded by~$\kappa$, intrinsic diameter 
\[
\mathrm{diam}(\mathcal M)=\sup_{\mb x,\mb y\in\mathcal M}d_{\mathcal M}(\mb x,\mb y),
\]
and set $\bar\kappa=\max\{1,\kappa\}$.  Assume
\[
\kappa\,\mathrm{diam}(\mathcal M)\ge1,\qquad
D^{1/12}\ge C_1\,\max\{\bar\kappa^2\,d,\;\kappa\,\mathrm{diam}(\mathcal M)\,d\},
\]
and the noise level~$\sigma$ satisfies
\[
\sigma\,D\ge1,\quad
\sigma\,D^{20}\ge\frac1\kappa,\quad
\sigma\sqrt{D}\le c_1\,\tau,
\]
where $\tau=\mathrm{reach}(\mathcal M)$.

\smallskip\noindent\textbf{Stage I (Coarse Localization).}  
Run Algorithm~\ref{alg:sl} for one minibatch with
\[
N_{\mathrm{mb},1}\;\ge\;C_2\,
\frac{\log(D)\,\mathrm{diam}(\mathcal M)^2}
     {\sigma^2\bigl(\log(D)+\kappa\,\mathrm{diam}(\mathcal M)\bigr)},
\qquad
R_1^2
=\sigma^2\,(2D - d - 3)
+C_3\,\bar\kappa^2\,\sigma^2\,d
+3C_3\,\bar\kappa\,\sigma^2\,\sqrt{D\,d}.
\]
Then with probability at least $1 - 5e^{-c_2d}$,
\[
d\bigl(\mb q^1,\mathcal M\bigr)
\le
C_4\,\bar\kappa\,\sigma\sqrt{d}
+
C_4\,\sigma\sqrt{\,d\,\bigl(\kappa\,\mathrm{diam}(\mathcal M)+\log(D)\bigr)}.
\]

\smallskip\noindent\textbf{Stage II (Fine Refinement).}  
If in addition 
\[
\frac1\kappa\;\ge\;C_5\,\sigma\sqrt{D\log(D)},
\]
run a second minibatch with
\[
N_{\mathrm{mb},2}
\;\ge\;
C_6\,
\frac{\log^2(D)\,\mathrm{diam}(\mathcal M)^2}
     {\sigma^2\bigl(\log(D)+\kappa\,\mathrm{diam}(\mathcal M)\bigr)},
     \qquad
R_2^2
=\sigma^2\bigl(D - 3 + D^{3/4} + 2C_7\,D^{5/12}\bigr).
\]
Then with probability at least $1 - 9e^{-c_2d}$, the final landmark $\wh{\mb q}$ satisfies
\[
d\bigl(\wh{\mb q},\mathcal M\bigr)
\;\le\;
C\,\sigma\,\sqrt{\,d\,\Bigl(\tfrac{\kappa\,\mathrm{diam}(\mathcal M)}{\log(D)}+1\Bigr)}.
\]
\end{theorem}

\begin{proof} We analyze the three steps of the algorithm:

\vspace{.1in}

\noindent {\em (a) Initialization.} From Lemma \ref{lem:noisy-point distance}, on an event $\mc E_0$ of probability at least $1 - 2 e^{-cd}$, we have 
\begin{equation}  \label{eqn:q0-distance}
    \sigma \sqrt{D-d} - C \bar\kappa \sigma \sqrt{d} \le d(\mb q_1^{(0)}, \mc M ) \le \sigma \sqrt{D-d} + C \bar\kappa \sigma \sqrt{d}. 
\end{equation}

\noindent {\em (b) First Round.} 

We observe that squaring both sides of equation \ref{eqn:q0-distance} gives 
\begin{equation}\label{eqn:sqaure of q0-distance}
-2C\bar{\kappa}\sigma^2\sqrt{Dd} \le d^2(\mb q_1^{(0)}, \mc M ) -\sigma^2(D-d)-C^2{\bar{\kappa}}^2\sigma^2d \le 2C\bar{\kappa}\sigma^2\sqrt{Dd}
\end{equation}
With our choice of 
\begin{equation}
    R_1^2 =  \sigma^2 (2D - d - 3) + C\bar{\kappa}^2\sigma^2d + 3C\bar{\kappa}\sigma^2\sqrt{Dd},
\end{equation}
we have 
\begin{equation}
    {s_\star}^2 = R_1^2 - \sigma^2(D-3) = \sigma^2 (D - d) + C\bar{\kappa}^2\sigma^2d + 3C\bar{\kappa}\sigma^2\sqrt{Dd}
\end{equation}
and 
\begin{equation}\label{eqn: bound on s star, parallel}
    C\bar{\kappa}\sigma^2\sqrt{Dd}
    \le {s_{\star,\parallel}}^2 = {s_\star}^2 - d^2(\mb q_1^{(0)}, \mc M ) 
    \le 5C\bar{\kappa}\sigma^2\sqrt{Dd}
\end{equation}

Write 
\begin{equation}
    \mc X_1 = \Bigl\{ \mb x_{1,1}, \dots, \mb x_{N_{\mr{mb},1}} \, \Bigr\} = \Bigl\{ \mb x_{1,1,\natural} + \mb z_{1,1}, \; \dots, \; \mb x_{N_{\mr{mb},1,\natural}} + \mb z_{N_{\mr{mb}},1} \, \Bigr\}
\end{equation}
We have 
\begin{eqnarray}
    \left\| \bb E\left[ \frac{1}{N_{\mr{mb}}} \sum_{\ell=1}^{N_{\mr{mb}}} \mb z_{\ell,1} \right] \right\| &\le& \left\| \bb E \left[ \mb z \, \middle | \, \mb q_1^{(0)}, \, \left\| \mb x_{\natural} + \mb z - \mb q_1^{(0)} \right\| \le R_1 \,  \right] \right\|_2.
\end{eqnarray}

We will control the $d(\mb q_1^{(1)}, \mc M )$ by using triangular inequality and breaking up the terms into noise and signal average. We will invoke Theorem \ref{thm:noise-size} for the noise average and Theorem \ref{lem:expected-signal-distance} and lemma \ref{lem:signal-avg} for the signal average. 

Theorem \ref{thm:noise-size} and \ref{lem:expected-signal-distance} expects the following conditions
\begin{align}
 \label{eq: upper bound on d(q, M) to apply noise and signal theorem}d\left(\mb q, \mc M\right) &\le \min\{\frac{c_1}{\kappa}, \frac{1}{2}\tau\} \\
 \label{eq: lower bound on s star to apply noise and signal theorem}\max\left\{C_2 \left(\log\left(D \right) + \kappa d \, \mr{diam}(\mc M) + d \log( \frac{1}{\kappa s_{\star,\parallel}} )\right)  \, \sigma^2 D^{2/3} , d^2(\mb q, \mc M )\right\}
    &\le s_\star^2 \\
    \label{eq: upper bound on s star to apply noise and signal theorem}
    s_\star^2 &\le 3 \sigma^2 D \le c \tau^2 \\ 
    \label{eq: complex bound 1 to apply signal theorem}
    \sigma \sqrt{\log \left(\frac{\mr{diam}(\mc M)} {\check{s}_{\star,\parallel}}\right) + d \times \kappa \mr{diam}(\mc M) + d\log\left(\frac{1}{\kappa s_{\star,\parallel}}\right)} &\le c_3\tau \\ 
    \label{eq: complex bound 2 to apply signal theorem}
    \check{s}_{\star,\parallel}^2 +\sqrt{\log \left(\frac{\mr{diam}(\mc M)}{\check{s}_{\star,\parallel}}\right) + d \times \kappa \mr{diam}(\mc M) + d\log\left(\frac{1}{\kappa s_{\star,\parallel}}\right)} \times \sigma^2 \sqrt{D} &\le c_4 \check{s}_{\star}^2 
\end{align}

We now verify the above conditions, 
from upper bound on $d^2(\mb q_1^{(0)}, \mc M )$ from equation \ref{eqn:sqaure of q0-distance} and our assumption that $\tau \ge C\bar\kappa\sigma\sqrt{d}, \tau \ge C\sigma\sqrt{D}$ we have 
\begin{eqnarray}
     d^2(\mb q_1^{(0)}, \mc M ) \le \sigma^2(D-d)+C^2{\bar{\kappa}}^2\sigma^2d +2C\bar{\kappa}\sigma^2\sqrt{Dd}
    \le c \tau^2
\end{eqnarray} so condition \ref {eq: upper bound on d(q, M) to apply noise and signal theorem} is satisfied. Next, by assumption $\sigma\,D^{20}\ge\frac1\kappa$ and bounds on $s_{\star,\parallel}$ from equation \ref{eqn: bound on s star, parallel}, we have
\begin{equation}\label{eqn: round 1 upper bound of 1 over kappa s star parallel}
    d \log( \frac{1}{\kappa s_{\star,\parallel}} ) \le d \log( \frac{\sigma D^{20}}{C\sigma D^{1/4}d^{1/4}} ) = Cd\log(D)
\end{equation}
Together with our assumption $D^{1/12} \ge C\kappa\mr{diam}(\mc M) \times d$ we have 
\begin{equation}\label{eqn: lower bound of s star}
   \left(\log\left(D \right) + \kappa d \, \mr{diam}(\mc M) + d \log( \frac{1}{\kappa s_{\star,\parallel}} )\right)  \, \sigma^2 D^{2/3} \le c\sigma^2 D \le s_\star^2
\end{equation}
along with the observation that $s_{\star, \parallel}^2 = s_\star^2 - d^2(\mb q_1^{(0)}, \mc M ) > C\bar{\kappa}\sigma^2\sqrt{Dd} > 0$, we verify that condition 
\ref {eq: lower bound on s star to apply noise and signal theorem} is satisfied.
For the upper bound \ref{eq: upper bound on s star to apply noise and signal theorem}, given our assumption $D > C'\bar{\kappa}^2d$, we have 
\begin{equation}
    s_\star^2\le \sigma^2 D + C\bar{\kappa}^2\sigma^2d + 3C\bar{\kappa}\sigma^2\sqrt{Dd} \le 3 \sigma^2D
\end{equation} for appropriate choice of $C'$. 
Then, to verify \ref{eq: complex bound 1 to apply signal theorem}, we observe that since $\check{s_{\star, \parallel}^2}= s_{\star, \parallel}^2 - 2\sigma^2 > C'\bar{\kappa}\sigma^2\sqrt{Dd} > cs_{\star, \parallel}^2$
\begin{eqnarray}
     \log \left(\frac{\mr{diam}(\mc M)} {\check{s}_{\star,\parallel}}\right) 
     &\le &
    C \log \left(\frac{\mr{diam}(\mc M)} {s_{\star,\parallel}}\right) \\ 
    &=& C \log \left(\frac{\kappa\mr{diam}(\mc M)} {\kappa s_{\star,\parallel}}\right) \\ 
    &=&\label{eqn: log diam M over s check star parallel}  C \log \left(\kappa\mr{diam}(\mc M)\right) + C \log \left(\frac{1} {\kappa s_{\star,\parallel}}\right)
\end{eqnarray}
and we observe that both terms can be included in later terms of \ref{eq: complex bound 1 to apply signal theorem}. Specifically, $\log \left(\kappa\mr{diam}(\mc M)\right)$ can be dissolved into the $d \times \kappa \mr{diam}(\mc M)$ term, and $\log \left(\frac{1} {\kappa s_{\star,\parallel}}\right)$ can be dissolved into  $d\log\left(\frac{1}{\kappa s_{\star,\parallel}}\right)$. Then, using our previous result \ref{eqn: lower bound of s star}, we have 
\begin{equation}
    d \times \kappa \mr{diam}(\mc M) + d\log\left(\frac{1}{\kappa s_{\star,\parallel}}\right) \le cD^{1/3}
\end{equation}
which means 
\begin{eqnarray}
    \sigma \sqrt{\log \left(\frac{\mr{diam}(\mc M)} {\check{s}_{\star,\parallel}}\right) + d \times \kappa \mr{diam}(\mc M) + d\log\left(\frac{1}{\kappa s_{\star,\parallel}}\right)} \le \sigma\sqrt{CD^{1/3}}\le c_3\tau,
\end{eqnarray} so \ref{eq: complex bound 1 to apply signal theorem} indeed holds. Lastly, for \ref{eq: complex bound 2 to apply signal theorem}, we observe 
since 
\begin{equation}
    {s_{\star,\parallel}}^2
    \le 5C\bar{\kappa}\sigma^2\sqrt{Dd} \le c\sigma^2\sqrt{D}
\end{equation}
and 
\begin{equation}
    \sqrt{\log \left(\frac{\mr{diam}(\mc M)}{\check{s}_{\star,\parallel}}\right) + d \times \kappa \mr{diam}(\mc M) + d\log\left(\frac{1}{\kappa s_{\star,\parallel}}\right)} \le D^{1/6}
\end{equation}
we have 
\begin{eqnarray}
    && \check{s}_{\star,\parallel}^2 +\sqrt{\log \left(\frac{\mr{diam}(\mc M)}{\check{s}_{\star,\parallel}}\right) + d \times \kappa \mr{diam}(\mc M) + d\log\left(\frac{1}{\kappa s_{\star,\parallel}}\right)} \times \sigma^2 \sqrt{D} \\
    &\le& c\sigma^2\sqrt{D} + \sigma^2D^{2/3}\\
     &\le& c_4 \check{s}_{\star}^2 
\end{eqnarray}
Finally, we can use Theorem \ref{thm:noise-size}, which gives us 

\begin{align}
   \left\|  \bb E\Bigl[ \, \mb z_i \mid \| \mb x_{i,\natural} + \mb z_i- \mb q^{(0)}_1 \| \le R_1 \, \Bigr] \right\| \;&\le&\; \frac{C_3 s_\star}{(D-3)^{1/2}} \sqrt{\log\left(D \right) + \kappa d \, \mr{diam}(\mc M) + d \log( \frac{1}{\kappa s_{\star,\parallel}} )} \\
   &\le&\; \frac{C \sigma\sqrt{D}}{(D-3)^{1/2}} \sqrt{\kappa d \, \mr{diam}(\mc M) + d \log( D)} \\
   &\le&\; C' \sigma \sqrt{d\times\left(\kappa\, \mr{diam}(\mc M) + \log(D)\right)} 
\end{align} 
where in the second line we've used previous result \ref{eqn: round 1 upper bound of 1 over kappa s star parallel} to bound $\log( \frac{1}{\kappa s_{\star,\parallel}} )$. 

Since 
\begin{equation}
\|\mb x_{i,\natural} + \mb z_i - \mb q_1^{(0)} \| \le R_1,    
\end{equation}
we have 
\begin{eqnarray}
    \|\mb z_i\| &\le& R_1 + d(\mb q_1^{(0)}, \mc M ) + \|\mb x_{i,\natural} - \mb q_{1,\natural}^{(0)}\| \\
    &\leq& 3\sigma\sqrt D + \sigma\sqrt D + C\bar{\kappa}\sigma\sqrt d + \mr{diam}(\mc M)\\ 
    &\le& 5 \mr{diam}(\mc M)
\end{eqnarray}
as $\mr{diam}(\mc M) \ge \frac{1}{\kappa} \ge \sigma \sqrt D$

The vector Hoeffding inequality \ref{Vector Hoeffding Lemma} gives that 
\begin{equation}
\bb P\left[ \left\| \frac{1}{N_{\mr{mb}}} \sum_{\ell=1}^{N_{\mr{mb}}} \mb z_{\ell} - \bb E\left[ \frac{1}{N_{\mr{mb}}} \sum_{\ell=1}^{N_{\mr{mb}}} \mb z_{\ell} \right] \right\| > t \right] \le D \exp\left( - \frac{ t^2 N_{mb}}{25 * 64 \mr{diam}^2(\mc M)} \right)
\end{equation} 
Setting 
\begin{equation}
    t = C \sigma \sqrt{d\times\left(\kappa\, \mr{diam}(\mc M) + \log(D)\right)}
\end{equation}
and with
\begin{equation}\label{eqn: phase 1 lower bound on N_mb}
    N_{\mr{mb, 1}} \ge C\frac{\log(D)\mr{diam}^2(\mc M)}{\sigma^2 \left(\log(D)+\kappa \mr{diam}(\mc M)\right)}
\end{equation}
we have with probability $> 1- e^{-Cd}$,
\begin{equation}\label{eqn: first round final niose average}
    \left\| \frac{1}{N_{\mr{mb,1}}} \sum_{\ell=1}^{N_{\mr{mb,1}}} \mb z_{\ell}\right\| \le C \sigma \sqrt{d\times\left(\kappa\, \mr{diam}(\mc M) + \log(D)\right)} 
\end{equation}

Meanwhile, equation \ref{eqn: expected value of signal average} from theorem \ref{lem:expected-signal-distance} gives
\begin{align}
    \bb E\Bigl[ \, d_{\mc M}^2(\mb x_\natural,\mb {q^{(0)}}_{1,\natural}) \mid \| \mb x - \mb {q^{(0)}_1}  \| \le R_1 \, \Bigr] &\le &C_2\check{s}_{\star,\parallel}^2 + C_2\sqrt{ d \times \kappa \mr{diam}(\mc M) + d\log\left(\frac{1}{\kappa s_{\star,\parallel}}\right)} \times \sigma^2 \sqrt{D}\\
    &\le &C_2'\bar{\kappa}\sigma^2\sqrt{Dd} + C_2'\sigma^2\sqrt{Dd \times \left(\kappa \mr{diam}(\mc M) + \log(D)\right)} 
\end{align}
    
where in the first line we've dropped the $\log\left(\frac{\mr{diam}(\mc M)}{\check{s}_{\star,\parallel}}\right)$ term based on previous argument \ref{eqn: log diam M over s check star parallel}. 
Similarly, equation \ref{eqn: variance of signal average} from theorem \ref{lem:expected-signal-distance} gives
\begin{align}
    \bb E\Bigl[ \, d_{\mc M}^4(\mb x_\natural,\mb {q^{(0)}}_{1,\natural}) \mid \| \mb x - \mb {q^{(0)}_1}  \| \le R_1 \, \Bigr] &\le &C_3\check{s}_{\star,\parallel}^4 + C_3\left(d \times \kappa \mr{diam}(\mc M) + d\log\left(\frac{1}{\kappa s_{\star,\parallel}}\right)\right) \times \sigma^4 D\\\label{eqn: variance of signal term in round 1}
    &\le &C_3'\bar{\kappa}^2\sigma^4 Dd + C_3'\sigma^4Dd \times \left(\kappa \mr{diam}(\mc M) + \log(D)\right)
\end{align}
by Bernstein's inequality (Theorem 2.8.4 from \cite{vershynin2018high}), we know for any  \(X_1,\dots,X_N\) independent, mean–zero random variables satisfying \(\lvert X_i\rvert\le K\) almost surely, when we let $ 
\sigma^2 \;=\;\sum_{i=1}^N \bb E[X_i^2]$, then, for every \(t\ge0\),
\begin{equation}\label{eqn: Berstein inequality}
\Pr\!\Bigl(\sum_{i=1}^N X_i \ge t\Bigr)
\;\le\;
2\exp\!\Bigl(
-\frac{t^2}{2\,\sigma^2 \;+\;\tfrac23\,K\,t}
\Bigr),
\end{equation}
so in our setting, we consider $X_i$ = $d_{\mc M}^2(\mb x_\natural,\mb {q^{(0)}}_{1,\natural}) - \bb E\Bigl[\:d_{\mc M}^2(\mb x_\natural,\mb {q^{(0)}}_{1,\natural})\:\Bigr]$ conditioned on event $\| \mb x - \mb {q^{(0)}_1}  \| \le R_1$, then, using equation \ref{eqn: variance of signal term in round 1} to bound $\bb E[X_i^2]$ and the fact that $\left|d_{\mc M}^2(\mb x_\natural,\mb {q^{(0)}}_{1,\natural})\right| \le \mr{diam}^2(\mc M)$ almost surely, 
\begin{eqnarray}
    &&\quad \bb P\left[ \frac{1}{N_{\mr{mb}}} \sum_{\ell=1}^{N_{\mr{mb}}} d^2_{\mc M}\Bigl(\mb x_{\ell,1,\natural}, \mb q^{(0)}_{1,\natural} \Bigr) > \bb E \left[ \frac{1}{N_{\mr{mb}}} \sum_{\ell=1}^{N_{\mr{mb}}} d^2_{\mc M}\Bigl(\mb x_{\ell,1,\natural}, \mb q^{(0)}_{1,\natural} \Bigr) \right] + t \right] \\
    &&\label{eqn: phase 1signal average prob bound, raw expression}\le  2\exp\left(\frac{-ct^2 N_{mb}}{\bar{\kappa}^2\sigma^4 Dd + \sigma^4Dd(\kappa \mr{diam}(\mc M) + \log(D)) + t\mr{diam}^2(\mc M)}\right)
\end{eqnarray}
And so with 
\begin{equation}
    t = c'\left(\bar{\kappa}\sigma^2\sqrt{Dd} + \sigma^2\sqrt{Dd \times \left(\kappa \mr{diam}(\mc M) + \log(D)\right)}\right)
\end{equation} 
we note $t \le \sigma^2 D\le \mr{diam}^2(\mc M),$ so the denominator of equation \ref{eqn: phase 1signal average prob bound, raw expression} can be upper bounded by $2 t\:\mr{diam}^2(\mc M).$ 
With the same $N_{\mr{mb}}$ as required by finite sample bound from noise average (equation \ref{eqn: phase 1 lower bound on N_mb}), 
\begin{eqnarray}
    &&\quad \bb P\left[ \frac{1}{N_{\mr{mb}}} \sum_{\ell=1}^{N_{\mr{mb}}} d^2_{\mc M}\Bigl(\mb x_{\ell,1,\natural}, \mb q^{(0)}_{1,\natural} \Bigr) > \bb E \left[ \frac{1}{N_{\mr{mb}}} \sum_{\ell=1}^{N_{\mr{mb}}} d^2_{\mc M}\Bigl(\mb x_{\ell,1,\natural}, \mb q^{(0)}_{1,\natural} \Bigr) \right] + t \right] \\ &\le&\label{eqn: phase i finite sample intermediate bound}2\exp\left(\frac{-ct N_{mb}}{\mr{diam}^2(\mc M)}\right)\\
    &\le& 2\exp\left(\frac{-c\left(\sigma^2\sqrt{Dd \times \left(\kappa \mr{diam}(\mc M) + \log(D)\right)}\right)\left(\frac{\log(D)\mr{diam}^2(\mc M)}{\sigma^2 \left(\log(D)+\kappa \mr{diam}(\mc M)\right)}\right)}{\mr{diam}^2(\mc M)}\right) \\
    &=& 2\exp\left(-\frac{c\log(D)\sqrt{D\times d}}{ \sqrt{\log(D)+\kappa \mr{diam}(\mc M)}}\right) \\
    &\le& 2\exp\left(-\frac{C\log(D)\sqrt{d\:\kappa \mr{diam} \times d}}{ \sqrt{\log(D)+\kappa \mr{diam}(\mc M)}}\right)\\
    &\le& 2 \exp\left(-C d\right)
\end{eqnarray}
Thus, with probability $> 1- 2e^{-Cd}$ 
\begin{equation}\label{eqn: state i bound on expected squared  distance along the manifold}
    \frac{1}{N_{\mr{mb}}} \sum_{\ell=1}^{N_{\mr{mb}}} d^2_{\mc M}\Bigl(\mb x_{\ell,1,\natural}, \mb q^{(0)}_{1,\natural} \Bigr) \le C\bar{\kappa}\sigma^2\sqrt{Dd} + C'\sigma^2\sqrt{Dd \times \left(\kappa \mr{diam}(\mc M) + \log(D)\right)} 
\end{equation}

Lemma \ref{lem:signal-avg} further gives
\begin{eqnarray}
    d\left( \frac{1}{N_{\mr{mb}}} \sum_{\ell=1}^{N_{\mr{mb}}} \mb x_{1,\natural,\ell} ,\mc M\right) &\le& \kappa\left(C\bar{\kappa}\sigma^2\sqrt{Dd} + C'\sigma^2\sqrt{Dd \times \left(\kappa \mr{diam}(\mc M) + \log(D)\right)} \right)\\
    \label{eqn: first round final signal avg}
    &\le& C\bar{\kappa}\sigma\sqrt{d} + C'\sigma\sqrt{d \times \left(\kappa \mr{diam}(\mc M) + \log(D)\right)}
\end{eqnarray}
Finally, we can use triangular inequality to combine results from \ref{eqn: first round final niose average} and \ref{eqn: first round final signal avg}, which yields 
{
\begin{eqnarray}
 d\left( \frac{1}{N_{\mr{mb}}} \sum_{\ell=1}^{N_{\mr{mb}}} \mb x_{1,\ell} ,\mc M\right)
&\le& C\bar{\kappa}\sigma\sqrt{d} + C'\sigma \sqrt{d\times\left(\kappa\, \mr{diam}(\mc M) + \log(D)\right)} 
\end{eqnarray}
}

\vspace{.1in} 

 \noindent {\em (c) Second Round.} The idea is to apply the same tools from  Theorem \ref{thm:noise-size} and \ref{lem:expected-signal-distance}, but we will first need to choose the correct accepting radius $R_2.$ We let 
\begin{equation}
    \bar{\mb x_1} = \frac{1}{N_{\mr{mb}}} \sum_{\ell=1}^{N_{\mr{mb}}} \mb x_{1,\ell}
\end{equation}
Consider $\mb \vartheta \sim_{\mr{iid}} \mc N(0,\sigma^2 D^{-1/4})$, and we let \[\mb q(t) = \bar{\mb x_1} + t \mb \vartheta,\quad \text{and} \quad \mb \gamma(t) = \mc P_{\mc M} \mb q(t).\] so $\mb q(1) = \mb q_1^{(1)}$ in algorithm \ref{alg:sl--appendix}(initialization for round 2 averaging). We first need to develop a high probability bound for \[d(\mb q(1), \mc M ) = \|\mb q(1) - \mb\gamma (1)\|_2.\] From bounds on chi square random variable (equation 2.19 in \cite{wainwright2019high}), we know that with probability $>1 - 2e^{cd}$, 
\begin{equation}
    \sigma^2D^{3/4} - 3\sigma^2D^{1/4}d^{1/2} \le \|\mb \vartheta\|_2^2 \le \sigma^2D^{3/4} + 3\sigma^2D^{1/4}d^{1/2}
\end{equation}
which also implies 
\begin{equation}
    \|\mb\vartheta\|_2 \le 2\sigma D^{3/8} \le \frac{1}{6}\tau
\end{equation}
which ensures 
\begin{equation}
    \forall 0 \le t \le 1, \:d(\mb q(t), \mc M ) \le d(\mb q(0), \mc M )+\|\mb \vartheta\|_2 \le \frac{1}{2}\tau.
\end{equation}
From Theorem C in \cite{leobacher2021existence}, for x within $\tau$ neighborhood from manifold $\mc M,$ the derivative of its projection $p(x)$ is given by 
\[
  Dp(x)
  = \bigl(\mathrm{id}_{T_{p(x)}\mc M}
        - \|x - p(x)\|\,L_{p(x),v}\bigr)^{-1}
    \,P_{T_{p(x)}\mc M},
\]
where 
\[ 
  v = \|x - p(x)\|^{-1}\,(x - p(x))
\quad\text{and}\quad
  L_{p(x),v}
\]
is the shape operator in direction \(v\) at \(p(x)\). Under our noise regime, let $\sff$ be the second fundamental form at $\mb \gamma(t)$ and $\sff^*$ be the symmetric operator defined by $\left(\sff^*\mb\eta([\mb p, \mb q]\right)\:=\: \langle \sff(\mb p, \mb q), \mb\eta\rangle,$ we have 
\begin{eqnarray}
       \|\frac{d}{dt} \mc P_{\mc M} \mb q(t) \|
       &=& \|\sum_{k=0}^\infty \left( \sff^*[\mb q(t) - \mb \gamma(t)] \right)^k \mc P_{T_{\mb\gamma(t)}\mc M}\: (1-t)\:\mb \vartheta\|,\\ 
       &\le& \sum_{k=0}^\infty \|\left( \sff^*[\mb q(t) - \mb \gamma(t)] \right)^k \mc P_{T_{\mb\gamma(t)}\mc M}\:\mb \vartheta\|,\\
       &\le& \frac{1}{ 1 - 3 \kappa \| \mb q(t) - \mb \gamma(t) \|_2} \left\| \mc P_{T_{\mb \gamma(t)}\mc M} \mb \vartheta \right\|, \\
       &\le& \frac{1}{ 1 - 3 \kappa \| \mb \vartheta \|_2} \left\| \mc P_{T_{\mb \gamma(t)}\mc M} \mb \vartheta \right\|, \\
       &\le& 2 \left\| \mc P_{T_{\mb \gamma(t)}\mc M} \mb \vartheta \right\|
\end{eqnarray}
which implies that 
\begin{eqnarray}
    \|\mb\gamma(1) - \mb\gamma(0)\| 
    &=& \|\int_0^1 \left[\frac{d}{ds} \mc P_{\mc M} \mb q(s) \right] \Bigr|_{s = t} dt \|\\
    &\le& 2 \sup_{0\le t\le 1}\left\| \mc P_{T_{\mb \gamma(t)}\mc M} \mb \vartheta \right\|\\ 
    &\le&\label{eqn: substitution with T_max} 2 T_{\max, \vartheta}
\end{eqnarray}
with 
\begin{equation}
T_{\max,\vartheta} = \sup_{\bar{\mb x} \in B_{\mc M}(\mb \gamma(0),1/\kappa), \mb v \in T_{\bar{\mb x}}\mc M} \innerprod{\mb v}{\mb \vartheta}.
\end{equation}
where in equation \ref{eqn: substitution with T_max} we've used that $\|\mb\gamma(t)-\mb\gamma(0)\|_2^2\le \|\mb\vartheta\|_2^2\le\frac{\tau^2}{36}\le\frac{1}{36\kappa^2} $ for any $t,$ which along with Lemma B.5 from \cite{wang2025fast}, implies $d_{\mc M}\left(\mb \gamma(t),\mb\gamma(0)\right) \le 1/\kappa,\:\forall t.$
{ From Lemma A.4 from \cite{wang2025fast}, we have with probability $\ge 1 - e^{-9d/2}$
\begin{equation}
    T_{\max, \vartheta}  \le C \bar\kappa \sigma D^{-1/2} \sqrt{d}, 
\end{equation}}
Therefore, we have 
\begin{align}
    d(\mb q(1), \mc M ) &= \|\mb q(1) - \mb\gamma (1)\| \\ 
    &= \|\mb q(0) -\mb \gamma(0) +\mb \gamma(0)+ \mb \vartheta - \mb\gamma (1)\|\\
    &\ge \|\mb \vartheta \| - \|\mb q(0) -\mb \gamma(0)\| - \|\mb\gamma(1) - \mb\gamma(0)\|  \\
    &\ge \|\mb \vartheta \| - C\bar{\kappa}\sigma\sqrt{d} - C'\sigma \sqrt{d\times\left(\kappa\, \mr{diam}(\mc M) + \log(D)\right)} - C'' \bar\kappa \sigma D^{-1/2} \sqrt{d}\\
    &\ge \|\mb \vartheta \| - C'''\bar{\kappa}\sigma\sqrt{d} - C'\sigma \sqrt{d\times\left(\kappa\, \mr{diam}(\mc M) + \log(D)\right)}\\
    &\ge \|\mb \vartheta \| - C\sigma\sqrt{D^{1/12}} \\ 
    &\ge \|\mb \vartheta \| - C\sigma D^{1/24}
\end{align}
where in the second to last line we've invoked our assumption that $D^{1/12} \ge C\max\{\bar{\kappa}^2d,\,\kappa\mr{diam}(\mc M)d\}.$ 
Taking the square gives us 
\begin{eqnarray}
     d^2(\mb q(1), \mc M )
     &\ge& \|\mb \vartheta \|_2^2 - 2C\|\mb \vartheta \|_2 \sigma D^{1/24}\\ 
     &\ge& 
     \sigma^2D^{3/4} - 3\sigma^2D^{1/4}d^{1/4} - 4C\sigma D^{3/8} \sigma D^{1/24}\\
     &\ge& 
     \sigma^2D^{3/4} - C'\sigma^2 D^{5/12}
\end{eqnarray}

Similarly we have the following upper bound  
\begin{equation}
    d(\mb q(1), \mc M ) \le  |\mb \vartheta \| + C\sigma D^{1/24}
\end{equation}
which means 
\begin{eqnarray}
     d^2(\mb q(1), \mc M )
     &\le& \|\mb \vartheta \|_2^2 + 2C\|\mb \vartheta \|\sigma D^{1/24} + \sigma^2 D^{1/12}\\
     &\le& 
     \sigma^2D^{3/4} + \sigma^2D^{1/4}d^{1/4}+ C'\sigma^2 D^{5/12}\\
     &\le& 
     \sigma^2D^{3/4} + C''\sigma^2 D^{5/12}
\end{eqnarray}
Therefore, with the appropriate choice of $C$, we have 
\begin{equation}
    \sigma^2D^{3/4} - C\sigma^2 D^{5/12} \le d^2(\mb q(1), \mc M ) \le \sigma^2D^{3/4} + C\sigma^2 D^{5/12}
\end{equation}
setting 
\begin{equation}
    R_2^2 = \sigma^2(D-3 +D^{3/4} + 2CD^{5/12})
\end{equation}
we would have corresponding 
\begin{equation}
    s_\star^2 = \sigma^2(D^{3/4} + 2CD^{5/12})
\end{equation}
and 
\begin{equation}
    C \sigma^2D^{5/12}\le s_{\star, \parallel}^2 \le 3C \sigma^2D^{5/12}
\end{equation}
We now need to verify  
Eqs.~\eqref{eq: upper bound on d(q, M) to apply noise and signal theorem}
--\eqref{eq: complex bound 2 to apply signal theorem}, which are the requirements for invoking Theorem \ref{thm:noise-size} and \ref{lem:expected-signal-distance}

We have $d\left(\mb q, \mc M\right) \le 2\sigma D^{3/8} \le c\tau$ trivially, and given the above lower bound on $s_{\star,\parallel}$, we have $\log( \frac{1}{\kappa s_{\star,\parallel}} )\le C\log(D)$ as before. With the assumption $\kappa d \, \mr{diam}(\mc M) \le D^{1/12}$  we further have 
\begin{equation}
    \left(\log\left(D \right) + \kappa d \, \mr{diam}(\mc M) + d \log( \frac{1}{\kappa s_{\star,\parallel}} )\right)  \, \sigma^2 D^{2/3} 
    \le 
    c\sigma^2 D^{1/12+2/3} \le s_\star^2
\end{equation}

Similar to round 1,  will apply Theorem \ref{thm:noise-size}, and with the new $s_\star^2,\,s_{\star,\parallel}^2$ of round two, it gives us 
\begin{align}
   \left\|  \bb E\Bigl[ \, \mb z_i \mid \| \mb x_{i,\natural} + \mb z_i- \mb q^{(1)}_1 \| \le R_2 \, \Bigr] \right\| \;&\le\; \frac{C_3 s_\star}{(D-3)^{1/2}} \sqrt{\log\left(D \right) + \kappa d \, \mr{diam}(\mc M) + d \log( \frac{1}{\kappa s_{\star,\parallel}} )} \\
   &\le\; C \frac{\sigma}{D^{1/8}}\sqrt{d\times\left(\kappa\, \mr{diam}(\mc M) + \log(D)\right)} 
\end{align}

The vector Hoeffding inequality \ref{Vector Hoeffding Lemma} gives that 
\begin{equation}
\bb P\left[ \left\| \frac{1}{N_{\mr{mb}}} \sum_{\ell=1}^{N_{\mr{mb}}} \mb z_{\ell} - \bb E\left[ \frac{1}{N_{\mr{mb}}} \sum_{\ell=1}^{N_{\mr{mb}}} \mb z_{\ell} \right] \right\| > t \right] \le D \exp\left( - \frac{ t^2 N_{mb}}{25 * 64 \mr{diam}^2(\mc M)} \right)
\end{equation} 
Setting 
\begin{equation}
    t = C \sigma \sqrt{d\times\left(\frac{\kappa\, \mr{diam}(\mc M)}{\log(D)} + 1\right)}
\end{equation}
and with
\begin{equation}\label{eqn: phase 2 lower bound on N_mb}
    N_{mb} \ge C\frac{\log(D)\mr{diam}^2(\mc M)}{\sigma^2\left(\frac{\kappa\, \mr{diam}(\mc M)}{\log(D)} + 1\right)}
\end{equation}
we have with probability $> 1- e^{-Cd}$,
\begin{equation}\label{eqn: second round final niose average}
    \left\| \frac{1}{N_{\mr{mb}}} \sum_{\ell=1}^{N_{\mr{mb}}} \mb z_{\ell}\right\| \le C \sigma \sqrt{d\times\left(\frac{\kappa\, \mr{diam}(\mc M)}{\log(D)} + 1\right)} 
\end{equation}

Meanwhile, with the same arguments as in round 1, equations \ref{eqn: expected value of signal average} and \ref{eqn: variance of signal average} from theorem \ref{lem:expected-signal-distance} gives
\begin{align}
    \bb E\Bigl[ \, d_{\mc M}^2(\mb x_\natural,\mb {q^{(2)}}_{1,\natural}) \mid \| \mb x - \mb {q^{(2)}_1}  \| \le R_2 \, \Bigr] &\le C_2\check{s}_{\star,\parallel}^2 + C_2\sqrt{ d \times \kappa \mr{diam}(\mc M) + d\log\left(\frac{1}{\kappa s_{\star,\parallel}}\right)} \times \sigma^2 \sqrt{D}\\
    &\le C_2'\sigma^2D^{5/12} + C_2'\sigma^2\sqrt{Dd \times \left(\kappa \mr{diam}(\mc M) + \log(D)\right)} 
\end{align}
and 
\begin{align}
    \bb E\Bigl[ \, d_{\mc M}^4(\mb x_\natural,\mb {q^{(1)}}_{1,\natural}) \mid \| \mb x - \mb {q^{(1)}_1}  \| \le R_2 \, \Bigr] &\le C_3\check{s}_{\star,\parallel}^4 + C_3\left(d \times \kappa \mr{diam}(\mc M) + d\log\left(\frac{1}{\kappa s_{\star,\parallel}}\right)\right) \times \sigma^4 D\\\label{eqn: variance of signal term in round 2}
    &\le C_3'\sigma^4 D^{5/6} + C_3'\sigma^4Dd \times \left(\kappa \mr{diam}(\mc M) + \log(D)\right)
\end{align}
For the finite sample bounds, we will apply the same treatment with Bernstein inequality as in round 1. 
To apply equation \ref{eqn: Berstein inequality}, we will use  equation \ref{eqn: variance of signal term in round 2} to bound the expected second moment and the fact that $\left|d_{\mc M}^2(\mb x_\natural,\mb {q^{(0)}}_{1,\natural})\right| \le \mr{diam}^2(\mc M)$ almost surely, which gives us 
\begin{eqnarray}
    &&\quad \bb P\left[ \frac{1}{N_{\mr{mb}}} \sum_{\ell=1}^{N_{\mr{mb}}} d^2_{\mc M}\Bigl(\mb x_{\ell,2,\natural}, \mb q^{(1)}_{1,\natural} \Bigr) > \bb E \left[ \frac{1}{N_{\mr{mb}}} \sum_{\ell=1}^{N_{\mr{mb}}} d^2_{\mc M}\Bigl(\mb x_{\ell,2,\natural}, \mb q^{(1)}_{1,\natural} \Bigr) \right] + t \right] \\
    &&\label{eqn: phase 2signal average prob bound, raw expression}\le  2\exp\left(\frac{-ct^2 N_{mb}}{\sigma^4 D^{5/6} + \sigma^4Dd \times \left(\kappa \mr{diam}(\mc M) + \log(D)\right) + t\mr{diam}^2(\mc M)}\right)
\end{eqnarray}
And so with 
\begin{equation}
    t = c'\left(\sigma^2 D^{5/12} + \sigma^2\sqrt{Dd \times \left(\kappa \mr{diam}(\mc M) + \log(D)\right)}\right)
\end{equation} 
we note $t \le \sigma^2 D\le \mr{diam}^2(\mc M),$ so the denominator of equation \ref{eqn: phase 2signal average prob bound, raw expression} can be upper bounded by $2 t\:\mr{diam}^2(\mc M).$ 
With the same $N_{mb}$ as required by finite sample bound from noise average in phase ii (equation \ref{eqn: phase 2 lower bound on N_mb}), 
\begin{eqnarray}
    &&\quad \bb P\left[ \frac{1}{N_{\mr{mb}}} \sum_{\ell=1}^{N_{\mr{mb}}} d^2_{\mc M}\Bigl(\mb x_{\ell,2,\natural}, \mb q^{(1)}_{1,\natural} \Bigr) > \bb E \left[ \frac{1}{N_{\mr{mb}}} \sum_{\ell=1}^{N_{\mr{mb}}} d^2_{\mc M}\Bigl(\mb x_{\ell,2,\natural}, \mb q^{(1)}_{1,\natural} \Bigr) \right] + t \right] \\ &\le&\label{eqn: phase ii finite sample intermediate bound} 2\exp\left(\frac{-ct N_{mb}}{\mr{diam}^2(\mc M)}\right)\\
    &\le& 2\exp\left(\frac{-c\left(\sigma^2\sqrt{Dd \times \left(\kappa \mr{diam}(\mc M) + \log(D)\right)}\right)\left(\frac{\log^2(D)\mr{diam}^2(\mc M)}{\sigma^2 \left(\log(D)+\kappa \mr{diam}(\mc M)\right)}\right)}{\mr{diam}^2(\mc M)}\right) \\
    &\le& 2 \exp\left(-C d\right)
\end{eqnarray}
where in the second to last line we've applied the lower bound $t\ge c''\sigma^2\sqrt{Dd \times \left(\kappa \mr{diam}(\mc M) + \log(D)\right)}.$
Thus, with probability $> 1- 2e^{-Cd}$ 
\begin{equation}
    \frac{1}{N_{\mr{mb}}} \sum_{\ell=1}^{N_{\mr{mb}}} d^2_{\mc M}\Bigl(\mb x_{\ell,2,\natural}, \mb q^{(1)}_{1,\natural} \Bigr) \le C\sigma^2 D^{5/12} + C'\sigma^2\sqrt{Dd \times \left(\kappa \mr{diam}(\mc M) + \log(D)\right)} 
\end{equation}

Lemma \ref{lem:signal-avg} further gives
\begin{eqnarray}
    d\left( \frac{1}{N_{\mr{mb}}} \sum_{\ell=1}^{N_{\mr{mb}}} \mb x_{2,\natural,\ell} ,\mc M\right) &\le& \kappa\left(C\sigma^2 D^{5/12} + C'\sigma^2\sqrt{Dd \times \left(\kappa \mr{diam}(\mc M) + \log(D)\right)}\right)\\
    \label{eqn: second round final signal avg}
    &\le& C\sigma\sqrt{d \times \left(\frac{\kappa \mr{diam}(\mc M)}{\log(D)} + 1\right)}
\end{eqnarray}
Finally, we can use triangular inequality to combine results from \ref{eqn: second round final niose average} and \ref{eqn: second round final signal avg}, which yields 

\begin{eqnarray}
 d(\mb q_1^{(2)}, \mc M ) = d\left( \frac{1}{N_{\mr{mb}}} \sum_{\ell=1}^{N_{\mr{mb}}} \mb x_{2,\ell} ,\mc M\right)
&\le& C\sigma\sqrt{d \times \left(\frac{\kappa \mr{diam}(\mc M)}{\log(D)} + 1\right)}
\end{eqnarray}

\end{proof}

\begin{corollary}[Signal Estimation via Local Averaging]\label{cor: signal estimation via local averaging-appendix}
Under the same assumptions as Theorem~\ref{thm:landmark-accuracy_appendix} and with the same accepting radius $R_1$ and minibatch size $N_{\mr{mb}}$, let $\mb q^{(0)}_1 $ be the coarse landmark in Stage I and 
$
\mb {q^{(0)}}_{1,\natural}$
be its projection onto $\mc M$. Then there exist constants \(C_1,C_2>0\) such that
with probability at least $1-3e^{-C_1d}$,
the local average's distance from the original clean signal is bounded by
\begin{equation}
    \left\|\frac{1}{N_{\mathrm{mb}}}
\sum_{\ell=1}^{N_{\mathrm{mb}}}
\mb x_{\ell}-\,\mb {q^{(0)}}_{1,\natural}\right\|
\le 
C_2\sigma\sqrt{\bar{\kappa}}(Dd)^{1/4}+C_2\sigma\left(Dd\:\left(\kappa\mr{diam}(\mc M) + \log(D)\right)\right)^{1/4}
\end{equation}

\end{corollary}
\begin{proof}

From equation \ref{eqn: state i bound on expected squared  distance along the manifold}, we have with probability $> 1- 2e^{-Cd}$ 
\begin{equation}
    \frac{1}{N_{\mr{mb}}} \sum_{\ell=1}^{N_{\mr{mb}}} d^2_{\mc M}\Bigl(\mb x_{\ell,1,\natural}, \mb q^{(0)}_{1,\natural} \Bigr) \le C\bar{\kappa}\sigma^2\sqrt{Dd} + C'\sigma^2\sqrt{Dd \times \left(\kappa \mr{diam}(\mc M) + \log(D)\right)} 
\end{equation}
Therefore, by Cauchy-Schwarz inequality, 
\begin{align}\paren{\sum_{\ell=1}^{N_{\mr{mb}}} d_{\mc M}\Bigl(\mb x_{\ell,1,\natural}, \mb q^{(0)}_{1,\natural} \Bigr)}^2 
    &\le N_{\mr mb}\times \sum_{\ell=1}^{N_{\mr{mb}}} d^2_{\mc M}\Bigl(\mb x_{\ell,1,\natural}, \mb q^{(0)}_{1,\natural} \Bigr)\\
    &\le {N_{\mr mb}}^2\paren{C\bar{\kappa}\sigma^2\sqrt{Dd} + C'\sigma^2\sqrt{Dd \times \left(\kappa \mr{diam}(\mc M) + \log(D)\right)}}.
\end{align}
Taking the square root gives
\begin{equation}
    \frac{1}{N_{\mr{mb}}} \sum_{\ell=1}^{N_{\mr{mb}}} d_{\mc M}\Bigl(\mb x_{\ell,1,\natural}, \mb q^{(0)}_{1,\natural} \Bigr) \le 2C\sigma\sqrt{\bar{\kappa}}(Dd)^{1/4}+2C'\sigma\left(Dd\:\left(\kappa\mr{diam}(\mc M) + \log(D)\right)\right)^{1/4}
\end{equation}
From equation 
\ref{eqn: first round final niose average}, we further have with probability $> 1- e^{-Cd}$
\begin{equation}
    \left\| \frac{1}{N_{\mr{mb}}} \sum_{\ell=1}^{N_{\mr{mb}}} \mb z_{\ell}\right\| \le C \sigma \sqrt{d\times\left(\kappa\, \mr{diam}(\mc M) + \log(D)\right)} 
\end{equation}
Lastly, by triangular inequality, we have 
\begin{eqnarray}
    \left\|\frac{1}{N_{\mathrm{mb}}}
\sum_{\ell=1}^{N_{\mathrm{mb}}}
\mb x_{\ell}-\,\mb {q^{(0)}}_{1,\natural}\right\| 
&\le& \left\|\frac{1}{N_{\mathrm{mb}}}
\sum_{\ell=1}^{N_{\mathrm{mb}}}
\mb x_{\ell,\natural}-\,\mb {q^{(0)}}_{1,\natural}\right\| + \left\|\frac{1}{N_{\mathrm{mb}}}
\sum_{\ell=1}^{N_{\mathrm{mb}}}
\mb x_{\ell}-\,\mb x_{\ell,\natural}\right\| 
\\
&\le& \frac{1}{N_{\mathrm{mb}}}\left\|
\sum_{\ell=1}^{N_{\mathrm{mb}}}
\mb x_{\ell,\natural}-\,\mb {q^{(0)}}_{1,\natural}\right\| + \left\|\frac{1}{N_{\mathrm{mb}}}
\sum_{\ell=1}^{N_{\mathrm{mb}}}
\mb x_{\ell}-\,\mb x_{\ell,\natural}\right\| 
\\
&\le& \frac{1}{N_{\mr{mb}}} \sum_{\ell=1}^{N_{\mr{mb}}} d_{\mc M}\Bigl(\mb x_{\ell,1,\natural}, \mb q^{(0)}_{1,\natural} \Bigr)  + \left\|\frac{1}{N_{\mathrm{mb}}}
\sum_{\ell=1}^{N_{\mathrm{mb}}}
\mb x_{\ell}-\,\mb x_{\ell,\natural}\right\| 
\\
&\le& 
C\sigma\sqrt{\bar{\kappa}}(Dd)^{1/4}+C\sigma\left(Dd\:\left(\kappa\mr{diam}(\mc M) + \log(D)\right)\right)^{1/4}
\end{eqnarray}

\end{proof}

\subsection*{Framework of Analysis} 
Our goal in the subsequent sections is to provide high probability bounds on the distance between the minibatch average 
\begin{equation}
\frac{1}{N_{\mr{mb}}} \sum_{\mb x \in \mc X_{\ell}} \mb x \label{eqn:mba}
\end{equation} 
at iteration $\ell$ and the manifold $\mc M$. Below, let 
\begin{equation}
\mb q_{1\natural} = \mc P_{\mc M}[ \mb q_1 ]
\end{equation}
be the nearest point on the manifold $\mc M$ to the landmark $\mb q_1$ at iteration $\ell$. Each sample in the $\ell$-th minibatch $\mc X_{\ell}$ can be expressed as $\mb x = \mb x_\natural + \mb z$, and so 
\begin{equation}
\frac{1}{N_{\mr{mb}}} \sum_{\mb x \in \mc X_{\ell}} \mb x \;=\;  \underset{\color{red} \text{\bf Signal average}}{ \frac{1}{N_{\mr{mb}}} \sum_{\mb x = \mb x_\natural + \mb z \in \mc X_{\ell}} \mb x_\natural } \;+\; \underset{\color{blue}\text{\bf Noise average}}{  \frac{1}{N_{\mr{mb}}} \sum_{\mb x = \mb x_\natural + \mb z \in \mc X_{\ell}} \mb z } \label{eqn:decomp}
\end{equation} 
\begin{equation}
d\left(  \frac{1}{N_{\mr{mb}}} \sum_{\mb x = \mb x_\natural + \mb z \in \mc X_{\ell}} \mb x_\natural, \; \mc M \right) \le \frac{\kappa}{N_{\mr{mb}}} \sum_{\mb x = \mb x_\natural + \mb z \in \mc X_{\ell}} d^2_{\mc M}( \mb x_\natural, \mb q_{1\natural} ),
\end{equation} 
and so the distance of the minibatch average to the manifold satisfies 
\begin{equation}
d\left(  \frac{1}{N_{\mr{mb}}} \sum_{\mb x = \mb x_\natural + \mb z \in \mc X_{\ell}} \mb x, \; \mc M \right) \le \frac{\kappa}{N_{\mr{mb}}} \sum_{\mb x = \mb x_\natural + \mb z \in \mc X_{\ell}} d^2_{\mc M}( \mb x_\natural, \mb q_{1\natural} ) + \left\| \frac{1}{N_{\mr{mb}}} \sum_{\mb x = \mb x_\natural + \mb z \in \mc X_{\ell}} \mb z \right\|
\end{equation}
The right hand side of this inequality involves two sums of independent random variables -- the {\em noise average} 
\begin{equation}
 \frac{1}{N_{\mr{mb}}} \sum_{\mb x = \mb x_\natural + \mb z \in \mc X_{\ell}} \mb z
\end{equation} 
and the {\em average squared intrinsic distance}
\begin{equation}
\frac{1}{N_{\mr{mb}}} \sum_{\mb x = \mb x_\natural + \mb z \in \mc X_{\ell}} d^2_{\mc M}( \mb x_\natural, \mb q_{1\natural} )
\end{equation}
We bound the expectations of these two terms in Sections \ref{sec:noise-avg} and \ref{sec:signal-avg} respectively. 

\subsection*{A Supporting Lemma to Bound Signal Error by Average Squared Intrinsic Distance} 
\begin{lemma} \label{lem:signal-avg} Let $\mb x_{1\natural} \dots \mb x_{N\natural} \in \mc M$, and let $\mb q_\natural$ be an arbitrary point on $\mc M$. Then 
\begin{equation}
d\left( \frac{1}{N} \sum_{i=1}^N \mb x_{i\natural}, \; \mc M \right) \le \frac{\kappa}{N} \sum_{i=1}^N d_{\mc M}^2\left(\mb x_{i\natural},\mb q_\natural \right). 
\end{equation}  
\end{lemma}

\begin{proof} 
For $i=1,2...N,$ consider the unique geodesic curve that goes from $\mb q_{\natural}$ to $\mb x_{i\natural}$: $$c_i:[0,1] \rightarrow \mc M, c_i(0)=\mb q_{\natural}, c_i(1)=\mb x_{i\natural}.$$
By definition of geodesics,  $\langle c_i''(t), c_i'(t) \rangle =0$, which implies that $||c_i'(t)||$ is a constant, and $d_{\mc M}(\mb q_\natural,\mb x_{i\natural})=\underset{0}{\overset{1}{\int}}||c_i'(t)||dt = ||c_i'(0)||$. 
Next, we define $c_*:[0,1]\rightarrow \mc M$ to be the unique geodesic such that $c_*(0)=\mb q_\natural,\, c_*^{'}(0)=\frac{1}{N}\underset{i=1}{\overset{N}{\sum}}c_i^{'}(0)$.  Let $\mb x_{\natural}^* = c_*(1)$, which is clearly on $\mc M$. 
Then, 
\begin{equation}
\begin{aligned}
    d\left( \frac{1}{N} \sum_{i=1}^N \mb x_{i\natural}, \; \mc M \right) &\le \left\|\mb x_{\natural}^* - \frac{1}{N} \sum_{i=1}^N \mb x_{i\natural} \right\|_2 \\
    &=\left\|c_*(0) + \underset{0}{\overset{1}{\int}} c_*^{'}(t) dt - \frac{1}{N} \sum_{i=1}^N \left(c_i(0) + \underset{0}{\overset{1}{\int}} c_i^{'}(t) dt\right) \right\|_2  \\
    &=\left\| \underset{0}{\overset{1}{\int}} c_*^{'}(t) dt - \frac{1}{N} \sum_{i=1}^N  \underset{0}{\overset{1}{\int}} c_i^{'}(t) dt \right\|_2  \\
    &=\left\| \underset{0}{\overset{1}{\int}} \left(c_*^{'}(0) + \underset{0}{\overset{t}{\int}} c_*^{''}(a) da \right)dt - \frac{1}{N} \sum_{i=1}^N  \underset{0}{\overset{1}{\int}} \left(c_i^{'}(0) + \underset{0}{\overset{t}{\int}} c_i^{''}(a) da \right)dt \right\|_2  \\
    &= \left\| \underset{0}{\overset{1}{\int}}\underset{0}{\overset{t}{\int}} c_*^{''}(a) da\:dt - \frac{1}{N} \sum_{i=1}^N  \underset{0}{\overset{1}{\int}} \underset{0}{\overset{t}{\int}} c_i^{''}(a) da\:dt \right\|_2  \\
    &\le \underset{0}{\overset{1}{\int}}\underset{0}{\overset{t}{\int}} \left\|c_*^{''}(a) -\frac{1}{N} \sum_{i=1}^N c_i^{''}(a) \right\|_2 da\:dt\\ 
    &\le \underset{0}{\overset{1}{\int}}\underset{0}{\overset{t}{\int}} \left\|c^{''*}(a)\right\|_2 + \frac{1}{N} \sum_{i=1}^N \left\|c_i^{''}(a) \right\|_2 da\:dt\\
\end{aligned}
\end{equation}
where in the third to last line, the $c_*^{'}(0)$ and $\frac{1}{N}\sum_{i=1}^N c_i^{'}(0)$ gets canceled out by our definition of $c_*^{'}(0)$.

Then, observing that we are left with only second derivatives in the expression above, we now need to consider the bounds by the  curvature $\kappa$. Since we know for any geodesic curve c, $\|c^{''}(t)\| \le \kappa\left\|c^{'}(t)\right\|_2^2$, we have 
\begin{equation}
    \begin{aligned}
     d\left( \frac{1}{N} \sum_{i=1}^N \mb x_{i\natural}, \; \mc M \right)  &\le \underset{0}{\overset{1}{\int}}\underset{0}{\overset{t}{\int}} \left\|c_*^{''}(a)\right\|_2 + \frac{1}{N} \sum_{i=1}^N \left\|c_i^{''}(a) \right\|_2 da\:dt \\
     &\le \underset{0}{\overset{1}{\int}}\underset{0}{\overset{t}{\int}} \kappa\left\|c_*^{'}(0)\right\|_2^2 + \frac{\kappa}{N} \sum_{i=1}^N \left\|c_i^{'}(0) \right\|_2^2 da\:dt \\
    &\le \underset{0}{\overset{1}{\int}}\underset{0}{\overset{t}{\int}} \kappa\left\|\frac{1}{N}\sum_{i=1}^N c_i^{'}(0)\right\|_2^2 + \frac{\kappa}{N} \sum_{i=1}^N \left\|c_i^{'}(0) \right\|_2^2 da\:dt \\
    &\le \underset{0}{\overset{1}{\int}}\underset{0}{\overset{t}{\int}}\frac{2\kappa}{N} \sum_{i=1}^N \left\|c_i^{'}(0) \right\|_2^2 da\:dt \\
    &\le \frac{\kappa}{N} \sum_{i=1}^N d_{\mc M}^2\left(\mb x_{i\natural},\mb q_\natural \right)
    \end{aligned}
\end{equation}
This completes the proof.
\end{proof}

\section{The ``Noise Size'' Subproblem} \label{sec:noise-avg}

The goal of this section is to upper bound the size of the conditional expectation of the noise $\mb z_i$ for points $\mb z_i$ assigned to landmark $\mb q$: 
\begin{equation}
\left\|  \bb E\Bigl[ \, \mb z_i \mid \| \mb x_i + \mb z_i- \mb q \| \le R \, \Bigr] \right\|. 
\end{equation} 
We prove:
\begin{theorem}\label{thm:noise-size} Suppose that $\kappa \,\mr{diam}(\mc M) \ge 1$ and $\sigma > D^{-1}$. There exist numerical constants $C_1\dots C_4$, $c_1, c_2 > 0$, such that whenever $D > C_1$, $d(\mb q,\mc M) < \tfrac{c_1}{\kappa}$, and $s_\star^2 = R^2 - \sigma^2 (D-3)$ satisfies $s_{\star,\parallel}^2 = s_\star^2 - d^2(\mb q, \mc M) \ge 0,$ and 
\begin{equation}
    \max\left\{C_2 \left(\log\left(D \right) + \kappa d \, \mr{diam}(\mc M) + d \log( \frac{1}{\kappa s_{\star,\parallel}} )\right)  \, \sigma^2 D^{2/3} , d^2(\mb q, \mc M )\right\}
    \le s_\star^2 
    \le \min \Bigl\{ 3 \sigma^2 D, d^2(\mb q, \mc M) + \frac{c_2 }{\kappa^2} \Bigr\} \label{eqn:noise-size-cond1}
\end{equation}
we have 
\begin{eqnarray}
   \left\|  \bb E\Bigl[ \, \mb z_i \mid \| \mb x_i + \mb z_i- \mb q \| \le R \, \Bigr] \right\| \;\le\; \frac{C_3 s_\star}{(D-3)^{1/2}} \sqrt{\log\left(D \right) + \kappa d \, \mr{diam}(\mc M) + d \log( \frac{1}{\kappa s_{\star,\parallel}} )},
\end{eqnarray} 
\end{theorem} 

\vspace{.1in}

\begin{proof} 
Let $\mc E_i$ denote the event that $\| \mb x_i + \mb z_i- \mb x_1 \| \le R$. Using the tower property of the conditional expectation, we have 
\begin{eqnarray}
    \bb E \Bigl[ \mb z_i \mid \mc E_i, \mb x_1 \Bigr] &=& \bb E_{\mb x_{i\natural} \mid \mc E_i, \mb x_1} \Biggl[ \, \bb E  \Bigl[ \mb z_i \mid \mc E_i, \mb x_1, \mb x_{i\natural} \Bigr] \, \Biggr] \\
    &=& \int_{\bar{\mb x}_{i\natural}} \bb E  \Bigl[ \mb z_i \mid \mc E_i, \mb x_1, \mb x_{i\natural} \Bigr] f_{\mb x_{i\natural} \mid \mc E_i, \mb x_1}(  \bar{\mb x}_{i\natural} ) \, d \mu_0(\bar{\mb x}_{i\natural}) 
\end{eqnarray}
The conditional expectation is 
\begin{equation}
    \bb E \Bigl[ \, \mb z_i \mid \mc E_i, \mb x_1, \mb x_{i\natural} \, \Bigr] = \frac{ \int_{\bar{\mb z}_i} \bar{\mb z}_i f_{\mb z}(\bar{\mb z_i}) \indicator{ \| \mb z_i - (\mb x_1 - \mb x_{i\natural} ) \|_2 \le R} \,  d \bar{\mb z}_i }{ \int_{\bar{\mb z}_i} f_{\mb z}(\bar{\mb z_i}) \indicator{ \| \mb z_i - (\mb x_1 - \mb x_{i\natural} ) \|_2 \le R} \,  d \bar{\mb z}_i }.
\end{equation}
Meanwhile the density is 
\begin{eqnarray}
    f_{\mb x_{i\natural} \mid \mc E_i, \mb x_1}(  \bar{\mb x}_{i\natural} ) = \frac{ \bb P[ \mc E_i \mid \mb x_{i\natural} = \bar{\mb x}_{i\natural}, \mb x_1 ] \rho_\natural(\bar{\mb x}_{i\natural} ) }{ \int_{\bar{\bar{\mb x}}_{i\natural}} \bb P[ \mc E_i \mid \mb x_{i\natural} = \bar{\bar{\mb x}}_{i\natural}, \mb x_1 ] \rho_\natural(\bar{\bar{\mb x}}_{i\natural} ) d\mu_0(\bar{\bar{\mb x}}_{i\natural}) },
\end{eqnarray}
whence 
\begin{eqnarray}
    \bb E \Bigl[ \, \mb z_i \mid \mc E_i, \mb x_1 \, \Bigr] = \int_{\bar{\mb x}_{i\natural}} \frac{ \int_{\bar{\mb z}_i} \bar{\mb z}_i f_{\mb z}(\bar{\mb z_i}) \indicator{ \| \mb z_i - (\mb x_1 - \mb x_{i\natural} ) \|_2 \le R} \,  d \bar{\mb z}_i }{ \int_{\bar{\mb z}_i} f_{\mb z}(\bar{\mb z_i}) \indicator{ \| \mb z_i - (\mb x_1 - \mb x_{i\natural} ) \|_2 \le R} \,  d \bar{\mb z}_i } \, \frac{ \bb P[ \mc E_i \mid \mb x_{i\natural} = \bar{\mb x}_{i\natural}, \mb x_1 ] \rho_\natural(\bar{\mb x}_{i\natural} ) }{ \int_{\bar{\bar{\mb x}}_{i\natural}} \bb P[ \mc E_i \mid \mb x_{i\natural} = \bar{\bar{\mb x}}_{i\natural}, \mb x_1 ] \rho_\natural(\bar{\bar{\mb x}}_{i\natural} ) d\mu_0(\bar{\bar{\mb x}}_{i\natural}) } \, d \mu_0(\bar{\mb x}_{i\natural}). \nonumber 
\end{eqnarray}
We can simplify this expression by noting that 
\begin{equation}
    \bb P[ \mc E_i \mid \mb x_{i\natural} = \bar{\bar{\mb x}}_{i\natural}, \mb x_1 ] = \int_{\bar{\mb z}_i} f_{\mb z}(\bar{\mb z_i}) \indicator{ \| \mb z_i - (\mb x_1 - \mb x_{i\natural} ) \|_2 \le R} \,  d \bar{\mb z}_i,
\end{equation}
yielding 
\begin{equation}
        \bb E \Bigl[ \, \mb z_i \mid \mc E_i, \mb x_1 \, \Bigr] =  \frac{ \int_{\bar{\mb x}_{i\natural}} \int_{\bar{\mb z}_i} \bar{\mb z}_i f_{\mb z}(\bar{\mb z_i}) \indicator{ \| \mb z_i - (\mb x_1 - \mb x_{i\natural} ) \|_2 \le R} \,  d \bar{\mb z}_i \rho_\natural(\bar{\bar{\mb x}}_{i\natural} ) \, d \mu_0(\bar{\mb x}_{i\natural}) }{ \int_{\bar{\bar{\mb x}}_{i\natural}} \bb P[ \mc E_i \mid \mb x_{i\natural} = \bar{\bar{\mb x}}_{i\natural}, \mb x_1 ] \rho_\natural(\bar{\bar{\mb x}}_{i\natural} ) d\mu_0(\bar{\bar{\mb x}}_{i\natural}) }.
\end{equation}
In the sequel, we assume a uniform density $\rho_\natural( \bar{\mb x}_{i\natural}) = \rho_0$, so the above simplifies to 
\begin{equation}
        \bb E \Bigl[ \, \mb z_i \mid \mc E_i, \mb x_1 \, \Bigr] =  \frac{ \int_{\bar{\mb x}_{i\natural}} \int_{\bar{\mb z}_i} \bar{\mb z}_i f_{\mb z}(\bar{\mb z_i}) \indicator{ \| \mb z_i - (\mb x_1 - \mb x_{i\natural} ) \|_2 \le R} \,  d \bar{\mb z}_i  \, d \mu_0(\bar{\mb x}_{i\natural}) }{ \int_{\bar{\bar{\mb x}}_{i\natural}} \bb P[ \mc E_i \mid \mb x_{i\natural} = \bar{\bar{\mb x}}_{i\natural}, \mb x_1 ] d\mu_0(\bar{\bar{\mb x}}_{i\natural}) }.
\end{equation}
We let $\Upsilon$ denote the normalizing constant in the denominator: 
\begin{equation}
    \Upsilon = \int_{\bar{\bar{\mb x}}_{i\natural}} \bb P[ \mc E_i \mid \mb x_{i\natural} = \bar{\bar{\mb x}}_{i\natural}, \mb x_1 ] \, d\mu_0(\bar{\bar{\mb x}}_{i\natural})
\end{equation}
For the numerator, using the rotational invariance of the Gaussian distribution, we have 
\begin{eqnarray} 
\int_{\bar{\mb z}_i} \bar{\mb z}_i f_{\mb z}(\bar{\mb z_i}) \indicator{ \| \mb z_i - (\mb x_1 - \mb x_{i\natural} ) \|_2 \le R} \,  d \bar{\mb z}_i &=& \frac{ \mb x_1 - \mb x_{i\natural} }{\| \mb x_1 - \mb x_{i\natural} \|_2 } \left[ \int_{\mb v = (v_1, \dots v_D)} \mb v_1 f_{\mb z}( \mb v ) \indicator{ \| \mb v - t \mb e_1  \| \le R } \, d \mb v \, \right] \Biggr|_{t = \| \mb x_1 - \mb x_{i\natural} \|} \quad 
\end{eqnarray} 
Below, we will let 
\begin{equation}
    \mb u( \mb x_1 - \mb x_{i\natural} ) = \frac{ \mb x_1 - \mb x_{i\natural} }{\| \mb x_1 - \mb x_{i\natural} \|_2 }.
\end{equation}
In that language, 
\begin{eqnarray}
     \bb E \Bigl[ \, \mb z_i \mid \mc E_i, \mb x_1 \, \Bigr] &=& \Upsilon^{-1} \int_{\bar{\mb x}_{i\natural}} \mb u(\mb x_1 - \bar{\mb x}_{i\natural}) \int_{\mb v = (v_1, \dots v_D)}  v_1 f_{\mb z}( \mb v ) \indicator{ \| \mb v - t \mb e_1  \| \le R } \, d \mb v \,  \Bigr|_{t = \| \mb x_1 - \mb x_{i\natural} \|} \, d \mu_0( \bar{\mb x}_{i\natural} ). \quad 
\end{eqnarray}
We can develop the inner integrand as follows: 
\begin{align}
    \lefteqn{ \int_{\mb v = (v_1, \dots v_D)}  v_1 f_{\mb z}( \mb g ) \indicator{ \| \mb v - t \mb e_1  \| \le R } \, d \mb v } \nonumber \\
    &= \int_{v_1} v_1 f_{z_1}(v_1) \Biggl[ \int_{v_2, \dots, v_D} f_{z_2,\dots,z_D}(v_2,\dots,v_D) \indicator{v_2^2+\dots+v_D^2 \le R^2 - (t-v_1)^2 } d v_2,\dots,v_D \Biggr] dv_1 \\
    &= \int_{v_1} v_1 f_{z_1}(v_1) h(t-v_1) d v_1 \label{eqn:inner-final} 
\end{align}
where 
\begin{eqnarray}
    h(s) &=& \bb P_{g_2,\dots, g_D \sim_{\mr{iid}} \mc N(0,1)} \left[ g_2^2 + \dots + g_D^2 \le \frac{R^2 - s^2  }{\sigma^2} \right]  \\ 
    &=& \bb P_{X \sim \chi^2_{D-1}} \left[ X \le \frac{R^2 - s^2}{\sigma^2}\right] \nonumber \\
    &=& \frac{1}{\Gamma(\tfrac{D-1}{2})} \gamma \left( \frac{D-1}{2}, \frac{R^2- s^2}{2 \sigma^2} \right).  
\end{eqnarray}
We can recognize the integral \eqref{eqn:inner-final} as a convolution between $\psi(v_1) = v_1 f_{z_1}(v_1)$ and $h$. Moreover, since 
\begin{eqnarray}
    \phi(v_1) \;\doteq\; f_{z_1}(v_1) &=& \frac{1}{\sqrt{2 \pi} \sigma} \exp\left( - \frac{v_1^2}{2 \sigma^2} \right) \\
    \dot{\phi}(v_1) \;\doteq\; \dot{f}_{z_1}(v_1) &=& -\frac{v_1}{\sigma^2} f_{z_1}(v_1) \\ 
    \psi(v_1) \;\doteq\; v_1 f_{z_1}(v_1) &=& -\sigma^2 \dot{\phi}(v_1),
\end{eqnarray}
and \eqref{eqn:inner-final} satisfies 
\begin{eqnarray}
    \int_{v_1} v_1 f_{z_1}(v_1) h(t-v_1) \, d v_1 &=& [ \psi \ast h ](t) \\
    &=& \sigma^2 [ -\dot{\phi} \ast h ](t) \\
    &=& \sigma^2 [ -\phi \ast \dot{h} ](t). 
\end{eqnarray}
Summing up, we have 
\begin{eqnarray}
    \bb E\Bigl[ \, \mb z_i \mid \mc E_i, \mb x_1 \, \Bigr] &=& \frac{ \int_{\bar{\mb x}_{i\natural}}  \mb u(\mb x_1 - \mb x_{i\natural} ) \, \sigma^2 \bigl[ -\phi \ast \dot{h} \bigr] \bigl( \| \mb x_1 - \bar{\mb x}_{i\natural} \| \bigr) \; d \mu_0( \bar{\mb x}_{i\natural} ) }{ \int_{\bar{\bar{\mb x}}_{i\natural}} \bb P[ \mc E_i \mid \mb x_{i\natural} = \bar{\bar{\mb x}}_{i\natural}, \mb x_1 ] \, d\mu_0(\bar{\bar{\mb x}}_{i\natural})  }
\end{eqnarray}
We can develop the denominator in similar terms: again using the rotational invariance of the gaussian distribution, we obtain 
\begin{eqnarray}
\lefteqn{ 
     \bb P[ \mc E_i \mid \mb x_{i\natural} = \bar{\bar{\mb x}}_{i\natural}, \mb x_1 ] \quad =\quad  \left[ \int_{v_1 \dots v_D} f_{\mb z}( v_1, \dots, v_D ) \indicator{ (v_1 - t)^2 + v_2^2 + \dots + v_D^2 \le R^2 } \, d v_1 \dots v_D  \right]_{t = \| \mb x_1 - \bar{\bar{\mb x}}_{i\natural} \|  }  
     } \nonumber \\
     &=& \Biggl[ \int_{v_1} f_{z_1}(v_1) \Biggl[ \int_{v_2, \dots, v_D} f_{z_2,\dots,z_D}(v_2,\dots,v_D) \indicator{v_2^2+\dots+v_D^2 \le R^2 - (t-v_1)^2 } d v_2,\dots,v_D \Biggr] dv_1 \Biggr]_{\| \mb x_1 - \bar{\bar{\mb x}}_{i\natural} \|} \\
     &=& \Biggl[ \int_{v_1} f_{z_1}(v_1) h(t-v_1) \,  d v_1  \Biggr]_{\| \mb x_1 - \bar{\bar{\mb x}}_{i\natural} \|} \\
     &=& \phi \ast h( \| \mb x_1 - \bar{\bar{\mb x}}_{i\natural} \| ),
\end{eqnarray}
and overall we obtain 
\begin{eqnarray}
    \bb E\Bigl[ \, \mb z_i \mid \mc E_i, \mb x_1 \, \Bigr] &=& \frac{ \int_{\bar{\mb x}_{i\natural}}  \mb u(\mb x_1 - \bar{\mb x}_{i\natural} ) \, \sigma^2 \bigl[ - \phi \ast \dot{h} \bigr] \bigl( \| \mb x_1 - \bar{\mb x}_{i\natural} \| \bigr) \; d \mu_0( \bar{\mb x}_{i\natural} ) }{ \int_{\bar{\bar{\mb x}}_{i\natural}} [\phi \ast h] ( \| \mb x_1 - \bar{\bar{\mb x}}_{i\natural} \| ) \, d\mu_0(\bar{\bar{\mb x}}_{i\natural})  }
\end{eqnarray}
\noindent We bound the norm of this quantity as follows. First, by the triangle inequality,
\begin{eqnarray}
   \left\|  \bb E\Bigl[ \, \mb z_i \mid \mc E_i, \mb x_1 \, \Bigr] \right\| &=& \left\| \frac{ \int_{\bar{\mb x}_{i\natural}}  \mb u(\mb x_1 - \bar{\mb x}_{i\natural} ) \, \sigma^2 \bigl[ - \phi \ast \dot{h} \bigr] \bigl( \| \mb x_1 - \bar{\mb x}_{i\natural} \| \bigr) \; d \mu_0( \bar{\mb x}_{i\natural} ) }{ \int_{\bar{\bar{\mb x}}_{i\natural}} [\phi \ast h] ( \| \mb x_1 - \bar{\bar{\mb x}}_{i\natural} \| ) \, d\mu_0(\bar{\bar{\mb x}}_{i\natural})  }  \right\| \\ 
    &\le & \left\| \frac{ \int_{\bar{\mb x}_{i\natural} : \| \bar{\mb x}_{i\natural} - \mb x_1 \| \le s_\star + \Delta }  \mb u(\mb x_1 - \bar{\mb x}_{i\natural} ) \, \sigma^2 \bigl[ - \phi \ast \dot{h} \bigr] \bigl( \| \mb x_1 - \bar{\mb x}_{i\natural} \| \bigr) \; d \mu_0( \bar{\mb x}_{i\natural} ) }{ \int_{\bar{\bar{\mb x}}_{i\natural}} [\phi \ast h] ( \| \mb x_1 - \bar{\bar{\mb x}}_{i\natural} \| ) \, d\mu_0(\bar{\bar{\mb x}}_{i\natural})  } \right\| \nonumber \\ 
    && \qquad + \qquad \left\| \frac{ \int_{\bar{\mb x}_{i\natural} : \| \bar{\mb x}_{i\natural} - \mb x_1 \| > s_\star + \Delta}  \mb u(\mb x_1 - \bar{\mb x}_{i\natural} ) \, \sigma^2 \bigl[ - \phi \ast \dot{h} \bigr] \bigl( \| \mb x_1 - \bar{\mb x}_{i\natural} \| \bigr) \; d \mu_0( \bar{\mb x}_{i\natural} ) }{ \int_{\bar{\bar{\mb x}}_{i\natural}} [\phi \ast h] ( \| \mb x_1 - \bar{\bar{\mb x}}_{i\natural} \| ) \, d\mu_0(\bar{\bar{\mb x}}_{i\natural})  } \right\|. \hspace{.25in} 
\end{eqnarray} 
Here, we have upper bounded the expectation as the sum of two terms -- an integration over $\bar{\mb x}_{i\natural}$ {\em near} to $\mb x_1$ and an integration over $\bar{\mb x}_{i\natural}$ {\em far} from $\mb x_1$. We use two different techniques to bound these terms: for the near term, we use the relative bound from Lemma \ref{lem:relative}. As long as 
\begin{equation}
\Delta < \tfrac{1}{12} \sigma (D-3)^{1/6}, 
\end{equation} 
for $s \le s_\star + \Delta$, 
Lemma \ref{lem:relative} gives 
\begin{equation}
\frac{ \sigma^2 [\phi \ast - \dot{h}](s) }{ [\phi \ast h](s)  }  \le C \max\left\{ \frac{\sigma^2}{\nu}, \frac{\sigma^2 \Delta}{\nu^2} \right\},
\end{equation} 
where 
\begin{equation}
    \nu^2 = \frac{\sigma^4 (D-3)}{2 s_\star^2}. 
\end{equation}
Using the triangle inequality, we have 
\begin{eqnarray}
\lefteqn{  \left\| \frac{ \int_{\bar{\mb x}_{i\natural} : \| \bar{\mb x}_{i\natural} - \mb x_1 \| \le s_\star + \Delta }  \mb u(\mb x_1 - \bar{\mb x}_{i\natural} ) \, \sigma^2 \bigl[ - \phi \ast \dot{h} \bigr] \bigl( \| \mb x_1 - \bar{\mb x}_{i\natural} \| \bigr) \; d \mu_0( \bar{\mb x}_{i\natural} ) }{ \int_{\bar{\bar{\mb x}}_{i\natural}} [\phi \ast h] ( \| \mb x_1 - \bar{\bar{\mb x}}_{i\natural} \| ) \, d\mu_0(\bar{\bar{\mb x}}_{i\natural})  } \right\| } \nonumber  \\ 
&\le&  \frac{ \int_{\bar{\mb x}_{i\natural} : \| \bar{\mb x}_{i\natural} - \mb x_1 \| \le s_\star + \Delta }    \sigma^2 \bigl[ - \phi \ast \dot{h} \bigr] \bigl( \| \mb x_1 - \bar{\mb x}_{i\natural} \| \bigr) \; d \mu_0( \bar{\mb x}_{i\natural} ) }{ \int_{\bar{\bar{\mb x}}_{i\natural}} [\phi \ast h] ( \| \mb x_1 - \bar{\bar{\mb x}}_{i\natural} \| ) \, d\mu_0(\bar{\bar{\mb x}}_{i\natural})  } \\ 
&\le&  C \max\left\{ \frac{\sigma^2}{\nu}, \frac{\sigma^2 \Delta}{\nu^2} \right\} \frac{ \int_{\bar{\mb x}_{i\natural} : \| \bar{\mb x}_{i\natural} - \mb x_1 \| \le s_\star + \Delta }  [\phi \ast h]  \bigl( \| \mb x_1 - \bar{\mb x}_{i\natural} \| \bigr) \; d \mu_0( \bar{\mb x}_{i\natural} ) }{ \int_{\bar{\bar{\mb x}}_{i\natural}} [\phi \ast h] ( \| \mb x_1 - \bar{\bar{\mb x}}_{i\natural} \| ) \, d\mu_0(\bar{\bar{\mb x}}_{i\natural})  }  \\ 
&\le& C \max\left\{ \frac{\sigma^2}{\nu}, \frac{\sigma^2 \Delta}{\nu^2} \right\}. 
\end{eqnarray} 
For the {\em far} term, we use the absolute bound \eqref{eqn:phi-conv-h-dot-upper} from Lemma \ref{lem:phi-conv-h-dot}, and the upper bound from Lemma \ref{lem:c-of-D}, which gives 
\begin{eqnarray}
\lefteqn{ \left\| \frac{ \int_{\bar{\mb x}_{i\natural} : \| \bar{\mb x}_{i\natural} - \mb x_1 \| > s_\star + \Delta}  \mb u(\mb x_1 - \bar{\mb x}_{i\natural} ) \, \sigma^2 \bigl[ - \phi \ast \dot{h} \bigr] \bigl( \| \mb x_1 - \bar{\mb x}_{i\natural} \| \bigr) \; d \mu_0( \bar{\mb x}_{i\natural} ) }{ \int_{\bar{\bar{\mb x}}_{i\natural}} [\phi \ast h] ( \| \mb x_1 - \bar{\bar{\mb x}}_{i\natural} \| ) \, d\mu_0(\bar{\bar{\mb x}}_{i\natural})  } \right\| } \nonumber \\
&\le& \Upsilon^{-1}  \int_{\bar{\mb x}_{i\natural} : \| \bar{\mb x}_{i\natural} - \mb x_1 \| > s_\star + \Delta}  \sigma^2 \bigl[ - \phi \ast \dot{h} \bigr] \bigl( \| \mb x_1 - \bar{\mb x}_{i\natural} \| \bigr) \; d \mu_0( \bar{\mb x}_{i\natural} ) \\ 
&\le& \sigma^2 \, \Upsilon^{-1} \, \mr{vol}(\mc M) \, \sup \Biggl\{  [ \phi \ast - \dot{h} ](t) \, | \, t > s_\star + \Delta \Biggr\} \\
&\le& C \frac{\mr{vol}(\mc M)}{\Upsilon} s_\star \exp\left( - \frac{\Delta^2}{2 \bar{\nu}^2} \right),
\end{eqnarray} 
with 
\begin{equation}
    \bar{\nu}^2 = \nu^2 + \sigma^2. 
\end{equation}
Using Lemma \ref{lem:vol-ratio} to control the volume ratio $\mr{vol}(\mc M) / \Upsilon$, we get 
\begin{eqnarray}
   \lefteqn{  \left\| \frac{ \int_{\bar{\mb x}_{i\natural} : \| \bar{\mb x}_{i\natural} - \mb x_1 \| > s_\star + \Delta}  \mb u(\mb x_1 - \bar{\mb x}_{i\natural} ) \, \sigma^2 \bigl[ - \phi \ast \dot{h} \bigr] \bigl( \| \mb x_1 - \bar{\mb x}_{i\natural} \| \bigr) \; d \mu_0( \bar{\mb x}_{i\natural} ) }{ \int_{\bar{\bar{\mb x}}_{i\natural}} [\phi \ast h] ( \| \mb x_1 - \bar{\bar{\mb x}}_{i\natural} \| ) \, d\mu_0(\bar{\bar{\mb x}}_{i\natural})  } \right\| } \nonumber \\
   &\le& C \exp \Bigl( \kappa d \, \mr{diam}(\mc M) - d \log(\kappa s_{\star,\parallel} ) \Bigr)   s_\star \exp\left( - \frac{\Delta^2}{2 \bar{\nu}^2 } \right). \qquad \qquad \qquad \qquad \qquad \qquad 
\end{eqnarray}
where 
\begin{equation}
s_{\star,\parallel}^2 = s_\star^2 - d^2(\mb q, \mc M).
\end{equation} 
Set 
\begin{equation}
    \Delta = C\bar{\nu} \sqrt{ 2 \left( \log\left(\frac{ s_\star{\nu}}{\sigma^2} \right) + \kappa d \, \mr{diam}(\mc M) - d \log( \kappa s_{\star,\parallel} ) \right) }
\end{equation}
We verify that this $\Delta$ satisfies $\Delta < \tfrac{1}{12} \sigma (D-3)^{1/6}$. 

Since 
$s_\star^2 \le 3 \sigma^2 D, D\ge C_1$, we have 
\begin{equation}
\nu^2 = \frac{\sigma^4 (D-3)}{2 s_\star^2} \ge \frac{\sigma^2}{7},
\end{equation}
and so 
\begin{equation}
    \bar{\nu}^2 = \nu^2 + \sigma^2 \le 8 \nu^2. 
\end{equation}

By definition of $\nu$ we also have 
\begin{equation}
    \log\left(\frac{ s_\star\nu}{\sigma^2} \right) 
    = \frac{1}{2}\log\left(\frac{ s_\star^2\nu^2}{\sigma^4}\right) 
    = \frac{1}{2}\log\left(\frac{ D-3}{2}\right),
\end{equation} which gives 

\begin{eqnarray}
   \Delta &\le&C'\nu\sqrt{\log\left(D \right) + \kappa d \, \mr{diam}(\mc M) + d \log( \frac{1}{\kappa s_{\star,\parallel}} ) }\\
   &=& C'\frac{\sigma^2 (D-3)^{1/2}}{s_\star}\sqrt{\log\left(D \right) + \kappa d \, \mr{diam}(\mc M) + d \log( \frac{1}{\kappa s_{\star,\parallel}} ) },
\end{eqnarray}
which, with our lower bound on $s_\star$ and appropriate choice of $C_2$, is bounded by $\tfrac{1}{12} \sigma(D-3)^{1/6}$. With this choice, the {\em far} term is bounded by $\frac{\sigma^2}{\nu}$, while the {\em near} term is bounded by 
\begin{equation}
\frac{C\sigma^2 \Delta}{\nu^2} \le \frac{C \sigma^2 \sqrt{\log\left(D \right) + \kappa d \, \mr{diam}(\mc M) + d \log( \frac{1}{\kappa s_{\star,\parallel}} ) }}{ \nu } \le \frac{ 2 C \sqrt{\log\left(D \right) + \kappa d \, \mr{diam}(\mc M) + d \log( \frac{1}{\kappa s_{\star,\parallel}} ) } }{ (D-3)^{1/2} } s_\star 
\end{equation}
giving the claim. 
\end{proof}

\subsection{Technical Lemmas for Controlling the Noise Size}

Consider the function 
\begin{equation}
h(s) = \bb P_{g_2,\dots, g_D \sim_{\mr{iid}} \mc N(0,1)} \left[ g_2^2 + \dots + g_D^2 \le \frac{R^2 - s^2  }{\sigma^2} \right]. 
\end{equation}
Because $h(s)$ is a decreasing function of $s$, 
\begin{equation}
-\dot{h}(s) > 0 \quad \forall \; s. 
\end{equation} 
Note that 
\begin{equation}
h(s) =  \frac{1}{\Gamma(\tfrac{D-1}{2})} \gamma \left( \frac{D-1}{2}, \frac{R^2- s^2}{2 \sigma^2} \right),
\end{equation}
and 
\begin{equation}
-\dot{h}(s) =  \frac{1}{\Gamma(\tfrac{D-1}{2})} \frac{s}{\sigma^2} \dot{\gamma} \left( \frac{D-1}{2}, \frac{R^2- s^2}{2 \sigma^2} \right),
\end{equation}
where $\gamma$ is the lower incomplete gamma function 
\begin{equation}
    \gamma(p,x) = \int_{t = 0}^x t^{p-1} e^{-t} dt.
\end{equation}
The following lemma upper and lower bounds the function $\dot{\gamma}(p,x)$ near the point $x_\star = p-1$: 

\begin{lemma}[Subgaussian Bounds for $\dot{\gamma}$ Near the Phase Transition.] \label{lem:icgf} Let $p \ge 28$ and $x_\star = p-1$. Then 
\begin{equation}
 \tfrac{1}{4} \dot{\gamma}(p,x_\star) \exp\left( - \frac{(x-x_\star)^2}{2 x_\star } \right) \le \dot{\gamma}(p,x) \le 4 \dot{\gamma}(p,x_\star) \exp\left( - \frac{(x-x_\star)^2}{2 x_\star } \right) \qquad x \in \left[x_\star-x_\star^{2/3},x_\star+x_\star^{2/3}\right]. 
\end{equation}
\end{lemma} 
\begin{proof}
    Please see Section \ref{sec:icgf}. 
\end{proof}

\noindent The following lemma uses these subgaussian bounds for the incomplete gamma function to obtain subgaussian bounds for the function $[-\dot{h}]$:

\begin{lemma}[Subgaussian Bounds for $-\dot{h}$.] Suppose that $D > 192$. Set 
\begin{equation}
s_\star = \sqrt{ R^2 - \sigma^2 (D-3) },
\end{equation}
and suppose that 
\begin{equation}
\sigma^2 D^{1/2} \le s_\star^2 \le 3 \sigma^2 D.
\end{equation}
Then the function $-\dot{h}(s)$ satisfies 
\begin{equation} \label{eqn:hdot-subgaussian} 
\frac{1}{8e}  \, C(D) \frac{s_\star}{\sigma^2} \exp\left( - \frac{\left( s - s_\star \right)^2 }{ 2 \nu^2 }\right) \le  -\dot{h}(s) \le 6e \, C(D) \frac{s_\star}{\sigma^2} \exp\left( - \frac{\left( s - s_\star \right)^2 }{ 2 \nu^2 }\right) 
\end{equation}
when  
\begin{equation}
|s - s_\star| \le \tfrac{1}{3} \sigma (D-3)^{1/6}.
\end{equation}
Here and in following lemmas, 
\begin{equation}
    C(D) = \frac{ \dot{\gamma}\left( \frac{D-1}{2}, \frac{D-3}{2} \right) }{ \Gamma\left(\frac{D-1}{2} \right) },
\end{equation}
and 
\begin{equation}
    \nu^2 = \frac{\sigma^4 (D-3)}{2 s_\star^2}. 
\end{equation}
\end{lemma}

\begin{proof} We first note that that when $\sigma^2 D^{1/2} \le s_\star^2 \le 3 \sigma^2 D$ and 
$|s - s_\star| \le \tfrac{1}{3} \sigma (D-3)^{1/6}$
we have 
\begin{equation} \label{eqn:s-rel-diff}
|s - s_\star| < \tfrac{1}{2} s_\star
\end{equation}
and 
\begin{equation} \label{eqn:s-mixed-cube}
    \frac{5 s_\star | s - s_\star |^3}{4 \sigma^4 (D-3)} < 1.
\end{equation}

We would like to apply Lemma \ref{lem:icgf} (and noting that with $D>192$, the condition $p > 28$ of the lemma is satisfied for $p = (D-1)/2$) to bound $\dot{\gamma}$, we have 
\begin{eqnarray}
\lefteqn{ - \dot{h}(s) \quad =\quad  \frac{1}{\Gamma\left(\frac{D-1}{2} \right)} \frac{s}{\sigma^2} \dot{\gamma}\left( \frac{D-1}{2}, \frac{R^2 - s^2 }{2 \sigma^2} \right) } \nonumber \\ 
 &\le& \frac{4 \dot{\gamma}\left( \frac{D-1}{2}, \frac{D-3}{2} \right) }{ \Gamma\left(\frac{D-1}{2} \right) } \frac{s}{\sigma^2} \exp\left( - \frac{\left( \frac{R^2 - s^2}{2 \sigma^2} - \frac{D-3}{2} \right)^2 }{ D-3 }\right) \qquad   s_\star^2 - \sigma^2 D^{2/3} \le s^2 \le s_\star^2 + \sigma^2 D^{2/3} \nonumber \\
 &=& 4 C(D) \frac{s}{\sigma^2} \exp\left( - \frac{\left( s_\star^2 - s^2 \right)^2 }{ 4 \sigma^4 (D-3) }\right) \qquad   s_\star^2 - \sigma^2 D^{2/3} \le s^2 \le s_\star^2 + \sigma^2 D^{2/3}. 
\end{eqnarray} 
Here, a direct application of Lemma \ref{lem:icgf} ensures that these bounds hold when 
\begin{equation} \label{eqn:hdot-bound-natural-cond}
     \frac{D-3}{2} - \left( \frac{D-3}{2} \right)^{2/3}  \le \frac{R^2 - s^2}{2 \sigma^2} \le \frac{D-3}{2} + \left( \frac{D-3}{2} \right)^{2/3}.
\end{equation}
by plugging our definition of $s_\star^2 = R^2 - \sigma^2(D-3)$, this is equivalent as 
\begin{equation}
     \left|\frac{s_\star^2 - s^2}{2 \sigma^2}\right| \le \left( \frac{D-3}{2} \right)^{2/3}.
\end{equation}
so it's sufficient to show 
\begin{equation}
     |s_\star - s|(2s_\star+|s_\star - s|) \le {2 \sigma^2}\left(\frac{D-3}{2} \right)^{2/3}.
\end{equation}
applying our bounds of $|s_\star - s|$ and $s_\star$, it further suffices to show 

\begin{equation}
     \frac{1}{3}\sigma(D-3)^{1/6}(2\sqrt3\sigma\sqrt D + \frac{1}{3}\sigma(D-3)^{1/6}) \le {2 \sigma^2}\left(\frac{D-3}{2} \right)^{2/3}.
\end{equation}
rearranging the terms and canceling out $\sigma$, we need to show 
\begin{equation}
     \frac{2\sqrt3}{3}(D-3)^{1/6}\sqrt D + \frac{1}{9}(D-3)^{1/3} \le 2^{1/3}\left(D-3\right)^{2/3}.
\end{equation}
So as long as $D \ge 2^6 \times 3, \sqrt D \le \sqrt{D-3} \sqrt{\frac{64}{63}}$, and $(D-3)^{1/3} \ge 3$, we have
\begin{equation}
     \frac{2\sqrt3}{3}(D-3)^{1/6}\sqrt D + \frac{1}{9}(D-3)^{1/3} 
     \le (\sqrt{\frac{64}{63}}\frac{2\sqrt3}{3}+\frac{1}{27})(D-3)^{2/3}
     \le 2^{1/3}\left(D-3\right)^{2/3}.
\end{equation}
so the conditions for Lemma \ref{lem:icgf} hold. We next build a subgaussian bound for the exponential. Using convexity of the function $f(t) = (2 + t )^2$, i.e., $f(t) \ge f(0) + 4 t,$ we have  
\begin{eqnarray}
    (s + s_\star)^2 &=& s_\star^2 \left( 1 + \frac{s}{s_\star} \right)^2 \nonumber \\
    &=& s_\star^2 \left( 2 + \frac{s - s_\star}{s_\star} \right)^2 \nonumber \\
    &\ge& 4 s_\star^2 + 4 s_\star (s - s_\star).
\end{eqnarray}

 Hence 
\begin{eqnarray}  
 -\dot{h}(s) &\le& 4 C(D) \frac{s}{\sigma^2} \exp\left( - \frac{\left( s - s_\star \right)^2 \left(s + s_\star \right)^2 }{ 4 \sigma^4 (D-3) }\right)  \nonumber \\
  &\le& 4 C(D) \frac{s}{\sigma^2} \exp\left( - \frac{\left( s - s_\star \right)^2 s_\star^2 }{  \sigma^4 (D-3) }\right) \exp\left( - \frac{s_\star (s - s_\star)^3 }{\sigma^4 (D-3)} \right) \\
  &\le& 4e C(D) \frac{s}{\sigma^2} \exp\left( - \frac{\left( s - s_\star \right)^2 s_\star^2 }{  \sigma^4 (D-3) }\right), \label{eqn:h-dot-inter-1} \\
&\le& 6 e C(D) \frac{s_\star}{\sigma^2} \exp\left( - \frac{\left( s - s_\star \right)^2 s_\star^2 }{  \sigma^4 (D-3) }\right), \label{eqn:h-dot-inter-2}
\end{eqnarray}
where in \eqref{eqn:h-dot-inter-1} we have used \eqref{eqn:s-mixed-cube} and in \eqref{eqn:h-dot-inter-2} we have used \eqref{eqn:s-rel-diff}. 
For the lower bound, we similarly note that 
\begin{eqnarray}
 - \dot{h}(s) &\ge& \frac{1}{4} C(D) \frac{s}{\sigma^2} \exp\left( - \frac{\left( s_\star^2 - s^2 \right)^2 }{ 4 \sigma^4 (D-3) }\right) \qquad   s_\star^2 - \sigma^2 D^{2/3} \le s^2 \le s_\star^2 + \sigma^2 D^{2/3}. 
\end{eqnarray}
For $\tfrac{1}{2} s_\star \le s \le \tfrac{3}{2} s_\star$, we have 
\begin{eqnarray}
    (s + s_\star)^2 &=& s_\star^2 \left( 1 + \frac{s}{s_\star} \right)^2 \nonumber \\
    &=& s_\star^2 \left( 2 + \frac{s - s_\star}{s_\star} \right)^2 \nonumber \\
    &\le& 4 s_\star^2 + 5 s_\star |s - s_\star|
\end{eqnarray}
we obtain 
\begin{eqnarray}
    -\dot{h}(s)  &\ge& \frac{1}{4} C(D) \frac{s}{\sigma^2} \exp\left( - \frac{\left( s - s_\star \right)^2 s_\star^2 }{  \sigma^4 (D-3) }\right) \exp\left( - \frac{5 s_\star |s - s_\star|^3 }{4 \sigma^4 (D-3)} \right) \\
    &\ge& \frac{1}{4e} C(D) \frac{s}{\sigma^2} \exp\left( - \frac{\left( s - s_\star \right)^2 s_\star^2 }{  \sigma^4 (D-3) }\right) \label{eqn:h-dot-inter-3} \\
    &\ge& \frac{1}{8e} C(D) \frac{s_\star}{\sigma^2} \exp\left( - \frac{\left( s - s_\star \right)^2 s_\star^2 }{  \sigma^4 (D-3) }\right)  \label{eqn:h-dot-inter-4}  
\end{eqnarray}
where in \eqref{eqn:h-dot-inter-3} we have used \eqref{eqn:s-mixed-cube} and in \eqref{eqn:h-dot-inter-4} we have used \eqref{eqn:s-rel-diff}. This completes the proof. \end{proof} 

\begin{lemma}[Monotonicity of $-\dot{h}$ outside region of phase transition] \label{lem:monotoncity of -h dot}
Suppose 
\begin{equation}
96 \log(6) \times \sigma^2 D^{2/3} \le s_\star^2 \le 3 \sigma^2 D
\end{equation}
and $D > 48^3 + 3$. 
\newline 
Set 
\begin{equation*}
    \Delta_s = \tfrac{1}{12} \sigma (D-3)^{1/6}
\end{equation*}
Then 
\begin{eqnarray}
     \forall s \le s_\star - \Delta_s, \frac{d}{ds} \Bigl[ -\dot{h}(s) \Bigr] \ge 0 \\
          \forall s \ge s_\star + \Delta_s, \frac{d}{ds} \Bigl[ -\dot{h}(s) \Bigr] \le 0
\end{eqnarray}
\end{lemma} 

\vspace{.1in}

\begin{proof} 
We note that in the following proof $C(D), C'(D)$ represent positive constants independent of $s$. 
By definition, we have 
\begin{equation}
\begin{aligned}
    -\dot{h}(s) & =  \frac{1}{ \Gamma\left(\frac{D-1}{2} \right)} \frac{s}{\sigma^2}\left( \frac{R^2 - s^2}{2\sigma^2} \right)^{(D-3)/2} \exp\left( - \frac{R^2 - s^2}{2\sigma^2} \right) \\
    & = C(D) s \left( R^2 - s^2\right)^{(D-3)/2} \exp\left( - \frac{R^2 - s^2}{2\sigma^2} \right) \\
    & = C'(D) x \left(1-x^2\right)^{(D-3)/2} \exp\left( - \frac{R^2(1-x^2)}{2\sigma^2} \right) \text{ where } x = \frac{s}{R}\\
    & = C'(D) x \left(1-x^2\right)^{\alpha} \exp\left( - \beta(1-x^2) \right) \text{ where } \alpha = \frac{D-3}{2}, \beta = \frac{R^2}{2\sigma^2}
\end{aligned}
\end{equation}
Let $f(x) = x \left(1-x^2\right)^{\alpha} \exp\left( - \beta(1-x^2) \right) $, then we have 
\begin{equation}
    \left.\frac{d}{ds}\left[-\dot{h}(s)\right]\right|_{s=s_0} \ge 0 \quad \Longleftrightarrow \quad \left.\frac{d}{dx}\left[f(x)\right]\right|_{x=\frac{s_0}{R}} \ge 0.
\end{equation}
\\
We can further compute 
\begin{equation}
    f'(x) = \left(1-x^2\right)^{\alpha - 1}\exp\left( - \beta(1-x^2) \right)\left( -2\beta x^4 + (-2\alpha + 2\beta -1)x^2 + 1\right)
\end{equation}
which in turn implies that 
\begin{equation}
    \frac{d}{dx} \Bigl[f(x)\Bigr] \ge 0 \Longleftrightarrow -2\beta x^4 + (-2\alpha + 2\beta -1)x^2 + 1 \ge 0
\end{equation}
Since the right hand side is a downward-facing parabolic function with respect to $x^2$, with a positive root and a negative root, we have 
\begin{equation}
\label{eqn:quadratic root} 
    \frac{d}{dx} \Bigl[f(x)\Bigr] \ge 0 \longleftrightarrow x^2 \le \frac{-2\alpha + 2\beta - 1 + \sqrt{(-2\alpha + 2\beta - 1)^2 + 8\beta}}{4\beta}
\end{equation}
By assumption $s_\star^2 \ge  \sigma^2 D^{2/3}$, so we make the observation that $-2\alpha + 2\beta - 1 = \frac{R^2 - (D-2)\sigma^2}{\sigma^2} = \frac{{s_\star}^2 - \sigma^2}{\sigma^2} \ge 0$.\\

Then, we can plug $\alpha$ and $\beta$ back into the expression of the root on the right hand side of 
\eqref{eqn:quadratic root}, after simplifying, we have 
\begin{equation}
    \frac{d}{dx} \Bigl[f(x)\Bigr] \ge 0 \longleftrightarrow x^2 \le \frac{{s_\star}^2 - \sigma^2 
+ \sqrt{({s_\star}^2 - \sigma^2 )^2 + 4R^2\sigma^2}}{2R^2}
\end{equation}

Lastly, to show $\forall s \le s_\star - \Delta_s, \frac{d}{ds} \Bigl[ -\dot{h}(s) \Bigr] \ge 0$, it's sufficient to show 
\begin{equation}
    \frac{(s_\star - \Delta_s)^2}{R^2}\le \frac{{s_\star}^2 - \sigma^2 
+ \sqrt{({s_\star}^2 - \sigma^2 )^2 + 4R^2\sigma^2}}{2R^2}
\end{equation}
it's further sufficient to show 
\begin{equation}
    (s_\star - \Delta_s)^2\le {s_\star}^2 - \sigma^2 
\end{equation}
which is equivalent of 
\begin{equation}
    {s_\star}^2(1-\frac{\Delta_s}{s_\star})^2 \le {s_\star}^2(1-\frac{\sigma^2}{{s_\star}^2})
\end{equation}
So it suffices to show 
\begin{equation}
1-\frac{\Delta_s}{s_\star} \le 1- \frac{\sigma^2}{{s_\star}^2}    
\end{equation}
which is the same as 
\begin{equation}
\label{trivial condition}
    s_\star \Delta_s \ge \sigma^2
\end{equation}
Since $s_\star \ge \sigma D^{1/3}$ and $\Delta_s = \tfrac{1}{12} \sigma (D-3)^{1/6}$, with the given lower bound on D, \eqref{trivial condition} is trivially satisfied.\\

\vspace{.1in}

Similarly, to show $\forall s \ge s_\star + \Delta_s, \frac{d}{ds} \Bigl[ -\dot{h}(s) \Bigr] \le 0$, it's sufficient to show 
\begin{equation}
    \frac{(s_\star - \Delta_s)^2}{R^2}\ge \frac{{s_\star}^2 - \sigma^2 
+ \sqrt{({s_\star}^2 - \sigma^2 )^2 + 4R^2\sigma^2}}{2R^2}
\end{equation}
it's further sufficient to show 
\begin{equation}
    (s_\star + \Delta_s)^2\ge \sqrt{({s_\star}^2 - \sigma^2 )^2 + 4R^2\sigma^2}
\end{equation}
which is equivalent of 
\begin{equation}
    (s_\star + \Delta_s)^4\ge ({s_\star}^2 - \sigma^2 )^2 + 4R^2\sigma^2
\end{equation}
Since ${s_\star}^4 \ge ({s_\star}^2 - \sigma^2 )^2$, it's enough to show 
\begin{equation}
    4{s_\star}^3\Delta_s + 6{s_\star}^2{\Delta_s}^2 \ge 4R^2\sigma^2
\end{equation}
Since we have defined ${s_\star}^2 = R^2 - \sigma^2 (D-3) $  and recall from \eqref{trivial condition} $ s_\star \Delta_s \ge \sigma^2$
\begin{equation}
\begin{aligned}
    4{s_\star}^3\Delta_s + 6{s_\star}^2{\Delta_s}^2 &\ge 4 (R^2 - \sigma^2 (D-3))\sigma^2 +6{s_\star}^2{\Delta_s}^2 \\ 
    &= 4R^2\sigma^2 - 4 \sigma^4 (D-3)
 +6{s_\star}^2\frac{1}{144}\sigma^2(D-3)^{1/3} \\ 
    &\ge 4R^2\sigma^2 - 4 \sigma^4 (D-3) + 4\log(6) \sigma^4 D^{2/3}(D-3)^{1/3}\\
    &\ge 4R^2\sigma^2
\end{aligned}
\end{equation}
where in the second line we have used definition $\Delta_s = \tfrac{1}{12} \sigma (D-3)^{1/6}$,and in the third line we used the assumption ${s_\star}^2 \ge 96 \log(6) \times \sigma^2 D^{2/3}$
\end{proof} 

\paragraph{Convolutions with $\phi$.} The following lemma controls the convolution $\phi \ast - \dot{h}$:

\begin{lemma}[Upper and Lower Bounds with $\phi \ast -\dot{h}$] \label{lem:phi-conv-h-dot} Suppose 
\begin{equation}
24 \log(6) \times \sigma^2 D^{2/3} \le s_\star^2 \le 3 \sigma^2 D
\end{equation}
and $D > (48\log(6))^3 + 3$. Then 
\begin{eqnarray}
     \Bigl[\phi \ast -\dot{h}\Bigr] (s) &\le& 9 e C(D)\frac{\nu}{(\nu^2 + \sigma^2)^{1/2}} \frac{s_\star}{\sigma^2} \exp\left( - \frac{\left( s - s_\star \right)^2 }{ 2 (\nu^2+\sigma^2) }\right) \label{eqn:phi-conv-h-dot-upper}  \\
          \Bigl[\phi \ast -\dot{h}\Bigr] (s) &\ge& \frac{1}{16 e} C(D) \frac{\nu}{(\nu^2 + \sigma^2)^{1/2}} \frac{s_\star}{\sigma^2} \exp\left( - \frac{\left( s - s_\star \right)^2 }{ 2 ( \nu^2 + \sigma^2 ) }\right) 
\end{eqnarray}
when
\begin{equation}
|s-s_\star| \le \tfrac{1}{6} \sigma (D-3)^{1/6}.
\end{equation}
\end{lemma} 
\vspace{.1in}
\begin{proof} Since 
$s_\star^2 \le 3 \sigma^2 D, D\ge 48^3 + 3$, we have 
\begin{equation}
\nu^2 = \frac{\sigma^4 (D-3)}{2 s_\star^2} \ge \frac{\sigma^2}{6},
\end{equation}
which in turn implies that 
\begin{equation}
    \frac{\nu}{(\nu^2 + \sigma^2)^{1/2}} = \left( \frac{\nu^2}{\nu^2 + \sigma^2} \right)^{1/2} \ge \left( \frac{1/6}{1/6 + 1} \right)^{1/2} > \frac{1}{3}.
\end{equation}
Set 
\begin{eqnarray}
s_{\min}&=& s_\star - \tfrac{1}{3} \sigma (D-3)^{1/6}, \quad s_{\max} \;=\; s_\star + \tfrac{1}{3} \sigma (D-3)^{1/6}, \\ 
s_{\min}' &=& s_\star - \tfrac{1}{6} \sigma (D-3)^{1/6}, \quad s_{\max}' \;=\; s_\star + \tfrac{1}{6} \sigma (D-3)^{1/6}.
\end{eqnarray}
We have 
\begin{eqnarray}
\exp\left( - \frac{(s_{\max}' - s_\star)^2}{2 ( \nu^2 + \sigma^2 ) } \right) &\ge& \exp\left( - \frac{(s_{\max}' - s_\star)^2}{2 \nu^2 } \right) \nonumber \\ &\ge& \exp\left( - \frac{1}{4} \frac{(s_{\max} - s_\star)^2}{2 \nu^2 } \right) 
\;=\; \left[ \exp\left( -\frac{(s_{\max} - s_\star)^2}{2 \nu^2 } \right) \right]^{1/4}
\end{eqnarray}
and so 
\begin{align}
\exp\left( -\frac{(s_{\max} - s_\star)^2}{2 \nu^2 } \right) &\le \frac{\nu}{(\nu^2+\sigma^2)^{1/2}} \exp\left( - \frac{(s_{\max}' - s_\star)^2}{2 ( \nu^2 + \sigma^2 ) } \right) \times 3 \left[ \exp\left( - \frac{(s_{\max}' - s_\star)^2}{2 ( \nu^2 + \sigma^2 ) } \right) \right]^3 \\ 
&\le \frac{\nu}{(\nu^2+\sigma^2)^{1/2}} \exp\left( - \frac{(s_{\max}' - s_\star)^2}{2 ( \nu^2 + \sigma^2 ) } \right) \times 3 \exp\left( - \frac{(D-3)^{1/3}}{24} \frac{\sigma^2}{\nu^2+ \sigma^2} \right).
\end{align}
Under our hypotheses on $s_\star$ and $D$, 
\begin{equation}
\frac{(D-3)^{1/3}}{24} \frac{\sigma^2}{\nu^2+\sigma^2} \ge \log 6,
\end{equation}
and so for any $s \in [s_{\min}',s_{\max}']$, 
\begin{equation} \label{eqn:sg-conv-comparison}
\max \left\{ \exp\left( -\frac{(s_{\max} - s_\star)^2}{2 \nu^2 } \right), \exp\left( -\frac{(s_{\min} - s_\star)^2}{2 \nu^2 } \right) \right\} \le \frac{1}{2} \frac{\nu}{(\nu^2 + \sigma^2)^{1/2}} \exp\left( - \frac{(s - s_\star)^2 }{2 ( \nu^2 + \sigma^2) }\right).
\end{equation}
We use this relationship to establish the desired bounds: 

\vspace{.1in}

\noindent {\em Upper bound:} Write 
\begin{equation}
    \omega(s) = 6 e C(D) \frac{s_\star}{\sigma^2} \exp\left( - \frac{(s - s_\star)^2}{2 \nu^2} \right),
\end{equation}
and notice that on $[s_{\min},s_{\max}]$ the subgaussian approximation  \eqref{eqn:hdot-subgaussian} is valid. Then, using the monotonicity of $-\dot{h}$ from lemma \ref{lem:monotoncity of -h dot}, for all $s$ in the range, we have 
\begin{equation}
    -\dot{h}(s) \le \omega(s) + \max\Bigl\{ \omega(s_{\min}),\omega(s_{\max}) \Bigr\} \indicator{ s < s_{\min} \cup s > s_{\max} }, 
\end{equation}
whence 
\begin{eqnarray}
    \bigl[\phi \ast -\dot{h}\bigr](s) &\le& \phi \ast \omega(s) + \max\Bigl\{ \omega(s_{\min}),\omega(s_{\max}) \Bigr\} \\
     &=& 6e C(D) \frac{s_\star}{\sigma^2} \frac{\nu}{(\nu^2 + \sigma^2)^{1/2}} \exp\left( - \frac{(s-s_\star)^2}{2 (\nu^2 + \sigma^2 )} \right) +  \max\Bigl\{ \omega(s_{\min}),\omega(s_{\max}) \Bigr\},
\end{eqnarray}
where the first term follows from standard convolution identities (c.f., the convolution of two gaussians is itself a gaussian). Using \eqref{eqn:sg-conv-comparison}
\begin{equation}
\max\Bigl\{ \omega(s_{\min}),\omega(s_{\max}) \Bigr\} \le 3e C(D) \frac{s_\star}{\sigma^2} \frac{\nu}{(\nu^2 + \sigma^2)^{1/2}} \exp\left( - \frac{(s-s_\star)^2}{2 (\nu^2 + \sigma^2 )} \right) 
\end{equation}
for any $s \in [ s_{\min}', s_{\max}' ]$. This gives the claimed upper bound. 

\vspace{.1in}

\noindent {\em Lower bound:} Similarly, letting
\begin{equation}
    \eta(s) = \frac{1}{8 e} C(D) \frac{s_\star}{\sigma^2} \exp\left( - \frac{(s - s_\star)^2}{2 \nu^2} \right)
\end{equation}
we have 
\begin{equation}
-\dot{h}(s) \ge \eta(s) - \eta(s_{\min}) \indicator{s < s_{\min}} - \eta( s_{\max}) \indicator{s > s_{\max}}, 
\end{equation}
and
\begin{align} \label{eqn:phi-ast-h-dot-inter-lower}
\bigl[ \phi \ast - \dot{h} \bigr] (s) &\ge \frac{1}{8e} C(D) \frac{s_\star}{\sigma^2} \frac{\nu}{(\nu^2 + \sigma^2)^{1/2}} \exp\left( - \frac{(s-s_\star)^2}{2 (\nu^2 + \sigma^2 )} \right) -  \max\Bigl\{ \eta(s_{\min}),\eta(s_{\max}) \Bigr\}. 
\end{align}
By the same reasoning above, we have 
\begin{equation}
 \max\Bigl\{ \eta(s_{\min}),\eta(s_{\max}) \Bigr\} \le \frac{1}{16e} C(D) \frac{s_\star}{\sigma^2} \frac{\nu}{(\nu^2 + \sigma^2)^{1/2}} \exp\left( - \frac{(s-s_\star)^2}{2 (\nu^2 + \sigma^2 )} \right)
\end{equation}
for any $s \in [ s_{\min}', s_{\max}' ]$. Combining with \eqref{eqn:phi-ast-h-dot-inter-lower} gives the claimed lower bound.   
\end{proof} 

\begin{lemma}[Upper and Lower Bounds on $\phi \ast h$]  
\label{lem:phi-conv-h} Suppose that 
\begin{equation}
576\log(2) \times \sigma^2 D^{2/3} \le s_\star^2 \le 3 \sigma^2 D
\end{equation}
and 
\begin{equation}
D > (576\log(2))^3 + 3
\end{equation}
We have 
\begin{eqnarray}
\Bigl[ \phi \ast h \Bigr](s) &\ge&  \begin{cases} \frac{1}{64e^2} \cdot C(D) \frac{s_\star \nu}{\sigma^2} & 0 \le s \le s_\star + \bar{\nu}  \\ \frac{1}{64e} \cdot C(D) \frac{s_\star \nu}{\sigma^2} \frac{\bar{\nu}}{s - s_\star}  \exp\left( - \frac{(s-s_\star)^2}{2 (\nu^2 + \sigma^2)} \right) &  s_\star + \bar{\nu} \le s \le s_\star + \tfrac{1}{12} \sigma (D-3)^{1/6} \end{cases} \quad 
\\
\Bigl[ \phi \ast h \Bigr](s) &\le& \begin{cases} 1 & 0 \le s < \check{s}_\star\\
 4 \exp\left( - \frac{ ( s - \check{s}_\star )^2 }{ 2 ( \check{\nu}^2 + \sigma^2 ) } \right) & s \ge \check{s}_\star.
\end{cases} 
\end{eqnarray}
where 
\begin{eqnarray}
    \check{s}_\star^2 &=& R^2 - \sigma^2(D-1),\\
    \check{\nu}^2 &=& \frac{\sigma^4 (D-1)}{2 \check{s}_\star^2}.
\end{eqnarray}
\end{lemma}

\begin{proof}
Set $s_{\max} = s_\star + \tfrac{1}{6} \sigma (D-3)^{1/6}$. By Lemma \ref{lem:phi-conv-h-dot} any $s \in (s_\star,s_{\max})$ we have 
\begin{eqnarray}
    \bigl[\phi \ast h\bigr](s) &=& \bigl[\phi \ast h\bigr](s_{\max}) - \int_{t = s}^{s_{\max}} \frac{d}{du} [ \phi \ast h ]\Bigr|_{u=t} dt \\
    &=& \bigl[\phi \ast h\bigr](s_{\max}) - \int_{t = s}^{s_{\max}}  [ \phi \ast \dot{h} ](t) \, dt \\
    &\ge&  \int_{t = s}^{s_{\max}}  [ \phi \ast -\dot{h} ](t) \, dt \\
    &\ge& \frac{1}{16 e} C(D) \frac{s_\star}{\sigma^2} \frac{\nu}{(\nu^2+ \sigma^2)^{1/2}} \int_{t = s}^{s_{\max} } \exp\Bigl( - \frac{(t-s_\star)^2}{2 (\nu^2 + \sigma^2)} \Bigr) dt \\
    &=& \frac{1}{16 e} C(D) \frac{\nu s_\star}{\sigma^2}  \int_{u = \frac{s - s_\star}{(\nu^2 + \sigma^2)^{1/2}}}^{\frac{ s_{\max} - s_\star }{(\nu^2+\sigma^2)^{1/2}}} \exp\Bigl( - \frac{u^2}{2} \Bigr) du \\
    &=& \left[ \frac{\sqrt{2 \pi}}{16 e} C(D) \frac{\nu s_\star}{\sigma^2} \right] \Biggl( Q\left( \frac{s - s_\star}{\bar{\nu}} \right) - Q\left( \frac{s_{\max} - s_\star }{\bar{\nu}} \right) \Biggr) 
    \end{eqnarray}
    where $Q(\cdot)$ is the Gaussian $Q$-function, $\bar{\nu} = (\nu^2 + \sigma^2)^{1/2}$ and we  have used the change of variables 
    \begin{equation}
    u = \frac{ t- s_\star }{ (\nu^2+ \sigma^2)^{1/2} }.
    \end{equation} 
Note that for arbitrary $u > \Delta \ge 0$,  
\begin{eqnarray}
    Q(u) &=& \frac{1}{\sqrt{2\pi}}\int_{v = u}^\infty e^{-v^2/2} dv = \frac{1}{\sqrt{2\pi}}\int_{v = u-\Delta}^\infty e^{-(v+\Delta)^2/2} dv \\ &\le& \frac{e^{-\Delta^2/2}}{\sqrt{2\pi}} \int_{v = u-\Delta}^\infty e^{-v^2/2} dv \\ &\le&  e^{-\Delta^2/2} Q(u - \Delta). 
\end{eqnarray}
Set $s_{\max}' = s_\star + \tfrac{1}{12} \sigma(D-3)^{1/6}$. Then 
\begin{eqnarray}
    Q\left( \frac{s_{\max} - s_\star }{\bar{\nu}} \right) &\le& \exp\left( - \frac{(s_{\max} - s_{\max}')^2}{2 \bar{\nu}^2} \right)     Q\left( \frac{s_{\max}' - s_\star }{\bar{\nu}} \right) \\
    &\le& \exp\left( - \frac{\sigma^2 (D-3)^{1/3}}{288 \bar{\nu}^2} \right)     Q\left( \frac{s_{\max}' - s_\star }{\bar{\nu}} \right) \\
    &\le& \frac{1}{2} Q\left( \frac{s_{\max}' - s_\star }{\bar{\nu}} \right), 
\end{eqnarray}
where in the final line we have used the fact that our hypotheses on $s_\star$ and $D$ imply 
\begin{equation}
\frac{(D-3)^{1/3}}{288} \frac{\sigma^2}{\nu^2+\sigma^2} \ge \log 2,
\end{equation}
which holds because 
\begin{equation}
    \nu^2 = \frac{\sigma^4(D-3)}{{s_\star}^2} \le \frac{\sigma^4(D-3)}{576 \log(6) \sigma^2 D^{2/3}} \le \frac{\sigma^2(D-3)^{1/3}}{576 \log(2)}
\end{equation}
and 
\begin{equation}
    \frac{(D-3)^{1/3}}{288} \frac{\sigma^2}{\nu^2+\sigma^2} \ge \frac{(D-3)^{1/3}}{288} \frac{1}{\frac{(D-3)^{1/3}}{576 \log(2)}+1} =  \frac{576 \log(2)}{288} \frac{(D-3)^{1/3}}{(D-3)^{1/3}+576 \log(2)}\ge \log 2
\end{equation}
Therefore, for $s \in [s_\star, s'_{\max}]$, 
\begin{eqnarray}
    [\phi \ast h](s) \ge \left[ \frac{\sqrt{2 \pi}}{32 e} C(D) \frac{\nu s_\star}{\sigma^2} \right] Q\left( \frac{s - s_\star}{\bar{\nu}} \right).
\end{eqnarray}
Using the lower bound $Q(u) > (u / (1+ u^2)) \phi(u)$ for $u > 0$,  
\begin{equation}
    Q(u) \ge  \frac{1}{2\sqrt{2 \pi} u} \exp\left( - \frac{u^2}{2} \right)  \qquad u \ge 1
\end{equation}
we obtain 
\begin{equation}
   [\phi \ast h](s) \ge \left[ \frac{1}{64 e} C(D) \frac{\nu s_\star}{\sigma^2} \right] \frac{\bar{\nu}}{s - s_\star} \exp\left( - \frac{(s - s_\star)^2}{2 \bar{\nu}^2} \right)  \qquad  s_\star + \bar{\nu} \le s \le s_{\max}'.
\end{equation}
For $0 \le s < s_\star + \bar{\nu}$, monotonicity of $\phi \ast h$ gives 
\begin{equation}
   [\phi \ast h](s) \ge \frac{1}{64 e^2} C(D) \frac{\nu s_\star}{\sigma^2}.
\end{equation}

\vspace{.1in}

\noindent For the upper bound, we start with the following subgaussian bound for the lower tail of a $\chi$-square random variable (see \cite{laurent2000adaptive})
\begin{equation}
\bb P_{X \sim \chi^2_{D-1}} \Bigl[ \, X \le \bb E[X] - t \, \Bigr] \le \exp\left( - \frac{t^2}{4 (D-1)} \right)
\end{equation}
Setting $\check{s}_\star^2 = R^2 - \sigma^2(D-1)$, for $s > \check{s}_\star$ we have 
\begin{eqnarray}
h(s) &\le& \exp\left( - \frac{ \left(\frac{R^2 - s^2}{\sigma^2}- (D-1) \right)^2 }{ 4 (D-1) } \right) \\
&\le& \exp\left( - \frac{ \left( R^2 - s^2 - \sigma^2 (D-1) \right)^2 }{ 4 \sigma^4 (D-1) } \right) \\
&\le& \exp\left( - \frac{ \left( s^2 - \check{s}_\star^2 \right)^2 }{ 4 \sigma^4 (D-1) } \right) \\ 
&\le& \exp\left( - \frac{(s - \check{s}_\star)^2 ( s + \check{s}_\star)^2}{ 4 \sigma^4 (D-1)} \right) \\
&\le& \exp\left( - \frac{4 \check{s}_\star^2 (s - \check{s}_\star)^2 }{ 4 \sigma^4 (D-1)} \right) \\
&\le& \exp\left( - \frac{(s - \check{s}_\star)^2 }{ 2 \check{\nu}^2 } \right) 
\end{eqnarray}  
with 
\begin{equation}
\check{\nu}^2 = \frac{\sigma^4 (D-1)}{2 \check{s}_\star^2}.
\end{equation}
In total, we have 
\begin{equation}
h(s) \le \begin{cases} 1 & s \le \check{s}_\star \\ 
\exp\left( -\frac{(s - \check{s}_\star)^2}{2 \check{\nu}^2} \right) & s > \check{s}_\star  \end{cases}
\end{equation} 
Writing $\omega(s) = \exp\left( -\frac{(s - \check{s}_\star)^2}{2 \check{\nu}^2} \right)$, we have 
$h(s) \le \omega(s) + \indicator {s \le \check{s}_\star }$, and so for $s > \check{s}_\star  + \sigma$, we have 
\begin{eqnarray}
[\phi \ast h](s) &\le& [\phi \ast \omega](s) + \int_{s' = 0}^{ \check{s}_\star } \phi(s-s') \, ds',  \\
 &\le& \frac{\check{\nu}}{\sqrt{\check{\nu}^2 + \sigma^2}} \exp \left( - \frac{ ( s - \check{s}_\star )^2 }{2 ( \check{\nu}^2 + \sigma^2 )} \right) + Q\left(\frac{s - \check{s}_\star}{\sigma} \right), \\
 &\le&  \frac{\check{\nu}}{\sqrt{\check{\nu}^2 + \sigma^2}} \exp \left( - \frac{ ( s - \check{s}_\star )^2 }{2 ( \check{\nu}^2 + \sigma^2 )} \right) + \frac{\sigma}{s - \check{s}_\star} \exp\left( - \frac{ (s - \check{s}_\star)^2}{2 \sigma^2} \right). 
 \end{eqnarray}
When {$\check{s}_\star^2 \le 3 \sigma^2 (D-2)$}, we have $\check{\nu}^2 = \frac{\sigma^4 (D-1)}{2 \check{s}_\star^2}\ge\frac{\sigma^2(D-1)}{2(3D-2)}\ge\frac{\sigma^2}{8}$, so
\begin{equation}
    \frac{\check{\nu}}{\sqrt{\check{\nu}^2 + \sigma^2}} \ge \frac{1}{3}.
\end{equation}
Using this fact to combine terms, we have that for $s > \check{s}_\star + \sigma$, 
\begin{eqnarray}
    [\phi \ast h](s) &\le& 4 \frac{\check{\nu}}{\sqrt{\check{\nu}^2 + \sigma^2}} \exp \left( - \frac{ ( s - \check{s}_\star )^2 }{2 ( \check{\nu}^2 + \sigma^2 )} \right) \le 4 \exp \left( - \frac{ ( s - \check{s}_\star )^2 }{2 ( \check{\nu}^2 + \sigma^2 )} \right),
\end{eqnarray}
as claimed. 
Lastly, we note that for $ \check{s}_\star + \sigma \ge s \ge \check{s}_\star$, $4 \exp \left( - \frac{ ( s - \check{s}_\star )^2 }{2 ( \check{\nu}^2 + \sigma^2 )} \right)\ge 4 \exp \left( - \frac{\sigma^2}{2 ( \check{\nu}^2 + \sigma^2 )} \right) \ge 4e^{-1/2}\ge 1$, so this upper bound works whenever $s \ge \check{s}_\star$
\end{proof}

\begin{lemma}[Relative Bounds] \label{lem:relative} 
Suppose that 
\begin{equation}
576\log(2) \times \sigma^2 D^{2/3} \le s_\star^2 \le 3 \sigma^2 D
\end{equation}
and 
\begin{equation}
D > (576\log(2))^3 + 3
\end{equation}
 then when $0 \le s \le s_\star + \tfrac{1}{12} \sigma (D-3)^{1/6}$, we have 
\begin{equation}
\frac{\phi \ast - \dot{h}(s)}{ \phi \ast h(s) } \le C \max \left\{ \frac{1}{\nu}, \frac{s - s_\star}{\nu^2} \right\} 
\end{equation}
\end{lemma} 
\begin{proof}
We consider two cases: $0 \le s \le s_\star + \bar{\nu}$ and $s_\star + \bar{\nu} < s < s_\star + \tfrac{1}{12} \sigma (D-3)^{1/6}$:

\vspace{.1in}

\noindent {\em Case 1, $0 \le s \le s_\star + \bar{\nu}$:} For the first case,{using monotonicity of $\phi \ast - \dot{h}$} from Lemma \ref{lem:monotoncity of -h dot} on $[0, s_\star - \tfrac{1}{12} \sigma (D-3)^{1/6}]$ as well as the upper bound from Lemma \ref{lem:phi-conv-h-dot}, we have 
\begin{equation}
\phi \ast - \dot{h}(s) \le \frac{9e C(D) \nu s_\star }{ \bar{\nu} \sigma^2 }.
\end{equation}
Meanwhile by Lemma \ref{lem:phi-conv-h}
\begin{equation}
\phi \ast h(s) \ge \frac{1}{64 e^2} C(D) \frac{s_\star \nu}{\sigma^2},
\end{equation}
whence 
\begin{equation}
\frac{\phi \ast - \dot{h}(s)}{ \phi \ast h(s) } \le \frac{64 \times 9 e^3}{\bar{\nu}}.
\end{equation}

\vspace{.1in}

\noindent {\em Case 2, $s_\star + \bar{\nu} < s < s_\star + \tfrac{1}{12} \sigma (D-3)^{1/6}$:} Under the stated condition on $s$, we have 
\begin{equation}
\phi \ast - \dot{h}(s) \le 9 e C(D) \frac{\nu s_\star}{\bar{\nu} \sigma^2} \exp\left( - \frac{(s-s_\star)^2}{ 2 \bar{\nu}^2} \right),
\end{equation}
and  
\begin{equation}
\phi \ast h(s) \ge \frac{1}{64 e} C(D) \frac{s_\star \nu}{\sigma^2} \frac{\bar{\nu}}{s - s_\star} \exp\left( - \frac{(s- s_\star)^2}{2 \bar{\nu}^2} \right). 
\end{equation}
giving 
\begin{equation}
\frac{\phi \ast -\dot{h}(s) }{\phi \ast h(s)} \le 64 \times 9 e^2 \frac{s - s_\star}{\bar{\nu}^2},
\end{equation} 
establishing the claimed bounds. \end{proof}

\section{Controlling the Signal Average} \label{sec:signal-avg}
In this section, we consider the following situation: we have a noisy landmark $\mb q$, which does not necessarily reside on $\mc M$. Let 
\begin{equation}
    \mb q_\natural = P_{\mc M} \mb q
\end{equation}
denote the nearest point to $\mb q$ on $\mc M$. In this section, we attempt to bound 
\begin{equation}
    \bb E\Bigl[ \, d_{\mc M}^2(\mb x_\natural,\mb q_\natural) \mid \| \mb x - \mb q \| \le R \, \Bigr].
\end{equation}
The following lemma does the job: 

\begin{lemma}[Intrinsic Distances of Grouped Points]\label{lem:expected-signal-distance} 
There exist constants $c_1, c_2, c_3, c_4, C_1,C_2$ such that whenever $D > C_1$, $\sigma\sqrt{D} \le \frac{\tau}{2}$, $d\left(\mb q, \mc M\right) \le \min\{\frac{c_1}{\kappa}, \frac{1}{2}\tau\}$, and  $s_\star = R^2 - \sigma^2 (D-3)$ satisfies 
{
\begin{equation}
\max\Bigl\{ 576 \log 2 \times \sigma^2 D^{2/3}, d^2(\mb q, \mc M) \Bigr\} \le s_\star^2 \le \min\left\{ 3 \sigma^2 D, d^2(\mb q, \mc M) + c_2\tau^2\right\} 
\end{equation}
}
and $s_{\star,\parallel}^2 = s_\star^2 - d^2\left(\mb q, \mc M\right), \: \check{s_{\star,\parallel}}^2 =s_{\star,\parallel}^2 - 2\sigma^2$ satisfies 
\begin{equation}
\sigma \sqrt{\log \left(\frac{\mr{diam}(\mc M)} {\check{s}_{\star,\parallel}}\right) + d \times \kappa \mr{diam}(\mc M) + d\log\left(\frac{1}{\kappa s_{\star,\parallel}}\right)} \le c_3\tau
\end{equation}
and 
\begin{equation}
    \check{s}_{\star,\parallel}^2 +\sqrt{\log \left(\frac{\mr{diam}(\mc M)}{\check{s}_{\star,\parallel}}\right) + d \times \kappa \mr{diam}(\mc M) + d\log\left(\frac{1}{\kappa s_{\star,\parallel}}\right)} \times \sigma^2 \sqrt{D} \le c_4 \check{s}_{\star}^2 
\end{equation}
we have 
\begin{equation}\label{eqn: expected value of signal average}
    \bb E\Bigl[ \, d_{\mc M}^2(\mb x_\natural,\mb q_\natural) \mid \| \mb x - \mb q  \| \le R \, \Bigr] \le C_2\check{s}_{\star,\parallel}^2 + C_2\sqrt{\log \left(\frac{\mr{diam}(\mc M)}{\check{s}_{\star,\parallel}}\right) + d \times \kappa \mr{diam}(\mc M) + d\log\left(\frac{1}{\kappa s_{\star,\parallel}}\right)} \times \sigma^2 \sqrt{D}
\end{equation} 
and 
\begin{equation}\label{eqn: variance of signal average}
    \bb E\Bigl[ \, d_{\mc M}^4(\mb x_\natural,\mb q_\natural) \mid \| \mb x - \mb q  \| \le R \, \Bigr] \le C_3\check{s}_{\star,\parallel}^4 + C_3\log \left(\frac{\mr{diam}(\mc M)}{\check{s}_{\star,\parallel}}\right) + d \times \kappa \mr{diam}(\mc M) + d\log\left(\frac{1}{\kappa s_{\star,\parallel}}\right) \times \sigma^4 D
\end{equation} 
\end{lemma}

\begin{proof}
Using the notation of the previous section, we have 
\begin{equation}
 \bb E\Bigl[ \, d_{\mc M}^2(\mb x_\natural,\mb q_\natural) \mid \| \mb x - \mb q  \| \le R \, \Bigr]  =    \frac{ \int_{\mb x_\natural \in \mc M} d^2( \mb x_\natural, \mb q_\natural ) \; [\phi \ast h] \Bigl( \| \mb q - \mb x_\natural \| \Bigr) \, d\mu_0( \mb x_\natural ) }{ \int_{\mb x_\natural \in \mc M} [\phi \ast h]\Bigl( \| \mb q - \mb x_\natural \| \Bigr)\, d\mu_0( \mb x_\natural ) }.
\end{equation}
We break this average into a ``near'' term and a ``far term''. Indeed, for any $\xi > 0$, we have 
\begin{eqnarray}
 \lefteqn{   \frac{ \int_{\mb x_\natural \in \mc M} d^2( \mb x_\natural, \mb q_\natural ) \; [\phi \ast h] \Bigl( \| \mb q - \mb x_\natural \| \Bigr) \, d\mu_0( \mb x_\natural ) }{ \int_{\mb x_\natural \in \mc M} [\phi \ast h]\Bigl( \| \mb q - \mb x_\natural \| \Bigr)\, d\mu_0( \mb x_\natural ) } }\nonumber \\
  &=& \frac{ \int_{\mb x_\natural \in \mc M, d_{\mc M}(\mb x_\natural, \mb q_\natural) < \xi} d^2( \mb x_\natural, \mb q_\natural ) \; [\phi \ast h] \Bigl( \| \mb q - \mb x_\natural \| \Bigr) \, d\mu_0( \mb x_\natural ) }{ \int_{\mb x_\natural \in \mc M} [\phi \ast h]\Bigl( \| \mb q - \mb x_\natural \| \Bigr)\, d\mu_0( \mb x_\natural ) }  \nonumber \\
  && \quad + \quad \frac{ \int_{\mb x_\natural \in \mc M, d_{\mc M}(\mb x_\natural, \mb q_\natural) \ge \xi} d^2( \mb x_\natural, \mb q_\natural ) \; [\phi \ast h] \Bigl( \| \mb q - \mb x_\natural \| \Bigr) \, d\mu_0( \mb x_\natural ) }{ \int_{\mb x_\natural \in \mc M} [\phi \ast h]\Bigl( \| \mb q - \mb x_\natural \| \Bigr)\, d\mu_0( \mb x_\natural ) }  \nonumber \\
 &\le&\label{eqn: bound signal average by near and far term} \xi^2 + \frac{\mr{diam}^2(\mc M) \, \mr{vol}(\mc M) }{\Upsilon} \sup_{d_{\mc M}(\mb x_\natural,\mb q_\natural) \ge \xi} \Biggl\{  [\phi \ast h]\Bigl( \| \mb q - \mb x_\natural \| \Bigr) \Biggr\} 
\end{eqnarray}
where equation \ref{eqn: bound signal average by near and far term} controls the near term by bounding the worst case distance $\xi^2$ and controls the far term by bounding worst case function value of $[\phi \ast h].$ From Lemma \ref{lem:vol-ratio}, we have 
\begin{equation}\label{eqn: bound vol(M) over Upsilon}
    \frac{\mr{vol}(\mc M)}{\Upsilon} \le \frac{C \exp \left(\kappa(d-1)\diam (\mathcal{M})\right)}{\min\{ (\kappa s_{\star, ||})^d, 1 \} }.
\end{equation}
From Lemma \ref{lem:phi-conv-h}, for $ \check{s}_\star < s \le 2 \check{s}_\star$, 
we have
\begin{eqnarray} 
\Bigl[ \phi \ast h \Bigr](s) &\le& 4 \exp\left( - \frac{(s - \check{s}_\star)^2}{2 (\check{\nu}^2 + \sigma^2) } \right) \nonumber \\
&=& 4 \exp\left( - \frac{(s^2 - \check{s}_\star^2)^2}{2 (\check{\nu}^2 + \sigma^2) } \times \frac{(s-\check{s}_\star)^2}{(s^2 - \check{s}_\star^2)^2} \right) \nonumber \\
&=& 4 \exp\left( - \frac{(s^2 - \check{s}_\star^2)^2}{2 (\check{\nu}^2 + \sigma^2) } \times \frac{1}{(s + \check{s}_\star )^2 } \right) \nonumber \\
&=& 4 \exp\left( - \frac{(s^2 - \check{s}_\star^2)^2}{2 (\check{\nu}^2 + \sigma^2) } \times \frac{1}{\check{s}_\star^2 \left( 2 + \frac{s-\check{s}_\star}{\check{s}_\star} \right)^2 } \right) \nonumber \\
&\le& 4 \exp\left( - \frac{(s^2 - \check{s}_\star^2)^2}{18 \check{s}_\star^2 (\check{\nu}^2 + \sigma^2) } \right) \nonumber \\
&\le& 4 \exp\left( - \frac{(s^2 - \check{s}_\star^2)^2}{162 \, \sigma^4 D} \right).
\end{eqnarray}
In the final line, we have used the fact that $\check{s}_\star^2 ( \check{\nu}^2 + \sigma^2 ) \le 9 \sigma^4 D$, which follows from the definition of $\check{\nu}$, the assumption that $s_\star^2 < 3 \sigma^2D$ and the fact that $\check{s}_\star \le s_\star$. Writing 
\begin{equation}
    \check{s}_\star^2 = d(\mb q, \mc M)^2 + \check{s}_{\star,\parallel}^2 
\end{equation}
From Lemma \ref{lem:far-distance}, for $\xi < \frac{\tau}{2} \le \tfrac{1}{2 \kappa}$, 
\begin{eqnarray}
  \inf_{d_{\mc M}(\mb x_\natural, \mb q_\natural) \ge \xi} \Bigl\{   \| \mb q - \mb x_\natural \|_2^2 \Bigr\} &\ge& \min\Bigl\{ \| \mb q - \mb q_\natural \|_2^2 + c \xi^2, \tau^2 \Bigr\} \quad=\quad  \| \mb q - \mb q_\natural \|_2^2 + c \xi^2,
\end{eqnarray}
because $\| \mb q - \mb q_\natural\|_2^2 + c \xi^2 \le \tau^2$, and by the monotonicity of $\phi \ast h$, we have 
\begin{eqnarray}
    \sup_{d_{\mc M}(\mb x_\natural,\mb q_{\natural}) \ge \xi} \left[ \phi \ast h \right]\Bigl( \| \mb q - \mb x_\natural \|_2 \Bigr) &\le&\label{sup of phi convolved with h} \left[ \phi \ast h \right]\Bigl( \left( \| \mb q - \mb q_{\natural} \|_2^2 + c \xi^2 \right)^{1/2} \Bigr).
\end{eqnarray}

Provided 
\begin{equation} 
\check{s}_\star < \Bigl( \| \mb q - \mb q_{\natural} \|_2^2 + c \xi^2 \Bigr)^{1/2} \le 2 \check{s}_\star,
\end{equation}
we have 
\begin{eqnarray}\label{eqn: bound phi convolve h when intrinsically far}
    \sup_{d_{\mc M}(\mb x_\natural,\mb q_{\natural}) \ge \xi} \left[ \phi \ast h \right]\Bigl( \| \mb q - \mb x_\natural \|_2 \Bigr) &\le& 4 \exp \left( - \frac{\left(  c \xi^2 - \check{s}^2_{\star,\parallel}  \right)^2}{162 \sigma^4 D} \right)
\end{eqnarray}
Combining our results from equations \eqref{eqn: bound signal average by near and far term}, \eqref{eqn: bound vol(M) over Upsilon}, and \eqref{eqn: bound phi convolve h when intrinsically far}, we have \begin{equation}
\label{eqn: substituted bound of signal average}E\Bigl[ \, d_{\mc M}^2(\mb x_\natural,\mb q_\natural) \mid \| \mb x - \mb q  \| \le R \, \Bigr] \leq \xi^2 + \diam(\mathcal{M})^2 \ \cdot \ C \frac{\exp (\kappa(d-1)\diam (\mathcal{M}))}{(\kappa s_{\star, ||} )^d} \ \cdot \ 4\exp\left(-\frac{(c\xi^2 - \check{s}_{\star,||}^2)^2}{162\sigma^4 D}\right),
\end{equation}
as long as $\xi < \tfrac{\tau}{2}$ and $\check{s}_\star^2 \le d^2(\mb q, \mc M) + c \xi^2 \le 4 \check{s}_\star^2.$
Choosing 
\begin{align}
    \xi^2 &=  c^{-1} \check{s}_{\star,\parallel}^2 + c^{-1} \sqrt{\log \left( C\frac{\mr{diam}^2(\mc M)}{\check{s}_{\star,\parallel}^2} \exp\left( d \times \kappa \mr{diam}(\mc M) \right) (\kappa s_{\star,\parallel})^{-d} \right)} \times \sqrt{162 \sigma^4 D},\\
    &\le  C\check{s}_{\star,\parallel}^2 + C' \sqrt{\log \left(\frac{\mr{diam}(\mc M)}{\check{s}_{\star,\parallel}}\right) + d \times \kappa \mr{diam}(\mc M) + d\log\left(\frac{1}{\kappa s_{\star,\parallel}}\right)} \times \sigma^2 \sqrt{D}
\end{align} 
the bounds \eqref{eqn: substituted bound of signal average} gives 
\begin{eqnarray}
    &&\bb E\left[ \, d^2_{\mc M}\left(\mb x_\natural, \mb q_\natural \right)  \, | \, \| \mb x- \mb q \| \le R \, \right] \\
    &\le& \xi^2 + C \check{s}_{\star,\parallel}^2.  \\ 
    &=& C'\check{s}_{\star,\parallel}^2 + C'' \sqrt{\log \left(\frac{\mr{diam}(\mc M)}{\check{s}_{\star,\parallel}}\right) + d \times \kappa \mr{diam}(\mc M) + d\log\left(\frac{1}{\kappa s_{\star,\parallel}}\right)} \times \sigma^2 \sqrt{D}
\end{eqnarray}
where in the first line we plugged our definition of $\xi^2$ and in the next line we applied the upper bound on $\xi^2.$ To verify that $\xi^2 < \tfrac{\tau^2}{4}$, we note our assumption gives 
$s_\star^2 \le d^2(\mb q, \mc M) + c_1\tau^2$, so it's sufficient to verify 
\begin{equation}
  \sqrt{\log \left(\frac{\mr{diam}(\mc M)}{\check{s}_{\star,\parallel}}\right) + d \times \kappa \mr{diam}(\mc M) + d\log\left(\frac{1}{\kappa s_{\star,\parallel}}\right)} \times \sigma^2 \sqrt{D}\le c\tau^2
\end{equation}
which holds given our assumption \[\sigma\sqrt{D} \le \frac{\tau}{2}\quad \text{and} 
\quad\sigma \sqrt{\log \left(\frac{\mr{diam}(\mc M)} {\check{s}_{\star,\parallel}}\right) + d \times \kappa \mr{diam}(\mc M) + d\log\left(\frac{1}{\kappa s_{\star,\parallel}}\right)} \le c_2\tau.\]
To check $\check{s}_\star^2 \le d^2(\mb q, \mc M) + c \xi^2 \le 4 \check{s}_\star^2$, we note the lower bound trivially gets satisfied because $\xi^2 \ge C\check{s_{\star,\parallel}}^2$ and here $C \ge c^{-1}$ and the upper bound follows from our assumption that
\begin{equation}
\check{s}_{\star,\parallel}^2 +\sqrt{\log \left(\frac{\mr{diam}(\mc M)}{\check{s}_{\star,\parallel}}\right) + d \times \kappa \mr{diam}(\mc M) + d\log\left(\frac{1}{\kappa s_{\star,\parallel}}\right)} \times \sigma^2 \sqrt{D} \le c \check{s}_{\star}^2 
\end{equation}
Similarly, to show \ref{eqn: variance of signal average}, we  apply the exact same argument as before,
\begin{eqnarray}
    \bb E\Bigl[ \, d_{\mc M}^4(\mb x_\natural,\mb q_\natural) \mid \| \mb x - \mb q  \| \le R \, \Bigr]  &=&    \frac{ \int_{\mb x_\natural \in \mc M} d^4( \mb x_\natural, \mb q_\natural ) \; [\phi \ast h] \Bigl( \| \mb q - \mb x_\natural \| \Bigr) \, d\mu_0( \mb x_\natural ) }{ \int_{\mb x_\natural \in \mc M} [\phi \ast h]\Bigl( \| \mb q - \mb x_\natural \| \Bigr)\, d\mu_0( \mb x_\natural ) }\\
    &=& \frac{ \int_{\mb x_\natural \in \mc M, d_{\mc M}(\mb x_\natural, \mb q_\natural) < \xi} d^4( \mb x_\natural, \mb q_\natural ) \; [\phi \ast h] \Bigl( \| \mb q - \mb x_\natural \| \Bigr) \, d\mu_0( \mb x_\natural ) }{ \int_{\mb x_\natural \in \mc M} [\phi \ast h]\Bigl( \| \mb q - \mb x_\natural \| \Bigr)\, d\mu_0( \mb x_\natural ) }  \nonumber \\
    &&\quad + \quad \frac{ \int_{\mb x_\natural \in \mc M, d_{\mc M}(\mb x_\natural, \mb q_\natural) \ge \xi} d^4( \mb x_\natural, \mb q_\natural ) \; [\phi \ast h] \Bigl( \| \mb q - \mb x_\natural \| \Bigr) \, d\mu_0( \mb x_\natural ) }{ \int_{\mb x_\natural \in \mc M} [\phi \ast h]\Bigl( \| \mb q - \mb x_\natural \| \Bigr)\, d\mu_0( \mb x_\natural ) }  \nonumber \\
 &\le& \xi^4 + \frac{\mr{diam}^4(\mc M) \, \mr{vol}(\mc M) }{\Upsilon} \sup_{d_{\mc M}(\mb x_\natural,\mb q_\natural) \ge \xi} \Biggl\{  [\phi \ast h]\Bigl( \| \mb q - \mb x_\natural \| \Bigr) \Biggr\} .
\end{eqnarray}
This time, choosing 
\begin{eqnarray}
    \bar{\xi}^2 &=&  c^{-1} \check{s}_{\star,\parallel}^2 + c^{-1} \sqrt{\log \left( C\frac{\mr{diam}^4(\mc M)}{\check{s}_{\star,\parallel}^4} \exp\left( d \times \kappa \mr{diam}(\mc M) \right) (\kappa s_{\star,\parallel})^{-d} \right)} \times \sqrt{162 \sigma^4 D},\\
    &\le&  C\check{s}_{\star,\parallel}^2 + C' \sqrt{\log \left(\frac{\mr{diam}(\mc M)}{\check{s}_{\star,\parallel}}\right) + d \times \kappa \mr{diam}(\mc M) + d\log\left(\frac{1}{\kappa s_{\star,\parallel}}\right)} \times \sigma^2 \sqrt{D},
\end{eqnarray} 
 the exponential bounds for $ \sup_{d_{\mc M}(\mb x_\natural,\mb q_{\natural}) \ge \bar{\xi}} \left[ \phi \ast h \right]\Bigl( \| \mb q - \mb x_\natural \|_2 \Bigr)$ which we previously developed in equation \ref{eqn: bound phi convolve h when intrinsically far} still holds, which gives us 
\begin{eqnarray}
    &&\bb E\left[ \, d^4_{\mc M}\left(\mb x_\natural, \mb q_\natural \right)  \, | \, \| \mb x- \mb q \| \le R \, \right] \\
    &\le& \bar{\xi}^4 + C \check{s}_{\star,\parallel}^4.  \\ 
    &\le& C'\check{s}_{\star,\parallel}^4 + C'' \log \left(\frac{\mr{diam}(\mc M)}{\check{s}_{\star,\parallel}}\right) + d \times \kappa \mr{diam}(\mc M) + d\log\left(\frac{1}{\kappa s_{\star,\parallel}}\right) \times \sigma^4 D
\end{eqnarray}
\end{proof}

\subsection*{Upper Bounding the Tail of $\phi \ast h$ over $\mc M$} 

\begin{lemma}[Points intrinsically far from $\mb q_\natural$ are far from $\mb q$.] \label{lem:far-distance} Suppose that $\| \mb q - \mb q_\natural \|_2 < \tfrac{1}{2}\tau$, where $\tau$ is the reach of $\mc M$. There exists a numerical constant $c > 0$ such that for $\xi < \tfrac{1}{ 2 \kappa }$,   
\begin{equation}
  \inf_{d_{\mc M}(\mb x_\natural, \mb q_\natural) \ge \xi} \Bigl\{   \| \mb q - \mb x_\natural \|_2^2 \Bigr\} \ge \min\Bigl\{ \| \mb q - \mb q_\natural \|_2^2 + c \xi^2, \tau^2 \Bigr\}.
\end{equation}
\end{lemma}

\begin{proof} 
We seek a lower bound on 
\begin{equation}
\inf_{d_{\mc M}(\mb x_\natural, \mb q_\natural) \ge \xi} f(\mb x_\natural) \equiv  \| \mb q - \mb x_\natural \|^2_2
\end{equation} 
This problem asks us to minimize a continuous function $f$ over the compact set $\{ \mb x_\natural \mid d_{\mc M}(\mb x_\natural, \mb q_\natural) \ge \xi \}$. By compactness, at least one minimizer exists. Moreover, every minimizer $\mb x_\natural^\star$ is either (i) a critical point of $f$ over $\mc M$, or (ii) occurs on the boundary, i.e., satisfies $d_{\mc M}(\mb x_\natural^\star,\mb q_\natural ) = \xi$. 

\vspace{.1in} 

\noindent {\em Case (i).} Consider the squared distance function $f(\mb x_\natural) = \| \mb q - \mb x_\natural\|_2^2$. Any critical point $\bar{\mb x}_\natural$ of this function satisfies 
\begin{equation}
\mb q - \bar{\mb x}_\natural \in N_{\bar{\mb x}_\natural} \mc M
\end{equation}
I.e., if we consider the normal mapping $(\mb x_\natural,\mb \eta)\mapsto N( \mb x_\natural, \mb \eta) = \mb x_\natural + \mb \eta$, then we have 
\begin{equation}
N(\bar{\mb x}_\natural, \mb q - \bar{\mb x}_\natural) = \mb q = N( \mb q_\natural, \mb q - \mb q_\natural ). 
\end{equation} 
Since the normal map is injective on 
\begin{equation}
\mc M \times \mr{int}\Bigl( B^D(0,\tau) \Bigr)
\end{equation}
this implies that $\| \mb q - \bar{\mb x}_\natural \| \ge \tau$. 

\vspace{.1in}

\noindent {\em Case (ii).} For $\mb x_\natural$ satisfying $d_{\mc M}(\mb q_\natural, \mb x_\natural) = \xi$, there exists $\mb v \in T_{\mb q_\natural} \mc M$ with $\| \mb v \|_2 = \xi$ such that $\mb x_\natural = \exp_{\mb q_\natural}(\mb v)$. We have 
\begin{eqnarray}
\| \exp_{\mb q_\natural} (\mb v) - \mb q \| &\ge& \| \mb q_\natural +  \mb v - \mb q \| - \tfrac{1}{2} \kappa \|\mb v \|^2 \\ 
&\ge&\label{eqn: intermediate for intrinsic extrinsic bound} \left( \| \mb q - \mb q_\natural \|^2 + \xi^2 \right)^{1/2} - \tfrac{1}{2} \kappa \xi^2.
\end{eqnarray} 
Since $\| \mb q - \mb q_\natural \| \le \frac{1}{2}\tau\le\frac{1}{2\kappa}$ and $\xi \le \frac{1}{2\kappa},$ squaring \ref{eqn: intermediate for intrinsic extrinsic bound} gives 
\begin{eqnarray}
    \| \mb x_\natural - \mb q \|_2^2 &\ge& \| \mb q - \mb q_\natural \|_2^2 + \xi^2 - \kappa \xi^2 \Bigl( \| \mb q - \mb q_\natural \|_2^2 + \xi^2 \Bigr)^{1/2} + \tfrac{1}{4} \kappa^2 \xi^4 \\ 
    &\ge& \| \mb q - \mb q_\natural \|_2^2 + \xi^2 - \tfrac{1}{\sqrt{2}} \xi^2 + \tfrac{1}{4} \kappa^2 \xi^4 \\
    &\ge& \| \mb q - \mb q_\natural \|_2^2 + \left( \frac{\sqrt{2} -1 }{ \sqrt{2} } \right) \xi^2. 
\end{eqnarray}
Combining the two cases, gives the claimed result.
\end{proof}

\subsection*{Upper Bounding $\mr{vol}(\mc M) / \Upsilon$} 

\begin{lemma}[Volumes] \label{lem:vol-ratio} Suppose that $\mc M$ is a connected, geodesically complete $d \ge 2$-dimensional submanifold of $\bb R^D$. For $\mb q \in \bb R^D$, set 
\begin{equation}
\Upsilon =  \int_{\mb x_\natural \in \mc M} [\phi \ast h]\Bigl( \| \mb q - \mb x_\natural \| \Bigr)\, d\mu_0( \mb x_\natural )
\end{equation}
There exist numerical constants $c> 0, C_1,C_2,C_3$ such that whenever if $D>C_1$, $d(\mb q, \mc M) < c /\kappa$, and 
\begin{equation}
\max\left\{ C_2 \sigma^2 D^{2/3}, d^2(\mb q, \mc M ) \right\} \le s_\star^2 \le 3 \sigma^2 D,
\end{equation}
we have 
\begin{equation}
\frac{\mr{vol}(\mc M)}{\Upsilon} \le  \frac{ C_3 \exp \Bigl( \kappa (d-1) \, \mr{diam}(\mc M) \Bigr) }{ \min\left\{ (\kappa s_{\star,\parallel})^d, 1 \right\} },
\end{equation}
with $s_{\star,\parallel}^2 = s_\star^2 - d^2(\mb q, \mc M)$.
\end{lemma}

\begin{proof} 
Set $\mb q_\natural = \mc P_{\mc M} \mb q$, so $d(\mb q,\mc M) = \| \mb q - \mb q_\natural \|_2$. For any choice of $\Delta > 0$, we have 
\begin{eqnarray}
\Upsilon = \int_{\mb x_\natural \in \mc M} [\phi \ast h]\Bigl( \| \mb q - \mb x_\natural \| \Bigr)\, d\mu_0( \mb x_\natural ) &\ge& \int_{d_{\mc M}( \mb x_\natural,\mb q_\natural) \le \Delta } \,  [\phi \ast h]\Bigl( \| \mb q - \mb x_\natural \| \Bigr)\, d\mu_0( \mb x_\natural ).
\end{eqnarray}
For $d_{\mc M}(\mb x_\natural, \mb q_\natural) \le \Delta \le \tfrac{1}{\kappa}$ and $c \le \tfrac{1}{4}$ we have 
\begin{eqnarray}
    \| \mb q - \mb x_\natural \|^2 
    &\le& \left( \| \mb q - \mb q_\natural - \log_{\mb q_\natural} \mb x_\natural \| + \tfrac{1}{2} \kappa \Delta^2 \right)^2 \nonumber \\
    &\le& \| \mb q - \mb q_\natural \|_2^2 + d^2_{\mc M}( \mb x_\natural, \mb q_\natural ) + \Bigl( \| \mb q - \mb q_\natural \|_2 + d_{\mc M}(\mb q_\natural,\mb x_\natural) \Bigr) \kappa \Delta^2 + \tfrac{1}{4} \kappa^2 \Delta^4 \nonumber \\
    &\le& \| \mb q - \mb q_\natural \|_2^2 + \Delta^2 + \Bigl( \frac{c}{\kappa} + \Delta  \Bigr) \kappa \Delta^2 + \tfrac{1}{4} \kappa^2 \Delta^4 \nonumber \\ 
    &\le& d^2(\mb q, \mc M) + \tfrac{5}{2} \Delta^2
\end{eqnarray}
and so, as long as 
\begin{equation}
    \Delta \le \min\left\{ \left( \frac{2}{5}\right)^{1/2} s_{\star,\parallel},  \; \frac{1}{\kappa} \right\},
\end{equation}
we have $\| \mb q - \mb x_\natural \|_2 \le s_\star$ and so by Lemma \ref{lem:phi-conv-h},
we have 
\begin{equation}
    [\phi \ast h ]\Bigl( \| \mb q - \mb x_\natural \| \Bigr) \ge \frac{1}{64e^2} \cdot C(D) \frac{s_\star \nu}{\sigma^2}
\end{equation}
where we recall $\nu^2 = \frac{\sigma^4 (D-3)}{2 s_\star^2},$ and $C(D) = \frac{
    \dot{\gamma}\left( \frac{D-1}{2}, \frac{D-3}{2} \right)}{ \Gamma(\frac{D-1}{2})},$ which is bounded in  Lemma \ref{lem:c-of-D}, further giving
\begin{equation}
    [\phi \ast h ]\Bigl( \| \mb q - \mb x_\natural \| \Bigr) \ge \frac{c}{\sqrt{D}}  \frac{s_\star \nu}{\sigma^2} = \frac{c}{\sqrt{D}} \sqrt{\frac{D-3}{2}} \ge c' 
\end{equation}
for $\mb x_\natural \in B_{\mc M}(\mb q_\natural, \Delta)$. Hence, 
\begin{eqnarray}
    \int_{\mb x_\natural \in \mc M} [\phi \ast h]\Bigl( \| \mb q - \mb x_\natural \| \Bigr) \, d\mu_0(\mb x_\natural ) &\ge& c' \mr{vol}\Bigl( B_{\mc M}(\mb q_\natural, \Delta ) \Bigr)
\end{eqnarray}
By Lemma \ref{lem:zovolumes}, 
\begin{equation}
    \frac{\mr{vol}\left(B_{\mc M}(\mb q_\natural,\Delta) \right)}{\mr{vol}(\mc M)} \ge \tfrac{1}{4} (2 \kappa \Delta)^d \exp \Bigl( - \kappa (d-1) \, \mr{diam}(\mc M) \Bigr),
\end{equation}
and hence 
\begin{equation}
    \frac{\mr{vol}(\mc M)}{\Upsilon} \le \frac{\mr{vol}(\mc M)}{c' \mr{vol}\left(B_{\mc M}(\mb q_\natural, \Delta) \right) } \le \frac{C \exp\Bigl( \kappa (d-1) \, \mr{diam}(\mc M) \Bigr) }{ (2 \kappa \Delta)^d }.
\end{equation}
Setting $\Delta = \tfrac{1}{2} \min\{ s_{\star,\parallel}, \kappa^{-1}\} $
completes the proof. 
\end{proof}

\section{Supporting Lemmas}

\subsection{Volumes} 

\begin{lemma}\label{lem:zovolumes} Suppose that $d \ge 2$. Then for any $\mb q \in \mc M$, 
\begin{equation}
    \frac{\mr{vol}\left(B_{\mc M}(\mb q,r) \right)}{\mr{vol}(\mc M)} \ge \tfrac{1}{4} (2 \kappa r)^d \exp \Bigl( - \kappa (d-1) \, \mr{diam}(\mc M) \Bigr)
\end{equation}
\end{lemma}  
\begin{proof}
We begin by noting that since the sectional curvatures of $\mc M$ are uniformly lower bounded by $-\kappa^2$, the Ricci curvarure satisfies 
\begin{equation}
    \mr{Ric} \ge -(d-1)\kappa^2. 
\end{equation}
Let $\mc K$ be a hyperbolic space of constant sectional curvature $-\kappa^2$. Then the Bishop-Gromov volume comparison inequality implies that for any $\mb q\in \mc M$
\begin{equation}
\frac{\mr{vol}(B_{\mc M}(\mb q,r))}{\mr{vol}(B_{\mc K}(\mb p,r))} \ge \frac{\mr{vol}(\mc M)}{ \mr{vol}( B_{\mc K}(\mb p,\mr{diam}(\mc M)) }
\end{equation}
for any $\mb p \in \mc K$, and so 
\begin{eqnarray}
\frac{\mr{vol}(B_{\mc M}(\mb q,r))}{\mr{vol}(\mc M)} &\ge& \frac{\mr{vol}(B_{\mc K}(\mb p,r))}{ \mr{vol}( B_{\mc K}(\mb p,\mr{diam}(\mc M)) } \\
&=& \frac{\int_{t = 0}^r \mr{sinh}^{d-1}(\kappa t) \, dt }{ \int_{t = 0}^{\mr{diam}(\mc M)} \mr{sinh}^{d-1}(\kappa t) \, dt } \\
&\ge& \frac{\int_{t=0}^r (\kappa t)^{d-1} dt }{ 2^{-d+1}\int_{t =0}^{\mr{diam}(\mc M)} e^{\kappa t (d-1)} dt } \label{eqn:vol-ratio-inter-1} \\ 
&=& \frac{\frac{1}{d} (2\kappa)^{d-1} r^d }{\frac{1}{\kappa (d-1)} ( e^{\kappa (d-1) \mr{diam}(\mc M)} - 1 ) } \\
&\ge& 2^{d-1} \kappa^d r^d  \left( \frac{d-1}{d} \right) \exp\Bigl( -\kappa (d-1) \, \mr{diam}(\mc M) \Bigr)
\end{eqnarray}
where in \eqref{eqn:vol-ratio-inter-1} 
we have used the relationships $\mr{sinh}(t) \le \frac{e^t}{2}$ and $\mr{sinh}(t) \ge t$. \end{proof}

\subsection{Projections} 
\begin{lemma} \label{lem:noisy-point distance} Let $\mb x = \mb x_\natural + \mb z$, with $\mb x_\natural \in \mc M$ and $\mb z \sim_{\mr{iid}} \mc N(0,\sigma^2)$. There exist constants $c_1, c_2 > 0, C_1$ such that if $\sigma^2 D < c_1 \tau^2(\mc M),$ then with probability at least $1 - e^{-c_2 d}$, 
\begin{equation} 
    \sigma \sqrt{D-d} - C_1 \bar\kappa \sigma \sqrt{d} \le d(\mb x, \mc M ) \le \sigma \sqrt{D-d} + C_1 \bar\kappa \sigma \sqrt{d}, 
\end{equation}
where $\bar\kappa = \max\{ \kappa, 1 \}$. 
\end{lemma} 

\begin{proof} 
We begin by bounding the expected distance between $\mb x$ and $\mc M$. 
\begin{eqnarray}
    \bb E_{\mb z} \Bigl[ d(\mb x, \mc M) \Bigr] &\le& \bb E_{\mb z} \Bigl\| \mb x - \exp_{\mb x_{\natural}} \left( \mc P_{T_{\mb x_{\natural}} \mc M} \mb z \right) \Bigr\| \\
    &\le& \bb E_{\mb z} \left\| \mc P_{N_{\mb x_{\natural}} \mc M } \mb z \right\| + \tfrac{1}{2} \kappa \bb E_{\mb z} \Bigl[ \| \mc P_{T_{\mb x_{\natural}} \mc M }  \mb z \|_2^2 \Bigr] \\
    &\le& \sigma \sqrt{D -d} + \tfrac{1}{2} \kappa \sigma^2 d.
\end{eqnarray}

Conversely, 
\begin{eqnarray}
\bb E_{\mb z} \Bigl[ d(\mb x, \mc M) \Bigr] &\ge& \bb E_{\mb z} \Bigl[ d(\mb x, \mc M) \indicator{\| \mb z \|_2 \le \tfrac{1}{6} \tau } \Bigr]
\end{eqnarray}
For $\| \mb z \|_2 \le \tfrac{\tau}{6}$, we can notice that for every $t \in [0,1]$ 
\begin{equation}
   \mb x_t = (1-t) \mb x_{\natural} + t \mb x = \mb x_{\natural} + t \mb z,
\end{equation}
satisfies $d(\mb x_t, \mc M) < \tau$. Hence the projection $\mc P_{\mc M} \mb x_t$ is differentiable at every $t \in [0,1]$, and satisfies 
\begin{equation}
    \frac{d}{dt} \mc P_{\mc M} \mb x_t = \sum_{k=0}^\infty \left( \sff^*[\mb \eta_t] \right)^k \mc P_{T_{\mc P_{\mc M}{\mb x_t}}\mc M} \mb z,
\end{equation}
whence (bounding the $k \ge 1$ terms in the above Neumann series)

\begin{equation}
    \left\|  \frac{d}{dt} \mc P_{\mc M} \mb x_t - \mc P_{T_{\mb x_t}\mc M} \mb z \right\| \le \frac{3 \kappa \| \mb z \|_2 }{ 1 - 3 \kappa \| \mb z \|_2} \left\| \mc P_{T_{\mb x_t}\mc M} \mb z \right\| \le  \| \mc P_{T_{\mb x_t} \mc M} \mb z \|_2 
\end{equation}
Applying this bound, when $\| \mb z \| \le \tfrac{1}{6} \tau$, we have 
\begin{eqnarray}
    d( \mb x, \mc M ) &=& \left\| \mb x - \mb x_{\natural} - \int_0^1 \left[\frac{d}{ds} \mc P_{\mc M} \mb x_s \right] \Bigr|_{s = t} dt \right\| \\
    &\ge& \left\| \mb x - \mb x_{\natural} - \int_0^1 \mc P_{T_{\mb x_t} \mc M} \mb z \, dt \right\| -  \int_0^1 \| \mc P_{T_{\mb x_t} \mc M} \mb z \| \, dt \\ 
    &\ge&  \left\| \mb x - \mb x_{\natural} \right\| -  2 \int_0^1 \| \mc P_{T_{\mb x_t} \mc M} \mb z \| \, dt  \\
    &\ge& \| \mb x - \mb x_\natural \| - 2 T_{\max},
\end{eqnarray} 
with 
\begin{equation}
T_{\max} = \sup_{\bar{\mb x} \in B_{\mc M}(\mb x_\natural,1/\kappa), \mb v \in T_{\bar{\mb x}}\mc M} \innerprod{\mb v}{\mb z}.
\end{equation}
{We note that from Theorem 3 from \cite{tpopt}}, we have
\begin{equation}
    \bb E[T_{\max} ] \le C \bar\kappa \sigma \sqrt{d}, 
\end{equation}
while 
\begin{eqnarray}
    \bb E[T_{\max}^2] &=& \left( \bb E T_{\max} \right)^2 + \bb E \left[ ( T_{\max} - \bb E T_{\max} )^2 \right] \\
    &=& \left( \bb E T_{\max} \right)^2 + \int_{s = 0}^\infty \bb P[ ( T_{\max} - \bb E T_{\max} )^2 \ge s ] \, ds \\
    &\le& C' \bar\kappa^2 \sigma^2 d + 2 \int_{s = 0}^{\infty} \exp\left( - \frac{s}{2 \sigma^2} \right) ds \\ 
    &=& C' \bar\kappa^2 \sigma^2 d + 4 \sigma^2 \\
    &\le& C'' \bar\kappa^2 \sigma^2 d. 
\end{eqnarray} 
So
\begin{eqnarray}
\bb E \Biggl[  d(\mb x, \mc M) \indicator{\| \mb z \| \le \tfrac{1}{6} \tau } \Biggr] &\ge& \bb E \Biggl[  \Biggl(  \left\| \mb x - \mb x_{\natural} \right\| -  2T_{\max} \Biggr) \indicator{\| \mb z \| \le \tfrac{1}{6} \tau } \Biggr] \\ 
&=&  \bb E \Biggl[  \left\| \mb x - \mb x_{\natural} \right\| -  2T_{\max} \Biggr] -  \bb E \Biggl[  \Biggl(  \left\| \mb x - \mb x_{\natural} \right\| -  2T_{\max} \Biggr) \indicator{\| \mb z \| > \tfrac{1}{6} \tau } \Biggr] \\
&\ge&  \bb E \Biggl[  \left\| \mb x - \mb x_{\natural} \right\| -  2T_{\max} \Biggr] -  \left( \bb P\Biggl[ \| \mb z \| > \tfrac{1}{6} \tau \Biggr] \bb E \Biggl[ \Biggl(  \left\| \mb x - \mb x_{\natural} \right\| -  2T_{\max} \Biggr)^2 \Biggr] \right)^{1/2} 
\end{eqnarray}
We have 
\begin{equation}
    \bb P\left[ \| \mb z \| > \sigma \sqrt{D} + t \right] \le \exp\left( - \frac{t^2}{2 \sigma^2} \right),
\end{equation}
and so with assumption $c\tau \ge \sigma\sqrt{D},$
\begin{equation}
    \bb P\left[ \| \mb z \| > \tfrac{1}{6} \tau \right] \le \exp\left( - c D \right).
\end{equation}
Meanwhile,
\begin{equation}
    \bb E \Biggl[ \Biggl(  \left\| \mb x - \mb x_{\natural} \right\| -  2T_{\max} \Biggr)^2 \Biggr] \le 2  \bb E \| \mb x - \mb x_\natural \|^2 + 8 \bb E \Bigl[T_{\max}^2\Bigr] \le 2 \sigma^2 D + 8 C'' \bar\kappa^2 \sigma^2 d 
\end{equation}
Combining, we obtain 
\begin{eqnarray}
\bb E \Biggl[  d(\mb x, \mc M) \indicator{\| \mb z \| \le \tfrac{1}{6} \tau } \Biggr] \ge \sigma \sqrt{D} - 2 C \bar\kappa \sigma \sqrt{d} - \sigma e^{-cD } \sqrt{\left( 2 D + 8 C'' \bar\kappa^2 d \right)} \ge \sigma \sqrt{D-d} - C^{(3)} \bar\kappa \sigma \sqrt{d}. 
\end{eqnarray}

\noindent Using the fact that $d(\mb x_\natural + \mb z,\mc M)$ is $1$-Lipschitz in the Gaussian vector $\mb z$, we have that with probability at least $1 - 2 e^{-t^2 / 2 }$, 
\begin{equation}
    \left| d(\mb x, \mc M) - \bb E_{\mb z} \left[ d(\mb x, \mc M) \right] \right| \le t \sigma.
\end{equation}
Setting $t = \sqrt{d}$ and combining with our bounds on $\bb E[ d(\mb x, \mc M) ]$ completes the proof. 
\end{proof}

\subsection{Bounds on the Incomplete Gamma Function} \label{sec:icgf}

\begin{lemma}[Subgaussian Bounds for the Integrand of the Gamma Function Near the Phase Transition.] Let $p \ge 9$ and $x_\star = p-1$. Then 
\begin{equation}
 \frac{1}{4} \dot{\gamma}(p,x_\star) \exp\left( - \frac{(x-x_\star)^2}{2 x_\star } \right) \le \dot{\gamma}(p,x) \le 4 \dot{\gamma}(p,x_\star) \exp\left( - \frac{(x-x_\star)^2}{2 x_\star } \right) \qquad x \in \left[x_\star-x_\star^{2/3},x_\star+x_\star^{2/3}\right]. 
\end{equation}
\end{lemma} 
\begin{proof} 
We have 
\begin{equation}
    \frac{d}{dx} \gamma(p,x) = x^{p-1} e^{-x},
\end{equation}
which is maximized at $x_\star = p-1$. For $0 < \Delta_x < x_\star$, and  $x \in [x_\star - \Delta_x,x_\star+\Delta_x]$, 
\begin{eqnarray}
\log x \quad \le \quad \log x_\star + \frac{x- x_\star}{x_\star} - \frac{1}{2 x_\star^2} (x-x_\star)^2 + \frac{\Delta_x^3}{3(x_\star-  \Delta_x)^3},
\end{eqnarray}
and so 
\begin{eqnarray}
\lefteqn{ \dot{\gamma}(p,x) \quad=\quad \exp\Biggl( (p-1) \log x - x \Biggr) } \nonumber \\  
&\le& \exp\Biggl( x_\star \Bigl( \log x_\star + \frac{x-x_\star}{x_\star} - \frac{1}{2 x_\star^2} ( x - x_\star)^2 + \frac{\Delta_x^3}{3(x_\star -\Delta_x)^3}  \Bigr) - x \Biggr), \quad \text{for} \quad  0 < x_\star - \Delta_x < x < x_\star + \Delta_x \nonumber \\ 
&\le& \exp\left( \frac{x_\star \Delta_x^3 }{3 (x_\star - \Delta_x)^3} \right) \dot{\gamma}(p,x_\star)  \exp\left( - \frac{(x -x_\star)^2}{ 2 x_\star } \right) \quad \text{for} \quad 0 < x_\star - \Delta_x < x < x_\star + \Delta_x, \nonumber \\
&\le& \exp\left( 9/8 \right) \dot{\gamma}(p,x_\star)  \exp\left( - \frac{(x -x_\star)^2}{ 2 x_\star } \right) \quad \text{for} \quad 0 < x_\star - \Delta_x < x < x_\star + \Delta_x, 
\end{eqnarray}
where  in the penultimate line we used that $\exp( x_\star \log x_\star - x_\star ) = \gamma(p,x_\star)$, and in the final line we used \begin{equation}
    \Delta_x < \min\{ x_\star / 3, x_\star^{2/3} \},
\end{equation}
which under our hypotheses is implied by $\Delta_x < x_\star^{2/3}$, since $p > 28 \implies x_\star > 27 \implies x_\star^{2/3} < x_\star / 3$. Using the corresponding lower bound \begin{eqnarray}
\log x \quad \ge \quad \log x_\star + \frac{x- x_\star}{x_\star} - \frac{1}{2 x_\star^2} (x-x_\star)^2 - \frac{\Delta_x^3}{3(x_\star-  \Delta_x)^3},
\end{eqnarray}
and applying the same simplifications, we obtain the claimed lower bound on $\dot{\gamma}(p,x)$. 
\end{proof} 

\begin{lemma} \label{lem:c-of-D} There exists a numerical constant $c > 0$ such that 
\begin{equation}
\frac{
    \dot{\gamma}\left( \frac{D-1}{2}, \frac{D-3}{2} \right)}{ \Gamma(\frac{D-1}{2})} \le \frac{c}{ \sqrt{D} }   
\end{equation}
for all $D \ge 4$.
\end{lemma}
\begin{proof}
By lemma \ref{tighter bound on Stirling approximation}, we have $n!\ge\sqrt{2\pi}n^{n+1/2}e^{-n+\frac{1}{12n+1}}$, thus \begin{eqnarray}
\frac{
    \dot{\gamma}\left( \frac{D-1}{2}, \frac{D-3}{2} \right)}{ \Gamma(\frac{D-1}{2})} &= &\frac{\left(\frac{D-3}{2}\right)^{\frac{D-3}{2}} e^{-\frac{D-3}{2}}}{(\frac{D-3}{2})!} \\
    &\le& \frac{\left(\frac{D-3}{2}\right)^{\frac{D-3}{2}} e^{-\frac{D-3}{2}}}{\sqrt{2\pi\left(\frac{D-3}{2}\right)}\left(\frac{D-3}{2}\right)^{\frac{D-3}{2}}e^{-\frac{D-3}{2}+\frac{1}{6(D-3)+1}}}\\
    &=& \frac{1}{\sqrt{\pi (D-3)}}e^{-\frac{1}{6(D-3)+1}}
\end{eqnarray}
as claimed
\end{proof}

\begin{lemma}\label{tighter bound on Stirling approximation}
For any integer n, 
$n!=\sqrt{2\pi}n^{n+1/2}e^{-n+\epsilon_n}$ for some $\frac{1}{12n+1}\le \epsilon_n\le \frac{1}{12n}$
\end{lemma}
\begin{proof}
We first note that since $\log(n)$ is monotonically increasing, we have 
\begin{equation}
    {\int_1}^n \log(x) dx \le \log(n!)=\overset{n}{\sum_{m=1}}\log(m)\le {\int_1}^{n+1}\log(x) dx
\end{equation}
which implies 
\begin{equation}
    n\log(n)-n \le \log(n!)\le (n+1)\log(n+1) - n
\end{equation}
This motivates us to consider the "middle" and define the series $C_n = \log(n!) - (n+\frac{1}{2})\log(n) + n$.
Observe that 
\begin{eqnarray*}
    C_n - C_{n+1} &=&
    -\log(n+1)-(n+\frac{1}{2})\log(n)+(n+\frac{3}{2})\log(n+1)-1\\
    \\
    &=& -1 + (n+\frac{1}{2})\log(1+\frac{1}{n})\\
    &=& -1 + (\frac{1}{2n+1})\log(\frac{2n+2}{2n})
\end{eqnarray*}
We let $t=\frac{1}{2n+1}$, then, we have 
\begin{eqnarray*}
    C_n - C_{n+1} 
    &=& -1 + (\frac{1}{2t})\log(\frac{1+t}{1-t})\\
    &=& -1 +(\frac{1}{2t})(-2)(t+\frac{t^3}{3}+\frac{t^5}{5}+...) \\
    &=& \sum_{k\ge1}\frac{t^{2k}}{2k+1}
\end{eqnarray*}
Since we have both the following lower and upper bound,
\begin{equation}
    \sum_{k\ge1}\frac{t^{2k}}{2k+1}\ge \frac{t^2}{3} =\frac{1}{12n^2+12n+3}\ge \frac{1}{12n+1}-\frac{1}{12(n+1)+1}
\end{equation}
\begin{equation}
    \sum_{k\ge1}\frac{t^{2k}}{2k+1}\le \frac{1}{3}\sum_{k\ge1}t^{2k} = \frac{t^2}{3(1-t^2}=\frac{1}{12n}-\frac{1}{12n+1}
\end{equation}
We have 
\begin{equation}\label{sandwitch C_n}
\frac{1}{12n+1} - \frac{1}{12(n+1)+1}\le C_n - C_{n+1} \le \frac{1}{12n}-\frac{1}{12n+1}    
\end{equation}
which implies that series $C_n$ converges, and let $\lim_{n \to \infty} C_n = C$.  Further, line \ref{sandwitch C_n} also implies $C_n = C +\epsilon_n$ where $\frac{1}{12n+1} \le \epsilon_n \le \frac{1}{12n}$. Then, from our definition of $C_n$ we have 
\begin{equation}\label{striling bound eq}
    n!=e^Cn^{n+1/2}e^{-n+\epsilon_n}
\end{equation}
Of course now we only have to show $C = \log(\sqrt{2\pi})$. We will prove and use the De Moivre Laplace Theorem from probability theory, which is a special case of central limit theorem and roughly states that the number of successes in binomial distribution approaches normal distribution when number of trials is large enough. From the theorem we shall see $e^{-C}e^{\frac{-x^2}{2}}$ is a probability distribution and conclude the value of $C$.

Let $X_1,X_2...$ be a sequence of Bernoulli Random Variables with fixed probability of success $p \in (0,1)$, and let $q = 1-p$. Consider $$S_n = \overset{n}{\sum_{i=1}}X_i$$
Then, for any non-decreasing sequence $A_1,A_2...$ such that $A_1\ge1$, $\lim_{n \to \infty} \frac{A_n}{n^{1/6}} = 0$, for any $k$ such that $|np-k| \le A_n\sqrt{npq}$, we have 
\begin{equation}\label{bound for local bernoulli S_n}
    \bb P\left(S_n = k\right) = \frac{1}{\sqrt{npq}}\exp\left(-C+O(\frac{A_n^3}{\sqrt{n}})-\frac{(np-k)^2}{2npq}\right)
\end{equation}
To prove this, we will use the condition $|np-k| \le A_n\sqrt{npq}$, which implies $k = O(n)$ and invoke the results from line \ref{striling bound eq} to substitute the factorials, 
\begin{align}
    \bb P(S_n = k) &= \frac{n!p^kq^{n-k}}{k!(n-k)!}\\ 
    &= \frac{\exp\left(O\left(\frac{1}{n}\right)\right)n^{n+1/2}p^kq^{n-k}}{\exp\left(C+O\left(\frac{1}{k}\right)+O(\frac{1}{n-k})\right)k^{k+1/2}(n-k)^{n-k+1/2}}\\ 
    &= \exp\left(-C+O(\frac{1}{n})\right)\frac{1}{\sqrt{npq}}\frac{n^{n+1}p^{k+1/2}q^{n-k+1/2}}{k^{k+1/2}(n-k)^{n-k+1/2}}\\
    &= \exp\left(-C+O(\frac{1}{n})\right)\frac{1}{\sqrt{npq}}\frac{1}{(\frac{k}{np})^{k}(\frac{n-k}{nq})^{n-k}}\frac{1}{(\frac{k}{np})^{1/2}(\frac{n-k}{nq})^{1/2}}\\
    \label{eq:decomposition of prob Sn = k}&= \exp\left(-C+O(\frac{1}{n})\right)\frac{1}{\sqrt{npq}}\underbrace{\frac{1}{(1-\frac{np-k}{np})^{k}(1+\frac{np-k}{nq})^{n-k}}}_{\textbf Term (I)}\underbrace{\frac{1}{(1-\frac{np-k}{np})^{1/2}(1+\frac{np-k}{nq})^{1/2}}}_{\text Term(II)}
\end{align}
We will analyze Term(I) first by second order Taylor expansion, as follows 
\begin{align}
    \log(\text Term(I)) 
    &= -k\left(-\frac{np-k}{np}-\frac{(np-k)^2}{2n^2p^2}+O(\frac{A_n^3}{n\sqrt n})\right) - (n - k)\left(\frac{np-k}{nq}-\frac{(np-k)^2}{2n^2q^2}+O(\frac{A_n^3}{n\sqrt n})\right)\\
    &= -\frac{(np-k)(-kq+(n-k)p)}{npq} +\frac{(np-k)^2}{2n^2p^2q^2}\left(kq^2+(n-k)p^2\right) + O\left(\frac{A_n^3}{\sqrt n}\right)\\
    &= -\frac{(np-k)^2}{npq} +\frac{(np-k)^2}{2n^2p^2q^2}\left(np^2-k(p-q)\right) + O\left(\frac{A_n^3}{\sqrt n}\right)\\
    &= -\frac{(np-k)^2}{2n^2p^2q^2}\left(2npq-np^2+kp-kq\right) + O\left(\frac{A_n^3}{\sqrt n}\right)\\
    &= -\frac{(np-k)^2}{2n^2p^2q^2}\left(npq+(q-p)(np-k)\right) + O\left(\frac{A_n^3}{\sqrt n}\right)\\
    &= -\frac{(np-k)^2}{2npq} - \underbrace{\frac{(np-k)^3(q-p)}{2n^2p^2q^2}}_{\text{dissolved in } O\left(\frac{A_n^3}{\sqrt n}\right)} + O\left(\frac{A_n^3}{\sqrt n}\right)\\
    &= -\frac{(np-k)^2}{2npq} + O\left(\frac{A_n^3}{\sqrt n}\right)
\end{align}
Then, we consider another Taylor expansion for Term(II), but this time only zero-order, giving us 
\begin{eqnarray}
    \log(\text Term(II)) 
    &=& O(\frac{A_n}{\sqrt n})
\end{eqnarray}
Combining the terms back into line \ref{eq:decomposition of prob Sn = k}, we have 
\begin{align}
    \bb P(S_n = k) &= \exp\left(-C+O(\frac{1}{n})\right)\frac{1}{\sqrt{npq}}\exp\left(-\frac{(np-k)^2}{2npq} + O\left(\frac{A_n^3}{\sqrt n}\right)\right)\exp\left(O(\frac{A_n}{\sqrt n})\right)\\
    &= \frac{1}{\sqrt{npq}}\exp\left(-C+O(\frac{A_n^3}{\sqrt{n}})-\frac{(np-k)^2}{2npq}\right)
\end{align}
where in the last line we've used the fact that $A_1\ge1$ and $A_n$ non-decreasing, which allows us to dissolve terms. We've shown line \ref{bound for local bernoulli S_n}. Now we will finally prove and use the  DeMoivre-Laplace Theorem, that for any fixed $a < b$, we have 
\begin{equation}\label{DeMoivre-Laplace Theorem}
    \lim_{n \to \infty} \bb P\left(a\le \frac{S_n-np}{\sqrt{npq}} < b\right) = \frac{\sqrt{2\pi}}{e^{C}}\int_{t=a}^{b} \psi(t)dt
\end{equation}, where $\psi(t) = \frac{1}{\sqrt{2\pi}}e^{-t^2/2}$
We will prove this by invoking and triangular inequality. For any n, we have 
\begin{eqnarray}
   && \left|P(a\le\frac{S_n-np}{\sqrt{npq}}<b)- \frac{\sqrt{2\pi}}{e^{C}}\int_{t=a}^{b} \psi(t)dt\right| \\
   &\le& \left|P(a\le\frac{S_n-np}{\sqrt{npq}}<b)- \frac{\sqrt{2\pi}}{e^C\sqrt{npq}}\sum_{k=\lceil a\sqrt{npq}+np\rceil}^{\lceil b\sqrt{npq}+np - 1\rceil}\psi(\frac{k-np}{\sqrt{npq}})\right|\\
    &+& \frac{\sqrt{2\pi}}{e^C}\left|\frac{1}{\sqrt{npq}}\sum_{k=\lceil a\sqrt{npq}+np\rceil}^{\lceil b\sqrt{npq}+np - 1\rceil}\psi(\frac{k-np}{\sqrt{npq}})-\int_{t=a}^{b} \psi(t)dt\right|
\end{eqnarray}
Here the second bottom term is the difference between Riemman sum and integral which clearly converges to 0 as n approaches infinity, and to analyze the first term, we will consider constant $\lambda > \max\{1, a\sqrt{pq},b\sqrt{pq}\}$ and apply results from \ref{bound for local bernoulli S_n}, 
\begin{eqnarray}
    &&\left|P(a\le\frac{S_n-np}{\sqrt{npq}}<b)- \frac{\sqrt{2\pi}}{e^C\sqrt{npq}}\sum_{k=\lceil a\sqrt{npq}+np\rceil}^{\lceil b\sqrt{npq}+np - 1\rceil}\psi(\frac{k-np}{\sqrt{npq}})\right|\\
    &=&\sum_{k=\lceil a\sqrt{npq}+np\rceil}^{\lceil b\sqrt{npq}+np - 1\rceil}\left|P(S_n=k)- \frac{\sqrt{2\pi}}{e^C\sqrt{npq}}\psi(\frac{k-np}{\sqrt{npq}})\right|\\
    &=&\frac{1}{\sqrt{npq}}\left|\sum_{k=\lceil a\sqrt{npq}+np\rceil}^{\lceil b\sqrt{npq}+np - 1\rceil}e^{-C+O(\frac{\lambda^3}{\sqrt{n}})-\frac{(np-k)^2}{2npq}} - \frac{\sqrt{2\pi}}{e^C}\psi(\frac{k-np}{\sqrt{npq}})\right|\\
    &=&\frac{e^{-C}}{\sqrt{npq}}\left|\sum_{k=\lceil a\sqrt{npq}+np\rceil}^{\lceil b\sqrt{npq}+np - 1\rceil}e^{O(\frac{1}{\sqrt{n}})}e^{-\frac{(np-k)^2}{2npq}} - e^{-\frac{(np-k)^2}{2npq}}\right|\\
    &=&\frac{e^{-C}}{\sqrt{npq}}\left|\sum_{k=\lceil a\sqrt{npq}+np\rceil}^{\lceil b\sqrt{npq}+np - 1\rceil}O(\frac{1}{\sqrt n})e^{-\frac{(np-k)^2}{2npq}}\right|\\
    &\le&\frac{e^{-C}}{\sqrt{npq}}\left|\sum_{k=\lceil a\sqrt{npq}+np\rceil}^{\lceil b\sqrt{npq}+np - 1\rceil}O(\frac{1}{\sqrt n})\right|\\
    &=&O(\frac{1}{\sqrt n})
\end{eqnarray}
where in the last line we note that there's only a total of $O(\frac{1}{\sqrt n})$ possible terms. Therefore, the DeMoivre-Laplace Theorem (\ref{DeMoivre-Laplace Theorem}) is proven, and noting that since right hand side must be a probability distribution, we can conclude $e^C = \sqrt{2\pi}$, which finishes the proof.

\end{proof}

\begin{lemma}\label{Vector Hoeffding Lemma}
Let $\mb z_1, \mb z_2 \ldots \mb z_N$ be i.i.d. $D$ dimensional vectors, with $\|\mb z_i\| \le B$ almost everywhere, for some $B > 0$. Then, 
\begin{equation}
\bb P\left[ \left\| \frac{1}{N} \sum_{\ell=1}^{N} \mb z_{\ell} - \bb E\left[ \frac{1}{N} \sum_{\ell=1}^{N} \mb z_{\ell} \right] \right\| > t \right] \le D \exp\left( - \frac{t^2 N}{64B^2} \right)
\end{equation} 
\end{lemma}
\begin{proof}
For each $\mb z$, let $\mb z' = \mb z - E\left[\mb z\right]$, and  consider the $(D+1)*(D+1)$ self-adjoint matrix \[
Y =
\begin{bmatrix}
0 & \mb z'^\top \\
\mb z' & 0_{D \times D}
\end{bmatrix} = \mb{1}_{D+1} {\mb z'_{D+1}}^{\top} + \mb z'_{D+1} {\mb{1}_{D+1}}^{\top}
\]
where \[ \mb z'_{D+1}  =  =
\begin{bmatrix}
0 \\
z'^{(1)} \\
\vdots \\
z'^{(D)}
\end{bmatrix}
\in \mathbb{R}^{D+1}
\]
Since  $\|\mb z'_{D+1}\|\le 2B$, and 
\[
Y^2 = \begin{bmatrix}
\|\mb z'\|_2^2 & 0_{D\times 1} \\
0_{1\times D} & \mb z' \mb z'^{\top}
\end{bmatrix} =\mb z'_{D+1} {\mb z'_{D+1}}^{\top} + {\mb{1}_{D+1}\mb{1}_{D+1}}^{\top}\|\mb z'_{D+1}\|_2^2
\]
which implies $\|Y\|_{op} \le 8B^2$ so $Y^2 \preceq 8B^2 I.$ Then, by Matrix Hoeffding's Inequality, Theorem 1.3 from \cite{tropp2012user}, we have  for all $t \ge 0$,
\[
\mathbb{P}\left( \lambda_{\max}(\frac{Y_1+\cdots Y_N}{N}) \ge t \right) \le D \cdot \exp\left( \frac{-t^2N}{8*(8B^2)} \right).
\]
We now analyze $\lambda_{\max}(\bar{Y})$ where \[\bar{Y}=\frac{Y_1+\cdots Y_N}{N}= \begin{bmatrix}
0 & \bar{\mb z'}^\top \\
\bar{\mb z'} & 0_{D \times D}
\end{bmatrix} = \mb{1}_{D+1} \bar{\mb z'_{D+1}}^{\top} + \bar{\mb z'_{D+1}} \mb{1}_{D+1}^{\top}\] 
Since it's at most a rank 2 matrix, we can consider its representation in the subspace spanned by $\mb{1}_{D+1}$ and $\bar{\mb z'_{D+1}}$, which gives 
\[\bar{Y}|_{span(\mb{1}_{D+1},\bar{\mb z'_{D+1}})}=\begin{bmatrix}
0 & \|\bar{\mb z'}\|_2 \\
\|\bar{\mb z'}\|_2 & 0
\end{bmatrix},\]
therefore, we have \[\lambda_{\max}(\bar{Y})=\|\bar{\mb z'}\|_2=\|\frac{1}{N} \sum_{\ell=1}^{N}\left( \mb z_{\ell} - \bb E\left[ \mb z_{\ell} \right]\right)\|,\] which finishes the proof.
\end{proof}


\end{document}